\documentclass{article}

\usepackage[final]{neurips_2022}

\newcommand{\reel}{\mathbb{R}}
\def\msa{\mathsf{A}}
\def\msm{\mathsf{M}}

\usepackage[11pt]{moresize}
\usepackage{xr}
\makeatletter
\newcommand*{\addFileDependency}[1]{% argument=file name and extension
  \typeout{(#1)}
  \@addtofilelist{#1}
  \IfFileExists{#1}{}{\typeout{No file #1.}}
}
\makeatother

\newcommand*{\myexternaldocument}[1]{%
    \externaldocument{#1}%
    \addFileDependency{#1.tex}%
    \addFileDependency{#1.aux}%
}
\myexternaldocument{supplementary}

\usepackage[utf8]{inputenc} % allow utf-8 input
\usepackage[T1]{fontenc}    % use 8-bit T1 fonts
\usepackage{hyperref}       % hyperlinks %REMOVE THE DRAFT OPTION WHEN SUBMITTING !!!! REMOVE THE DRAFT OPTION WHEN SUBMITTING
\usepackage{url}            % simple URL typesetting
\usepackage{booktabs}       % professional-quality tables
\usepackage{amsfonts}       % blackboard math symbols
\usepackage{nicefrac}       % compact symbols for 1/2, etc.
\usepackage{microtype}      % microtypography
\usepackage{xcolor}         % colors
\usepackage{tikz}
\usepackage{caption}
\usepackage{float}
\usepackage{graphicx}
\usepackage{subcaption}
\usetikzlibrary{arrows.meta}
\usetikzlibrary{calc}
 \usepackage[utf8]{inputenc}   % LaTeX, comprends les accents !
\usepackage[T1]{fontenc}      % Police contenant les caractÃ¨res franÃ§ais
\usepackage{comment}
\usepackage{geometry}
\usepackage{natbib}
\usepackage[tbtags]{amsmath}
\usepackage{amsthm}
\allowdisplaybreaks
\usepackage{amssymb,mathrsfs}
\usepackage{nccmath}
\usepackage{amsfonts}
\usepackage{upgreek}
\usepackage{xspace}

\usepackage{graphicx}
\usepackage{color}
\usepackage{algorithm, algpseudocode}
\usepackage{stmaryrd}
\usepackage[inline]{enumitem}
\usepackage{url}

\usepackage{tikz}
\usetikzlibrary{calc}

\usepackage{pgfplots}
\usepackage{xcolor}
\usepackage{bbm}
\usepackage{ifthen}
\usepackage{xargs}
\usepackage[textwidth=1.8cm]{todonotes}

\usepackage{aliascnt}
\usepackage[capitalise,noabbrev]{cleveref}
\usepackage{autonum}
\makeatletter
\newtheorem{theorem}{Theorem}
\newtheorem*{lemma_nonumber*}{Lemma}

\newaliascnt{lemma}{theorem}
\newtheorem{lemma}[lemma]{Lemma}
\aliascntresetthe{lemma}
\newaliascnt{corollary}{theorem}
\newtheorem{corollary}[corollary]{Corollary}
\aliascntresetthe{corollary}
\newaliascnt{proposition}{theorem}

\aliascntresetthe{proposition}
\newaliascnt{definition}{theorem}

\aliascntresetthe{definition}
\newaliascnt{remark}{theorem}

\aliascntresetthe{remark}
\crefformat{assumption}{{\textbf{A}}#2#1#3}

\newtheorem{assumptionF}{\textbf{F}\hspace{-3pt}}
\crefformat{assumptionF}{{\textbf{F}}#2#1#3}

\Crefname{assumptionB}{\textbf{B}\hspace{-3pt}}{\textbf{B}\hspace{-3pt}}
\crefname{assumptionB}{\textbf{B}}{\textbf{B}}

\Crefname{assumptionC}{\textbf{C}\hspace{-3pt}}{\textbf{C}\hspace{-3pt}}
\crefname{assumptionC}{\textbf{C}}{\textbf{C}}

\Crefname{assumptionH}{\textbf{H}\hspace{-3pt}}{\textbf{H}\hspace{-3pt}}
\crefname{assumptionH}{\textbf{H}}{\textbf{H}}

\Crefname{assumptionT}{\textbf{T}\hspace{-3pt}}{\textbf{T}\hspace{-3pt}}
\crefname{assumptionT}{\textbf{T}}{\textbf{T}}

\Crefname{assumptionT}{\textbf{T}\hspace{-3pt}}{\textbf{T}\hspace{-3pt}}
\crefname{assumptionT}{\textbf{T}}{\textbf{T}}

\Crefname{assumptionL}{\textbf{L}\hspace{-3pt}}{\textbf{L}\hspace{-3pt}}
\crefname{assumptionL}{\textbf{L}}{\textbf{L}}

\Crefname{assumptionQ}{\textbf{Q}\hspace{-3pt}}{\textbf{Q}\hspace{-3pt}}
\crefname{assumptionQ}{\textbf{Q}}{\textbf{Q}}

\Crefname{assumptionAR}{\textbf{AR}\hspace{-3pt}}{\textbf{AR}\hspace{-3pt}}
\crefname{assumptionAR}{\textbf{AR}}{\textbf{AR}}

\usepackage{bm}
\usepackage{wrapfig}

\newcommand{\tta}{\mathtt{A}}

\newcommand{\Capprox}{\tta}

\newcommandx\ctun[1][1=T]{\Capprox_{#1,1}}

\newcommandx{\expec}[2]{{\mathbb E}\left[#1 \middle \vert #2  \right]} %%%% esperance conditionnelle

\newcommand{\Lip}{\mathtt{L}}

\newcommandx{\norm}[2][1=]{\ifthenelse{\equal{#1}{}}{\left\Vert #2 \right\Vert}{\left\Vert #2 \right\Vert^{#1}}}
\newcommandx{\normLigne}[2][1=]{\ifthenelse{\equal{#1}{}}{\Vert #2 \Vert}{\Vert #2\Vert^{#1}}}

\def\bfc{\mathbf{c}}

\def\msa{\mathsf{A}}
\def\msh{\mathsf{H}}

\def\msb{\mathsf{B}}

\def\msh{\mathsf{H}}
\def\msm{\mathsf{M}}

\newcommand{\mcb}[1]{\mathcal{B}(#1)}

\def\rset{\mathbb{R}}

\def\rmd{\mathrm{d}}

\newcommandx{\functionspace}[2][1=+]{\mathbb{F}_{#1}(#2)}
\newcommandx{\VarDeux}[3][3=]{\operatorname{Var}^{#3}_{#1}\left\{#2 \right\}}

\newcommand{\LeftEqNo}{\let\veqno\@@leqno}

\newcommand{\floor}[1]{\left\lfloor #1 \right\rfloor}

\newcommand{\N}{\ensuremath{\mathbb{N}}}

\newcommandx{\Vnorm}[2][1=V]{\| #2 \|_{#1}}
\newcommandx{\VnormEq}[2][1=V]{\left\| #2 \right\|_{#1}}
\newcommandx\probaMarkovTilde[2][2=]
{\ifthenelse{\equal{#2}{}}{{\widetilde{\mathbb{P}}_{#1}}}{\widetilde{\mathbb{P}}_{#1}\left[ #2\right]}}

\def\eqsp{\;}
\newcommand{\coint}[1]{\left[#1\right)}
\newcommand{\ocint}[1]{\left(#1\right]}
\newcommand{\ooint}[1]{\left(#1\right)}
\newcommand{\ccint}[1]{\left[#1\right]}

\newcommandx{\weight}[2][2=n]{\omega_{#1,#2}^N}

\newcommandx\sequence[3][2=,3=]
{\ifthenelse{\equal{#3}{}}{\ensuremath{\{ #1_{#2}\}}}{\ensuremath{\{ #1_{#2}, \eqsp #2 \in #3 \}}}}

\newcommandx\sequenceD[3][2=,3=]
{\ifthenelse{\equal{#3}{}}{\ensuremath{\{ #1_{#2}\}}}{\ensuremath{( #1)_{ #2 \in #3} }}}

\newcommandx{\sequencen}[2][2=n\in\N]{\ensuremath{\{ #1_n, \eqsp #2 \}}}
\newcommandx\sequenceDouble[4][3=,4=]
{\ifthenelse{\equal{#3}{}}{\ensuremath{\{ (#1_{#3},#2_{#3}) \}}}{\ensuremath{\{  (#1_{#3},#2_{#3}), \eqsp #3 \in #4 \}}}}
\newcommandx{\sequencenDouble}[3][3=n\in\N]{\ensuremath{\{ (#1_{n},#2_{n}), \eqsp #3 \}}}

\newcommand{\opnorm}[1]{{\left\vert\kern-0.25ex\left\vert\kern-0.25ex\left\vert #1
    \right\vert\kern-0.25ex\right\vert\kern-0.25ex\right\vert}}

\def\Lip{\operatorname{Lip}}

\def\Id{\operatorname{Id}}

\newcommandx{\CPE}[3][1=]{{\mathbb E}_{#1}\left[#2 \middle \vert #3  \right]} %%%% esperance conditionnelle
\newcommandx{\CPELigne}[3][1=]{{\mathbb E}_{#1}[#2  \vert #3  ]} %%%% esperance conditionnelle
\newcommandx{\CPEsq}[3][1=]{{\mathbb{E}^{1/2}}_{#1}\left[#2 \middle \vert #3  \right]} %%%% esperance conditionnelle
\newcommandx{\CPVar}[3][1=]{\mathrm{Var}^{#3}_{#1}\left\{ #2 \right\}}
\newcommand{\CPP}[3][]
{\ifthenelse{\equal{#1}{}}{{\mathbb P}\left(\left. #2 \, \right| #3 \right)}{{\mathbb P}_{#1}\left(\left. #2 \, \right | #3 \right)}}

\newcommandx{\osc}[2][1=]{\mathrm{osc}_{#1}(#2)}

\def\Id{\operatorname{Id}}

\newcommand{\ensembleLigne}[2]{\{#1\,:\eqsp #2\}}

\newcommand\coupling[2]{\Gamma(\mu,\nu)}

\def\vareps{\varepsilon}

\newcommandx{\KL}[2]{\operatorname{KL}\left( #1 | #2 \right)}
\newcommandx{\KLsqrt}[2]{\operatorname{KL}^{1/2}\left( #1 | #2 \right)}
\newcommandx{\Jef}[2]{\operatorname{J}\left( #1 , #2 \right)}
\newcommandx{\JefLigne}[2]{\operatorname{J}( #1 , #2 )}
\newcommandx{\KLLigne}[2]{\operatorname{KL}( #1 | #2 )}
\newcommandx{\KLLignesqrt}[2]{\operatorname{KL}^{1/2}( #1 | #2 )}

\def\gaStep
\def\QKer{Q}

\def\distance{\mathbf{d}}
\newcommandx{\wasserstein}[3][1=\distance,3=]{\mathbf{W}_{#1}^{#3}\left(#2\right)}
\newcommandx{\wassersteinLigne}[3][1=\distance,3=]{\mathbf{W}_{#1}^{#3}(#2)}
\newcommandx{\wassersteinD}[1][1=\distance]{\mathbf{W}_{#1}}
\newcommandx{\wassersteinDLigne}[1][1=\distance]{\mathbf{W}_{#1}}

\def\sigmaD{\sigma^2}

\newcommandx{\phibfs}[1][1=]{\pmb{\varphi}_{\sigmaD_{#1}}}

\newcommandx\sequenceg[3][2=,3=]
{\ifthenelse{\equal{#3}{}}{\ensuremath{( #1_{#2})}}{\ensuremath{( #1_{#2})_{ #2 \geq #3}}}}

\newcommandx{\distV}[1][1=\bfc]{\mathbf{W}_{#1}}
\newcommandx{\distVdeux}[1][1=W_2]{\mathbf{d}_{#1}}

 \usepackage{comment}
\hypersetup{colorlinks,citecolor=blue}

\title{Can Push-forward Generative Models  \\Fit Multimodal Distributions?}

\author{
    Antoine Salmona  \\
    Centre Borelli, \\
    ENS Paris Saclay, France \\
\And
  Agnès Desolneux \\
  Centre Borelli, CNRS\\
  ENS Paris Saclay, France \\
  \And
  Julie Delon \\
  MAP5, Université Paris Cité, France \\
  Institut Universitaire de France (IUF)\\
  \And
  Valentin De Bortoli \\
  Center for Sciences of Data, CNRS \\
  ENS Ulm, France\\
}

\begin{document}

\maketitle

\begin{abstract}
Many generative models synthesize data by transforming a standard Gaussian random variable using a deterministic neural network. Among these models are the Variational Autoencoders and the Generative
Adversarial Networks. In this work, we call them  "push-forward" models and study their expressivity. We formally demonstrate that the Lipschitz constant of these generative networks
has to be large in order to fit multimodal distributions. 
More precisely, we show that the total variation distance and the Kullback-Leibler divergence between the generated 
and the data distribution are bounded from below by a constant depending on the mode separation and the Lipschitz constant.
%otherwise the total 
%variation distance and the Kullback-Leibler divergence between the generated 
%and the data distribution are bounded from below by strictly positive
%constants depending on the modes separation and
%the Lipschitz constant of the generative network.
Since constraining the Lipschitz constants of neural networks is a common way to stabilize generative models, 
there is a provable trade-off between the ability of push-forward models to approximate multimodal distributions and the stability of their training.
We validate our findings on one-dimensional and image datasets and empirically show that the recently introduced diffusion models do not suffer of such limitation.
\end{abstract}

\section{Introduction}
Generative modeling has become over the last years one of the most popular
research topics in machine learning and computer vision.  From a mathematical
perspective, the goal of generative modeling can be seen as predicting new
synthetic samples from an unknown probability distribution $ \nu $ on
$ \rset^d $ given the information of $ m $ \emph{true} samples $x_i$ drawn
from $ \nu $ (the data distribution).  A general approach to solve this problem is to define a parametric family of probability distributions
$ (\nu_{\theta})_{\theta \in \Theta} $ and solve the problem
\begin{equation}
\textstyle{
\min_{\theta \in \Theta} D(\frac{1}{m}\sum_i\delta_{x_i},\nu_{\theta}) \eqsp ,}
\end{equation}
where $ D $ is a similarity measure between probability distributions and $ \delta_{x} $ is the delta distribution at $ x $.
Beside their direct application \cite[]{sandfort2019data,antoniou2018data},
generative models have been used in numerous applications in various machine
learning subfields, such as solving inverse problems
\cite[]{ravuri2021skilful,ledig2017photorealistic} or machine translation
\cite[]{isola2018image,yang2018improving}. However, most generative modeling
methods still lack theoretical understanding and it remains often unclear
whether the method approaches correctly the probability distribution $ \nu $ or
only generates samples that appear to have been drawn from $ \nu $ without fully
recovering the underlying structure of the distribution.  In this work, we focus
on the particular class of \emph{push-forward generative models}. 
Those models have in common that for any $ \theta \in \Theta $, the
parametric distribution $ \nu_{\theta} $ approaching $ \nu $ is of the form
\begin{equation} 
\nu_{\theta} =  g_{\theta\#}\mu_p \eqsp ,
\end{equation} 
where $ \mu_p = \mathrm{N}(0,\Id_p) $ is the Gaussian standard distribution in dimension $ p  $, $ \# $ is the 
push-forward operator \footnote{If $ \mu $ is a measure on $ \rset^p $ and $ f $ is a mapping from $ \rset^p $ to $ \rset^d $, 
the push-forward measure $ f_{\#}\mu $ is the measure on $ \rset^d $ such that for all Borel set $ \mathsf{A} $ of  $\rset^d $, $ f_{\#}\mu(\mathsf{A})=\mu(f^{-1}(\mathsf{A}))$.}, and 
$ g_{\theta} : \rset^p \rightarrow \rset^d $ is a deterministic neural network of parameter $ \theta $. 
This class includes two of the most popular generative models: the \emph{Variational Auto-Encoders} (VAEs) \footnote{In this work, the VAE model considered is the Gaussian-VAE since the data are real-valued. See \Cref{sec:diff} for details on why the generated distribution is of the form $ g_{\theta\#}\mu_p $ in this model.} \cite[]{kingma2013auto} and the \emph{Generative Adversarial Networks} (GANs) \cite[]{goodfellow2014generative}. It also includes other models such as most of normalizing flows \cite[]{rezende2015variational}.

Deep neural networks are most of the time Lipschitz mappings by design, 
since their activation functions are generally Lipschitz. 
%This is mainly due to the fact that deep neural networks have to be
%differentiable almost everywhere in order to be trained.
In the literature, constraining the Lipschitz constant of a neural network is widely used 
as a way to increase its robustness~\cite[]{scaman2019lipschitz,fazlyab2019efficient}, 
in particular to adversarial attacks~\cite[]{goodfellow2015explaining}. 
%It is also a common way to stabilize network training~\cite[]{miyato2018spectral}.
Common approaches to bound Lipschitz constants of neural networks are spectral normalization~\cite[]{miyato2018spectral}, 
adding a gradient penalization in the loss~\cite[]{gulrajani2017improved,mohajerin2018data}, 
or Jacobian regularization~\cite[]{pennington2017resurrecting}.   
These approaches have been widely used to stabilize the training of GANs, 
where Lipschitz constraints have been first imposed on 
discrimators~\cite[]{arjovsky2017wasserstein,kodali2017convergence,fedus2018many}, 
while recent state-of-the-art architectures such as BigGAN~\cite[]{brock2018large}, SAGAN \cite[]{zhang2019self}
or StyleGAN2~\cite[]{karras2020analyzing} also impose similar constraints on the generators 
through spectral normalization ~\cite[]{brock2018large,zhang2019self}, or Jacobian regularization \cite[]{karras2020analyzing}. 
In contrast to GANs, the recent study of \cite{kumar2020implicit} shows that the decoder Jacobian in VAEs is 
implicitly regularized, which limits its Lipschitz constant. 
%A similar implicit regularization might be operating in the case of normalizing flows~
%\cite[]{behrmann2021understanding}, for which it is known that limited Lipschitz 
%constant are necessary to ensure invertibility~\cite[]{behrmann2019invertible}, 
%and large bi-Lipschitz constants lead to numerical instability~\cite[]{behrmann2021understanding}. 

Recently, \cite{dhariwal2021diffusion} trained an unconditional \emph{Score-based Generative Model} (SGM) \cite[]{song2019generative,ho2020denoising} 
on ImageNet \cite[]{russakovsky2015imagenet} and achieved state-of-the-art generation. 
%beating BigGAN \cite[]{brock2018large} in terms of Frechet inception Distance \cite[]{heusel2017gans}. 
To the best of our knowledge, there is no push-forward generative model
capable of reaching this kind of performance on such a complex dataset 
without explicitly adding any conditional label information in the model, see \cite[]{brock2018large} for instance. SGMs (also known as \emph{diffusion models}) proceed as follows: first, noise is progressively added to the data distribution until we reach a standard Gaussian distribution. Then this forward dynamics is reversed leveraging recent advances in deep learning and tools from score-matching  \cite[]{hyvarinen2005estimation,vincent2011connection}. We refer to \cite{song2020score} for an introduction on SGMs. In those models,
the parametric distribution is also of the form $ \nu_{\theta} =  g_{\theta\#}\mu_{p} $,
where $ g_{\theta} $ is the whole reverse diffusion dynamic (which can be seen as a composition of deterministic Lipschitz mappings) and $ p = d(N+1)$  with $ N $ being the total number of steps in the dynamic (More details can be found in \Cref{sec:diff}).
However, those models are not push-forward generative models in the strict sense of the term, since the push-forward mapping is not a simple neural network anymore. An important difference is that optimization is not directly performed on the push-forward mapping itself but on an auxiliary function (the score). We therefore categorize them as \emph{indirect push-forward generative models} in this work.

\paragraph{Contributions of the paper.}
In this paper, we study the expressivity of direct and indirect push-forward generative models in relation to the Lipschitz constant of the push-forward mapping they learn. More precisely, in Section~\ref{sec:main-result}, for a Lipschitz function $g$ and a given multimodal probability distribution $ \nu $, we formally demonstrate that the Lipschitz constant of $ g $ must necessarily be large in order for $ g_{\#}\mu_p $ to approximate $\nu $ correctly, as it has been already intuitively observed in the literature \cite[]{lu2020implicit,luise2020generalization,khayatkhoei2018disconnected}. As a direct consequence, we exhibit lower bounds on $ D(g_{\#}\mu_p,\nu)$, where $ D $ is the total variation distance or
the Kullback-Leibler divergence, with an explicit dependence on the Lipschitz constant $\mathrm{Lip}(g)$ of $g$, which highlights that there is a fundamental trade-off for (direct) push-forward generative models between expressivity and stability of training. In \Cref{sec:experiments}, we illustrate these theoretical results on several experiments, showing the difficulties of GANs and VAEs to simulate  multimodal distributions. We compare these models with SGMs and show experimentally that SGMs seem to be able to generate correctly multimodal distributions while keeping the Lipschitz constant of the score network relatively small, suggesting that these models do not suffer of such previously mentioned limitations. All the proofs are postponed to the appendix.

\section{Related Works}
%As stated in the introduction, this paper focus on the link between the expressivity of pushforward generative models and the Lipschitz constant of the corresponding networks.  

Assessing the efficiency of push-forward models is a recurrent and important question in the literature. 
\citet{sajjadi2018assessing} and \citet{kynkaanniemi2019improved} 
propose Precison and Recall metrics to assess GANs, aiming to measure simultaneously 
the mode collapse and the proportion of off-manifold generated samples.
% In  \cite{kynkaanniemi2019improved}, the precision refers to the portion ofgenerated samples that are in the target manifolds, while the recall measures the proportion of the target distribution being reconstructed by the model distribution.
Using similar metrics, \citet{tanielian2020learning} prove an upper bound on the precision of vanilla GANs
(the proportion of generated samples which could have been generated by the target distribution). To overcome this limitation, 
 they simply propose to reject samples associated with large values of the generator Jacobian. 
 The intuition behind this idea is that those samples lie in regions of the space where the discontinuous
optimal generator would "jump" between modes and so are off-manifold.

In the context of normalizing flows, it has been shown that the invertibility constraint limits the expressivity of the model. Indeed,  ~\cite{cornish2020relaxing} show that distributions generated by invertible normalizing flows have a support which is necessarily homeomorphic to the support of the latent distribution. As an outcome,  the Lipschitz constant of the inverse flow has to approach infinity to correctly approximate distributions lying on disconnected manifolds \cite[]{cornish2020relaxing,hagemann2021stabilizing,behrmann2021understanding}. 
To improve the expressivity of normalizing flows, it has been proposed in \cite{cornish2020relaxing}
and \cite{wu2020stochastic} to inject stochasticity in the model.

%by replacing the single bijection by a continuously indexed family of bijections.

Another line of research focuses on the fact that the model has access to only the empirical distribution 
$ \nu_n = \frac{1}{n}\sum_{i} \delta_{x_i} $ and not to the true target distribution. 
For instance \cite{nagarajan2018theoretical} study to what extend GANs only memorize the data. 
\cite{gulrajani2018towards} highlight the fact that common GAN benchmarks prefer training set memorization  to
a model  which imperfectly fits the true distribution but covers more of its support. 
Related to this, \cite{stephanovitch2022optimal} study specifically the Wasserstein GAN case, 
where the latent distribution is uniform and construct an optimal generator which minimizes 
the Wasserstein distance of order $ 1 $ between the push-forward measure and the empirical distribution, thus deriving a lower bound
on the $1$-Wasserstein distance. In the same paper, and more related to our work, 
the authors study the asymptotic case of an infinite number of data and show that most of the time 
the minimal  $1$-Wasserstein  distance between the push-forward measure and the target distribution remains strictly positive.

\section{Push-forward measure and Lipschitz mappings}

%\subsection{Isoperimetric property of pushforward measures}\textcolor{red}{pas sûr du nom de cette}
\label{sec:main-result}

In this section, we study the properties of the push-forward measure
$ g_{\#}\mu_p $ when $ \mu_p = \mathrm{N}(0,\Id_p) $ is the standard
Gaussian distribution in dimension $ p $ and $ g$ is a Lipschitz mapping. 
First, for any probability measure $ \gamma $ on $ \rset^d$
and any Borel set $\msa$ of $ \rset^d$, we define the $ \gamma $-surface area of
$\msa$ by
\begin{equation}
  \textstyle{
    \gamma^{+}(\partial \msa) = \liminf_{\varepsilon \rightarrow 0^+} (\gamma(\msa_{\varepsilon}) - \gamma(\msa))/\varepsilon \eqsp , 
  }
\end{equation}
where
$ \msa_{\varepsilon} = \ensembleLigne{x \in \rset^d}{\textrm{there exists } a
  \in \msa, \ \|x-a\|\leq \varepsilon} $ is the $ \varepsilon$-extension of
$ \msa $ and $ \partial \msa $ is the boundary of $ A $. 
The $ \gamma $ - surface area can be interpreted as the mass of $ \gamma $ on the 
hypersurface $ \partial \msa $. Note that the support of $ \gamma $ and $ A $  can be 
sets of intrinsic dimension smaller than $ d $, which is most of the time the case
when working with real data which are likely to live on low dimensional manifolds \cite[]{pope2020intrinsic}. 
The main theoretical result of this paper establishes some properties
of push-forward measures depending on the regularity of the push-forward
mapping. 

\begin{theorem}\label{thm:perfectapprox}
  Let $ g: \rset^p \rightarrow \rset^d $ be a Lipschitz function with Lipschitz
  constant $ \mathrm{Lip}(g) $. Then for any Borel set
  $ \msa \in \mcb{\rset^d}$,
\begin{equation}\label{eq:inequalitydimn}
\textstyle{
\mathrm{Lip}(g)(g_{\#}\mu_p)^{+}(\partial \msa) \geq  \varphi\left(\Phi^{-1}(g_{\#}\mu_p(\msa))\right) \eqsp , }
\end{equation}
 where $ \varphi(x) = (2 \uppi)^{-1/2}\exp[-x^2/2] $ and $ \Phi(x) = \int_{-\infty}^{x}\varphi(t) \rmd t $.
In addition, we have that for any $  r \geq 0 $
\begin{equation}
\textstyle{
    g_{\#}\mu_p(\msa_r) \geq  \Phi\left(r/\mathrm{Lip}(g) + \Phi^{-1}(g_{\#}\mu_p(\msa))\right) \eqsp .
  }
  \label{eq:large_r}
\end{equation}
\end{theorem}

\begin{proof}[Sketch of proof]
The proof of this result consists in establishing lower-bounds on
$(g_{\#}\mu_p)^{+}(\partial \msa)$ and $g_{\#}\mu_p(\msa_r)$ which can be
expressed as Gaussian integrals. We conclude upon combining this result and the Gaussian
isoperimetric inequality, see \cite{sudakov1978extremal} $\mu_p^+(\partial \msa) \geq \varphi(\Phi^{-1}(\mu_p(\msa)) ).$
\end{proof}
\begin{figure}
\begin{minipage}[c]{0.64\linewidth}
Note that \eqref{eq:large_r} implies \eqref{eq:inequalitydimn} upon remarking
that \eqref{eq:large_r} is an equality for $r =0$, dividing by $ r $ and letting $r \to 0$. 
\Cref{thm:perfectapprox} recovers the Gaussian inequality in the
case where $g$ is the identity mapping and extends it to all Lipschitz mappings. 
As the Gaussian inequality, \Cref{thm:perfectapprox} is dimension free, in 
the sense that neither $d $, nor $p $, nor the intrinsic dimension of $ g(\rset^p) $ play a role in the lower bounds. In the following section, we are going to use \Cref{thm:perfectapprox} to (i) give a lower bound on the Lipschitz constant  so that push-forward generative models \emph{exactly} match the data distribution,
 (ii) give a lower bound on the total variation and the
\end{minipage} \hfill
\begin{minipage}[c]{0.34\linewidth}
\vspace{-1.47em}
  \centering
  \includegraphics[width=\textwidth]{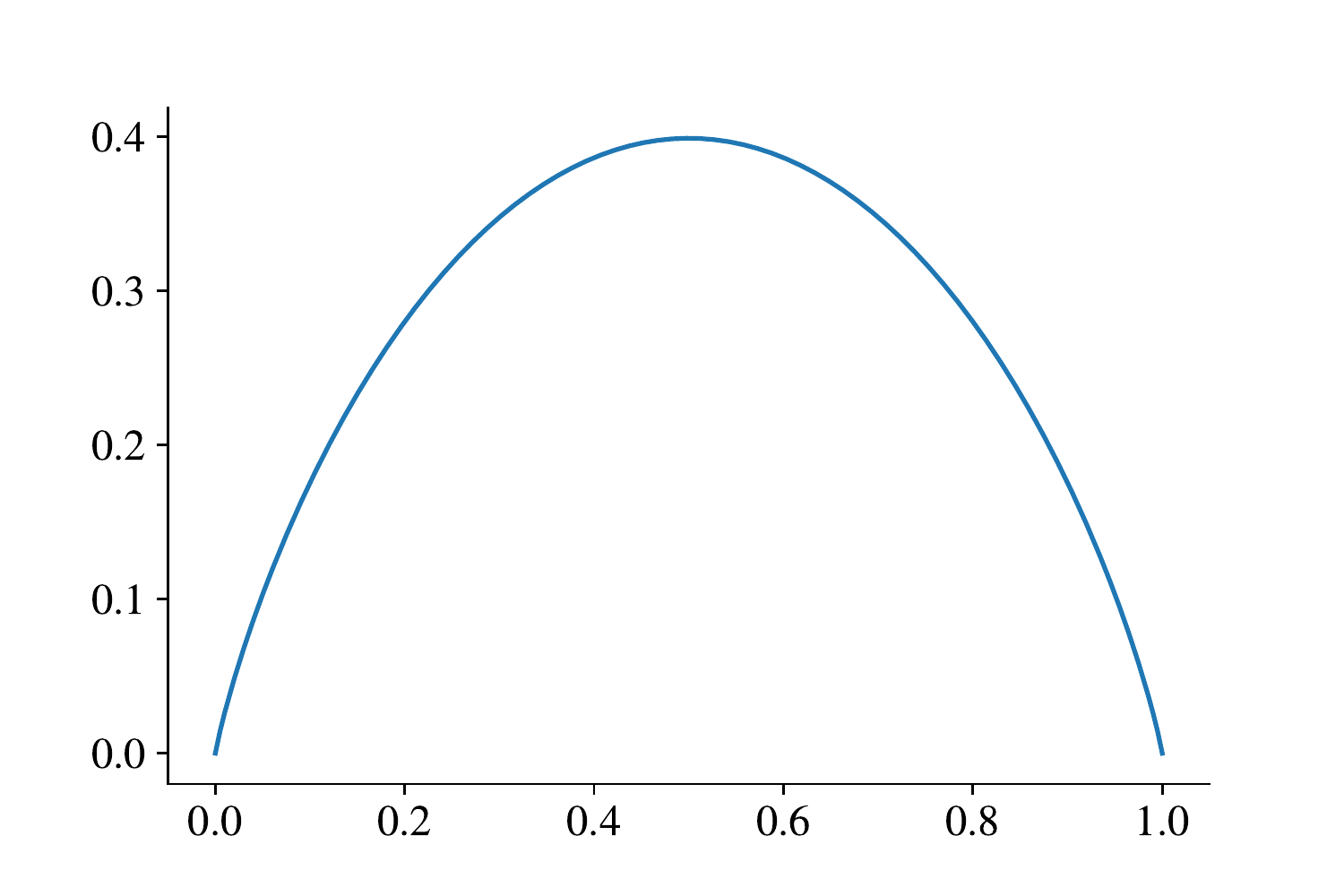}  
  \caption{Graph of $ \varphi \circ \Phi^{-1} $.}\label{fig:I_map}
\vspace{0.5em}
\end{minipage}
 Kullback-Leibler divergence between the push-forward and data distributions which depends on the Lipschitz constant of the model.
\end{figure}

\begin{comment}
Note that \eqref{eq:large_r} implies \eqref{eq:inequalitydimn} upon remarking
that \eqref{eq:large_r} is an equality for $r =0$, dividing by $ r $ and letting $r \to 0$. 
\Cref{thm:perfectapprox} recovers the Gaussian inequality in the
case where $g$ is the identity mapping and extends it to all Lipschitz mappings. 
As the Gaussian inequality, \Cref{thm:perfectapprox} is dimension free, in 
the sense that neither $d $, nor $p $, nor the intrinsic dimension of $ g(\rset^p) $ play a role in the lower bounds.
In the following section, we are going to use \Cref{thm:perfectapprox} to (i) give a lower bound on the Lipschitz constant so that push-forward generative models \emph{exactly} match the data distribution, (ii) give a lower bound on the total variation and the Kullback-Leibler divergence between the push-forward and data distributions which depends on the Lipschitz constant of the model.
\end{comment}

\subsection{Lower bounding the Lipschitz constant of push-forward mappings}
\label{sec:lower-bound-lip}

Equation \eqref{eq:inequalitydimn} implies that the Lipschitz constant of $ g $ must necessarily be large for $ g_{\#}\mu_p $ to be multimodal. It provides indeed a lower bound on the Lipschitz constant of
the mappings $ g $ which push $ \mu_p $ into a given measure $ \nu $.
In the extreme case where the support of $ \nu $ is composed
of disconnected manifolds, we retrieve
that there doesn't exist any Lipschitz mapping which pushes $ \mu_p $ into $ \nu $ since it
can be found Borel sets $ \msa $ with null $ \nu $ -surface area 
but such that the right-hand term of \eqref{eq:inequalitydimn} is strictly 
positive (which occurs when $ 0 < \nu(\msa) < 1 $). In the intermediate
case where the support of $ \nu $ is connected but $ \nu $ is multimodal,
the less mass $ \nu $ has between modes, the larger must be the Lipschitz constant of
the mappings which push $ \mu_p $ into $ \nu $. Indeed, if $ \nu $ has 
little mass between its modes, one can find sets $ \msa $ with small $ \nu $-surface area and such 
that $ 0 < \nu(\msa) < 1 $. As a toy example, we get 
an explicit bound on the Lipschitz constant of the mappings which push
$ \mu_p $ into a mixture of two isotropic Gaussians. 

\begin{corollary}\label{coro:gm1}
Let
$ \nu = \lambda\mathrm{N}(m_1,\sigma^2\Id_d) + (1 - \lambda)\mathrm{N}(m_2,\sigma^2\Id_d) $
with $ m_1, m_2 \in \rset^d $, $ \sigma > 0 $ and $ \lambda \in (0,1). $ Assume that there exists
$ g : \rset^p \rightarrow \rset^d $ Lipschitz such that $ g_\#\mu_p = \nu $. Then
\begin{equation}
  \textstyle{
\mathrm{Lip}(g) \geq  \sigma\exp\left[\normLigne{m_2-m_1}^2 / (8 \sigma^2) - (\Phi^{-1}(\lambda))^2/2\right] \eqsp . }
\end{equation}
\end{corollary}

\begin{proof}[Sketch of proof]
The proof of this result consists into applying Inequality 
\eqref{eq:inequalitydimn}  of \Cref{thm:perfectapprox} 
on the half-space $ \msh $ such that $ \partial \msh $ is 
the equidistant line of the means of the two Gaussians, and then 
to explicit the value of $ \nu^+(\partial \msh).$ See \Cref{fig:hyperplan} 
for a visualization of $ \partial \msh $ in the univariate case.
\end{proof}
\begin{comment}
\begin{figure}[h!]
\begin{minipage}[c]{0.54\linewidth}
Note that assuming there exists $ g : \rset^p \rightarrow \rset^d $ such
that $ g_\#\mu_p = \nu $ implies  $ p \geq d $ since 
$ \nu $ covers the whole ambient space and so $ g $ must be a surjective mapping. 
This bound is maximal in the balanced case when $ \lambda = 1/2 $ since 
$ \Phi^{-1}(\lambda) = 0 $ in that case. Otherwise, the more unbalanced the modes are,
the smaller the bound is since the two terms in the exponential compensate each other more and more.
Extending this corollary to mixtures of more than two Gaussians with different covariance matrices is 
technically difficult but we could \makebox[\linewidth][s]{expect a similar exponential growth in the square} \par
\end{minipage} \hfill
\begin{minipage}[c]{0.44\linewidth}
\vspace{-4.47em}
  \centering
  \includegraphics[width=\textwidth]{hyperplan.pdf}  
  \caption{The hypersurface $ \partial \msh $ of the proof of \Cref{coro:gm1} in the univariate case.}\label{fig:I_map}
\vspace{0.47em}
\end{minipage}
distance
between modes since it depends  mainly on the order of magnitude of the local minima of the
distribution density.
As a by product of \Cref{thm:perfectapprox}, we also get the following result which shows that (in the
one-dimensional case) the optimal transport map for the $\ell_2$ cost minimizes
the Lipschitz constant of the push-forward mapping.
\end{figure}
\end{comment}

Note that assuming there exists $ g : \rset^p \rightarrow \rset^d $ such
that $ g_\#\mu_p = \nu $ implies  $ p \geq d $ since 
$ \nu $ covers the whole ambient space and so $ g $ must be a surjective mapping. 
This bound is maximal in the balanced case when $ \lambda = 1/2 $ since 
$ \Phi^{-1}(\lambda) = 0 $ in that case. Otherwise, the more unbalanced the modes are,
the smaller the bound is since the two terms in the exponential compensate each other more and more.
Extending this corollary to mixtures of more than two Gaussians with different covariance matrices is 
technically difficult but we could expect a similar exponential growth in the square distance
between modes since it depends mainly on the order of magnitude of the local minima of the
distribution density.
As a by product of \Cref{thm:perfectapprox}, we also get the following result which shows that (in the
one-dimensional case) the optimal transport map for the $\ell_2$ cost minimizes
the Lipschitz constant of the push-forward mapping.

\begin{corollary}\label{coro:mongemap}
  Let $ \nu $ be a probability measure on $ \rset $ with density w.r.t.  the
  Lesbesgue measure and such that $ \textup{supp}(\nu) = \rset $. Assume that
  there exists $ g: \rset^p \rightarrow \rset $ Lipschitz such that
  $ \nu = g_\#\mu_p $. Let us denote
  $ T_{\textup{OT}} = \Phi^{-1}_{\nu} \circ \Phi $ the Monge map between $ \mu_1 $
  and $ \nu $, where $ \Phi_{\nu} $ is the cumulative distribution function of
  $ \nu $. Then we have  $\mathrm{Lip}(g) \geq \mathrm{Lip}(T_{\textup{OT}})$. 
\end{corollary}

To the best of our knowledge, extending this proposition to the case where
$d > 1$ remains an open problem. We show now that Equation \eqref{eq:large_r} of 
\Cref{thm:perfectapprox} allows to derive lower bounds on 
similarity measures between the push-forward measure
and the target distribution.

\subsection{Lower bounds on similarity measures between probability distributions}
\label{sec:lower-bound-distance}
Equation \eqref{eq:large_r} provides a bound on the 
minimal mass the push-forward measure $ g_{\#}\mu_p $ can have on a given set 
when $ g $ is fixed with Lipschitz constant $ \mathrm{Lip}(g) $. 
As a consequence, if $ \nu $ is a distribution
such that there exists sets on which $ \nu $ has less mass than 
the minimal quantity that $ g_{\#}\mu_p $ can reach on those sets
given the value of $ \mathrm{Lip}(g) $, then $ g_{\#}\mu_p $ cannot 
be equal to $ \nu $, implying that most of 
similarity measures between $ g_{\#}\mu_p $ and $ \nu $ 
will be automatically strictly positive. In the following, we consider 
that $ g $ and $ \nu $ are fixed and
we derive lower bounds on the total variation distance and the Kullback-Leibler
divergence between $ g_{\#}\mu_p $ and $ \nu $. 
We recall that the total variation distance between two probability measures on
$\rset^d$, $\nu_0, \nu_1$ is given by

\begin{equation}
d_{\mathrm{TV}}(\nu_0, \nu_1)  = \sup \ensembleLigne{\nu_0(\msa) - \nu_1(\msa)}{\msa \in \mcb{\rset^d}} \eqsp . 
\end{equation}
Similarly, we define the Kullback-Leibler divergence between two
probability measures on $\rset^d$, $\nu_0, \nu_1$, using the Donsker-Varadhan
representation \cite[Lemma 1.4.3a]{dupuis2011weak}:
\begin{equation}
\textstyle{
d_{\mathrm{KL}}(\nu_0||\nu_1)  = \sup \ensembleLigne{ \int_{\rset^d} f(x)  \rmd \nu_0(x) - \log\left(\int_{\rset^d} \exp[f(x)] \rmd \nu_1(x)\right)}{f \in \mathfrak{B}(\rset^d,\rset)} \eqsp ,
  }
\end{equation}

where $ \mathfrak{B}(\rset^d,\rset) $ denotes the set of all bounded mappings from $ \rset^d $ to $ \rset $.
In the following, we will denote for any $ \msa \in \mcb{\reel^d} $ and
$ r > 0 $, 
\begin{align}
    &\alpha_g(\msa,r) = \Phi\left(r/\mathrm{Lip}(g) + \Phi^{-1}(g_{\#}\mu_p(\msa))\right) \eqsp , \\
    & \beta_g(\msa,r) = \alpha_g(\msa,r) - g_{\#}\mu_p(\msa) \eqsp , 
\end{align}
where $\alpha_g(\msa,r)$ and $\beta_g(\msa,r)$ are the  lower bounds of $ g_{\#}\mu_p(\msa_r) $  and $ g_{\#}\mu_p(\msa_r \setminus \msa) $ provided by \Cref{thm:perfectapprox}. 
We start by proving lower bounds on the total variation distance.

\begin{theorem}\label{thm:lowerboundtv}
Let $ \nu $ be a probability measure on $ \rset^d $ and let
 $ g : \rset^p \rightarrow \reel^d $ be a Lipschitz function. Then,
\begin{equation}
\resizebox{0.94\hsize}{!}{$\textstyle{d_{\mathrm{TV}}(g_\#\mu_p,\nu) \geq \sup \ensembleLigne{\alpha_g(\msa,r) - \min\{g_{\#}\mu_p(\msa),\nu(\msa)\} - \nu(\msa_r \setminus \msa)}{\msa \in \mcb{\rset^d}, r > 0} \eqsp .}$}
\label{eq:boundtv}
\end{equation} 
\end{theorem}

\begin{proof}[Sketch of Proof]
The proof of this result consists in bounding from below the total variation 
distance on one hand by $ |g_{\#}\mu_p(\msa_r \setminus \msa) - \nu(\msa_r \setminus \msa) | $ 
and on the other hand by $ |g_{\#}\mu_p(\msa_r) - \nu(\msa_r) | $ 
for a given $ \msa \in \mcb{\reel^d} $ and a given $ r > 0 $, and then applying \Cref{thm:perfectapprox}.
\end{proof}

Observe that \eqref{eq:boundtv} always holds but the right-hand term may become negative if
the Lipschitz constant of $ g $ is large enough. The main idea behind this bound
is to find a set $ \msa $ and a real $ r > 0 $ such that 
$ \nu $ has a lot of mass on $ \msa $ but almost no mass on $ \msa_r \setminus \msa $. 
For instance, if $ \nu $ is a distribution on two disconnected manifolds $ \msm_1 $ 
and $ \msm_2 $, an optimal choice for $ \msa $ would either be $ \msm_1 $ or $ \msm_2 $  
and the optimal $ r $ would be the distance between the two manifolds. 
Using \Cref{thm:lowerboundtv}, 
one can derive smaller but more explicit lower bounds only depending on $ \nu $ and the Lipschitz 
constant of $ g $. As a first example, we derive an explicit lower bound 
in the case where $ \nu $ is a bi-modal distribution on two 
disconnected manifolds.

\begin{corollary}\label{coro:discomanifold}
Let $ \nu $ be a measure on $ \rset^d $ on two disconnected manifolds $ \msm_1 $ and $ \msm_2 $ 
such that $ \nu(\msm_1 ) = \lambda $ and $\nu(\msm_2) = 1 -  \lambda $, with $ \lambda \in \coint{1/2,1} $, and let
$ g : \rset^p \rightarrow \reel^d $ be a Lipschitz function. Then,
\begin{equation}\label{eq:lowbounddiscomanifold}
\textstyle{d_{\mathrm{TV}}(g_\#\mu_p,\nu) \geq \int_{\Phi^{-1}(\lambda)}^{d(\msm_1,\msm_2)/ 2 \mathrm{Lip}(g) + \Phi^{-1}(\lambda)} \varphi(t)dt\eqsp , } 
\end{equation}
where $ d(\msm_1,\msm_2)  = \inf \ensembleLigne{\normLigne{m_1 - m_2}}{m_1 \in \msm_1, m_2 \in \msm_2} $.
\end{corollary} 

As a second example, we also get an explicit lower bound in the case where $ \nu $ is a 
mixture of two isotropic Gaussians. For simplicity we stick to the balanced case. 

\begin{corollary}\label{coro:gm2}
Let $ \nu = (1/2)[\mathrm{N}(m_1,\sigma^2\Id_d) + \mathrm{N}(m_2,\sigma^2\Id_d)] $ 
with $ m_1,m_2\in \rset^d $ and $ \sigma \geq 0 $. Let
$ g : \rset^p \rightarrow \reel^d $ be a Lipschitz function. Then,
\begin{equation}\label{eq:lowboundgm}
\textstyle{ d_{\mathrm{TV}}(g_\#\mu_p,\nu) \geq \int_0^{\normLigne{m_2-m_1}/ 4\sigma\mathrm{Lip}(g)} \varphi(t) \rmd t  - (1/2)\int_{\normLigne{m_2 - m_1}(2\sigma - 1)/4\sigma^2}^{\normLigne{m_2 - m_1}(2\sigma  + 1)/4\sigma^2}\varphi(t) \rmd t \eqsp . }
\end{equation}
\end{corollary}

In both corollaries, the lower bound tends to $ 1/2 $ when the distance between the modes tends to infinity,
meaning that $ g_{\#}\mu_p $ is far from well approaching $ \nu $. 
Note that the lower bound exhibited in \Cref{coro:discomanifold} is always strictly positive regardless of the 
value of the Lipschitz constant of $g$. One can also observe that this latter bound is maximal in
the balanced case, when $ \lambda = 1/2 $, since the standard normal distribution concentrates its 
mass around $ 0 $.
Finally, we end this section by deriving a similar lower bound on the Kullback-Leibler divergence between
$ g_{\#}\mu_p $ and $ \nu $. We consider the Kullback-Leibler divergence since this is a measure of similarity between measures which is bounded and is very sensitive to the mismatch of supports between the generated and the data distributions.

\begin{theorem}\label{thm:lowerboundkl}
Let $ \nu $ be a probability measure on $ \rset^d $ and 
let $ g : \rset^p \rightarrow \rset^d $ be a Lipschitz function. Then,
\begin{equation}
\resizebox{\hsize}{!}{$
\textstyle{
d_{\mathrm{KL}}(g_{\#}\mu_p||\nu) \geq \sup \ensembleLigne{\beta_g(\msa,r)\log\left(\frac{\beta_g(\msa,r)}{\nu(\msa_r\setminus \msa)}\right) + \left(1 - \beta_g(\msa,r)\right)\log\left(\frac{1 - \beta_g(\msa,r)}{1 - \nu(\msa_r \setminus \msa)}\right)}{\msa \in \mcb{\rset^d}, r > 0} \eqsp .
}$}
\label{eq:boundkl}
\end{equation}
\end{theorem}

\begin{proof}[Sketch of Proof]
The proof of this result consists in setting $ f = \zeta\chi_{\msa_r \setminus \msa} $ with $ \zeta > 0 $ and  where $ \chi_{\msa} $ is the characteristic function of the set $ \msa $ and plugging it in the Donsker-Varadhan representation in order to get a lower bound depending on
the probability $ g_{\#}\mu_p(\msa_r \setminus \msa) $  for a given $ \msa $, a given $ r $ and a given $ \zeta $.  
Then we apply \Cref{thm:perfectapprox} and we derive the optimal value of $ \zeta $. 
\end{proof}

As above, this bound always holds but the right-hand term becomes negative 
if $ \mathrm{Lip}(g) $ is large enough. As for \Cref{thm:lowerboundtv}, the main idea
is to find a set $ \msa $ and a real $ r $ such that $ \nu $ has a lot of mass on $ \msa $, but
$ \nu $ has almost no mass on $ \msa_r \setminus \msa $. Observe that if
$ \nu(\msa_r \setminus \msa) $ tends to $ 0 $, the left-hand term of the bound tends to infinity. This is coherent with the behavior of the Kullback-Leibler divergence. Similarly to \Cref{coro:gm2}, we also get an explicit 
lower bound in the case where $ \nu $ is a mixture of two isotropic 
Gaussians. As for \Cref{coro:gm2}, we stick to the balanced case for simplicity. 

\begin{corollary}\label{coro:kl_gm}
Let $ \nu = (1/2)\left[\mathrm{N}(m_1,\sigma^2\Id_d) + \mathrm{N}(m_2,\sigma^2\Id_d) \right]$ with $ m_1, m_2 \in \rset^d $ and $ \sigma \geq 0 $. Let $ g : \rset^p \rightarrow \rset^d $ be a Lipschitz function. We denote  
\begin{equation} 
\lambda = g_{\#}\mu_p\left(\ensembleLigne{(m_2-m_1)^T\left(x - (m_2+m_1)/2\right) \leq 0}{x \in \reel^d}\right) \eqsp, 
\end{equation} and we suppose without loss of generality, that $ \lambda \in \ocint{0,1/2} $.
Then,
\begin{equation}
d_{\mathrm{KL}}(g_{\#}\mu_p,\nu) \geq 
\textstyle{A\log\left(\frac{A}{B}\right) + (1-A)\log\left(\frac{1-A}{1-B}\right)}\eqsp ,
\end{equation}
where
\begin{equation}
A = \textstyle{\int_{-\Phi^{-1}(1-\lambda)}^{\normLigne{m_2-m_1}/ 4\sigma\mathrm{Lip}(g) - \Phi^{-1}(1-\lambda)} \varphi(t) \rmd t} \eqsp \text{, and }  B = \textstyle{(1/2)\int_{\normLigne{m_2 - m_1}(2\sigma - 1)/4\sigma^2}^{\normLigne{m_2 - m_1}(2\sigma  + 1)/4\sigma^2}\varphi(t) \rmd t} \eqsp.
\end{equation}
\end{corollary}

Observe that this time, $ \Lip(g) $ is no longer the only dependency in $ g $ 
since the bound also depends on the proportion of the modes of $ g_{\#}\mu_p $.
However, it should be noted that when $ g_{\#}\mu_p $ approximates correctly $ \nu $, $ \lambda $ is automatically close to $ 1/2 $  and so $ \Phi^{-1}(1-\lambda) $ is small in that case. To conclude, this section, we highlight the fact that, if our results are dimension free in theory, the dimension might be hidden in the distances between modes and the Lipschitz constant of $ g$ when working with real datasets. Indeed, the order of magnitude of the Euclidean distance between two samples $ x_i $ is likely to increases with the dimension $ d $. As an outcome, the orders of magnitude of the distance between modes and so the Lipschitz constant that $ g $ must reach for correct generation probably increase with $d$ also. 

\section{Experiments}
\label{sec:experiments}

In what follows, we illustrate the pratical implications of our results by training 
GANs, VAEs and SGMs on simple bi-modal distributions. More precisely, we show on one hand that generating multimodal distributions with GANs and VAEs is difficult since for those models, good generation necessarily involves generative networks with large Lipschitz constants. On the other hand, we show that SGMs seem to be able to generate multimodal distributions while keeping the Lipschitz constant of the score network relatively small and thus do not suffer of the same limitation.
First, we focus on the univariate case where we can easily assess
the Lipschitz constants of the networks.
Then we illustrate our results in higher dimensions by training the three models
on datasets derived from MNIST \cite[]{lecun1998gradient}.
%and show that models seem relatively behave in the same way as in the univariate case. 
In all our experiments, we use the same architecture for the VAE decoder 
and the GAN generator in order to offer rigorous comparisons of
the different models. For score-based modeling, we use 
architectures with similar numbers of learnable parameters. All details on the experiments and architecture of the networks can be found in Appendix \ref{sec:expedetails}.

\subsection{Univariate case}
First, we train a VAE and a GAN with one-dimensional latent spaces 
on $ 50000 $ independent samples drawn from
a balanced mixture of two univariate Gaussians $ \nu = (1/2)[\mathrm{N}(-m,1) + \mathrm{N}(m,1)] $
for different values of $m > 0$. We also train a SGM on the same samples. 
%We generate  samples in order to compare histograms of generated distributions with the target distribution densities. 
%The experiments are averaged over $ 10 $ runs and we plot the standard deviation as a measure of the training stability of the models. 

\paragraph{Histograms of generated distributions.}
\Cref{fig:histo_expe1_dim1} shows histograms of generated distributions 
for $ m = 10 $ with the three different models. VAE models seem to generate Gaussians modes but interpolate significantly between them, while GANs do not 
interpolate but fail to retrieve the structure of the target distribution
and forget parts of their support, which is known as \emph{mode collapse} and is a common pitfall of such models \cite[]{arjovsky2017towards,metz2017unrolled}.
On the same task, SGMs do not suffer from such shortcomings. In the following section, we will link the interpolation/mode-collapsing properties of these models with their Lipschitz constants.

\begin{figure}[h!]
  \centering
  \begin{tabular}{ccc}
  \includegraphics[width=0.26\textwidth]{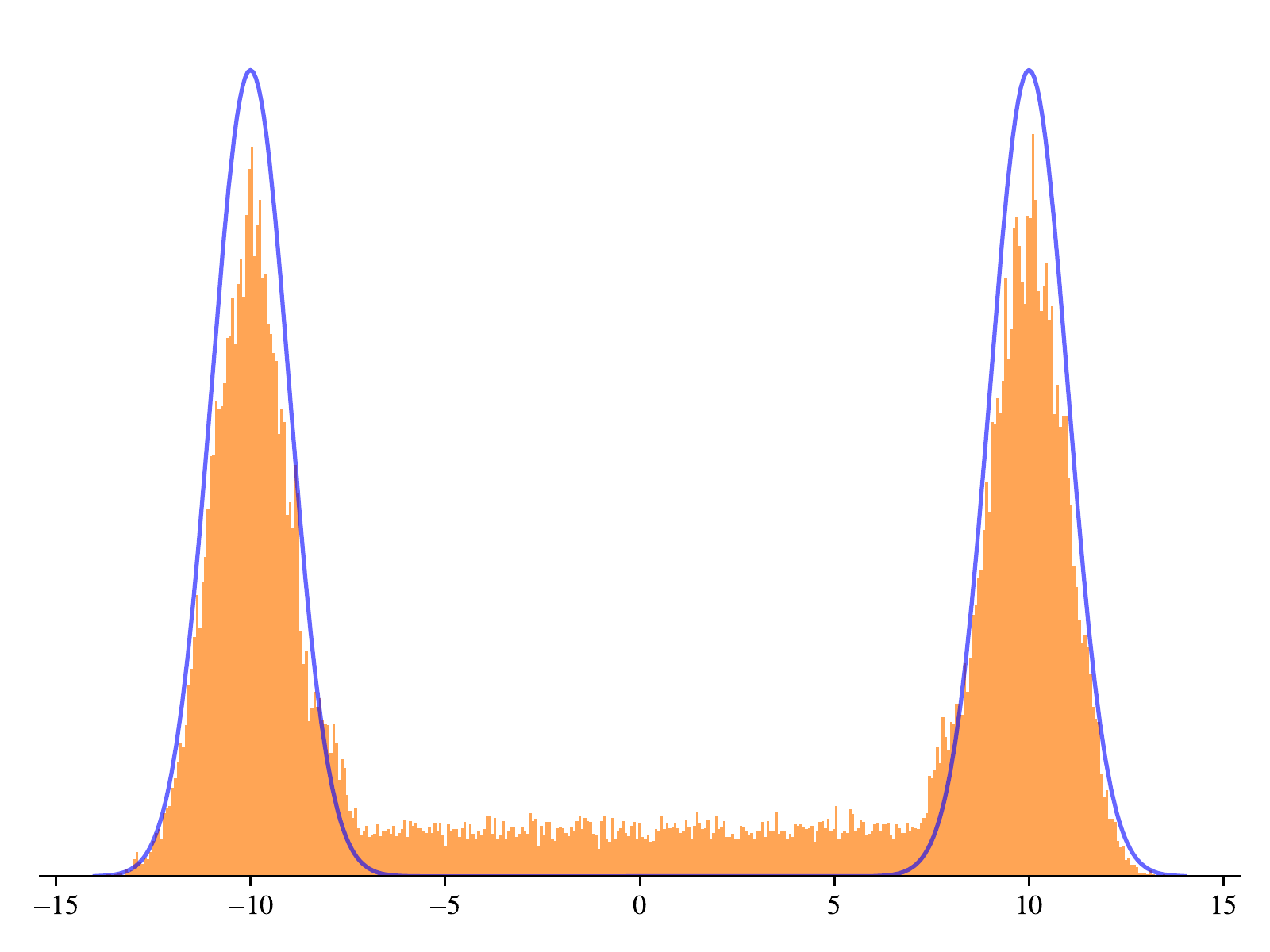}  & 
  \includegraphics[width=0.26\textwidth]{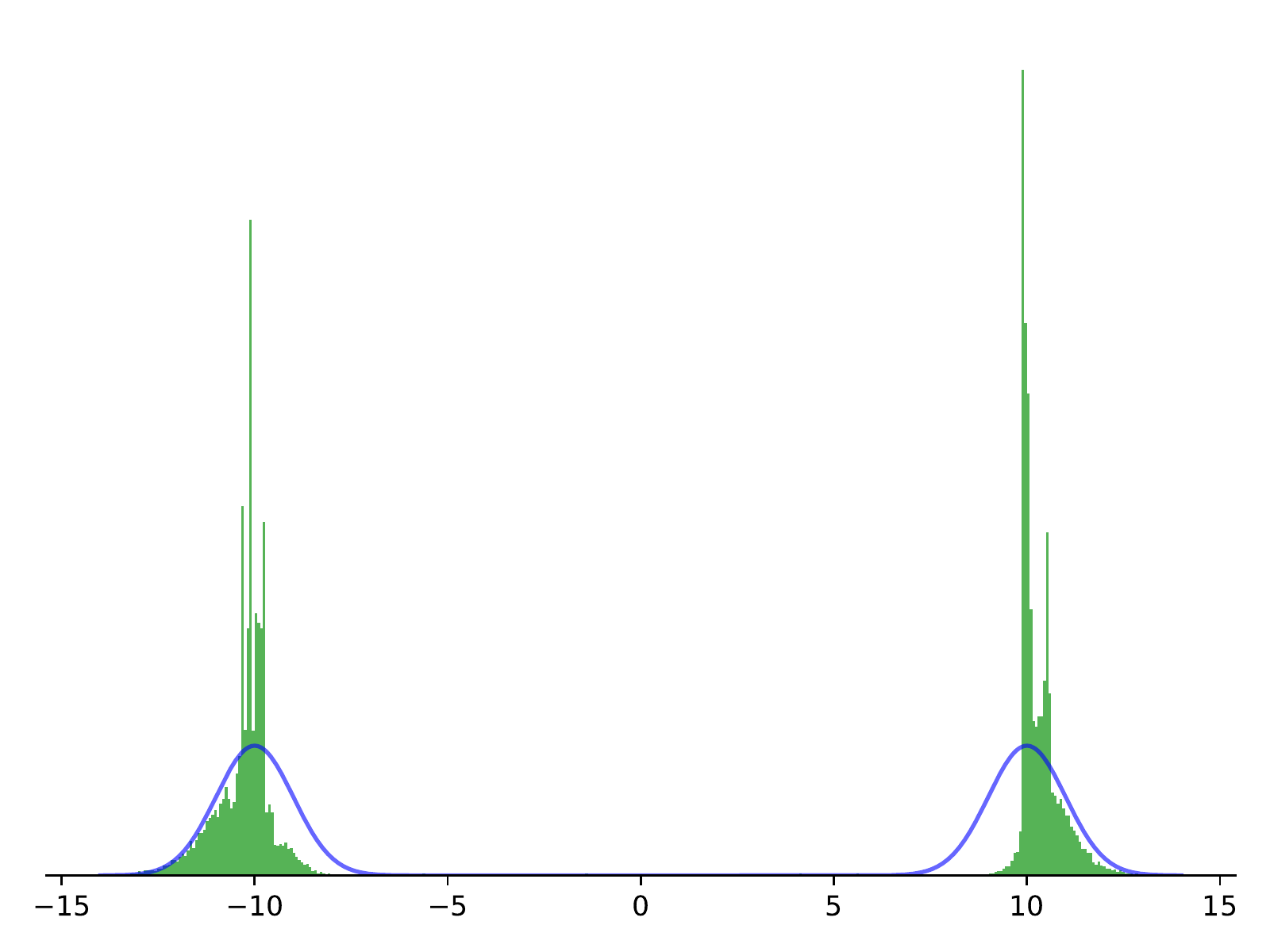}  &
  \includegraphics[width=0.26\textwidth]{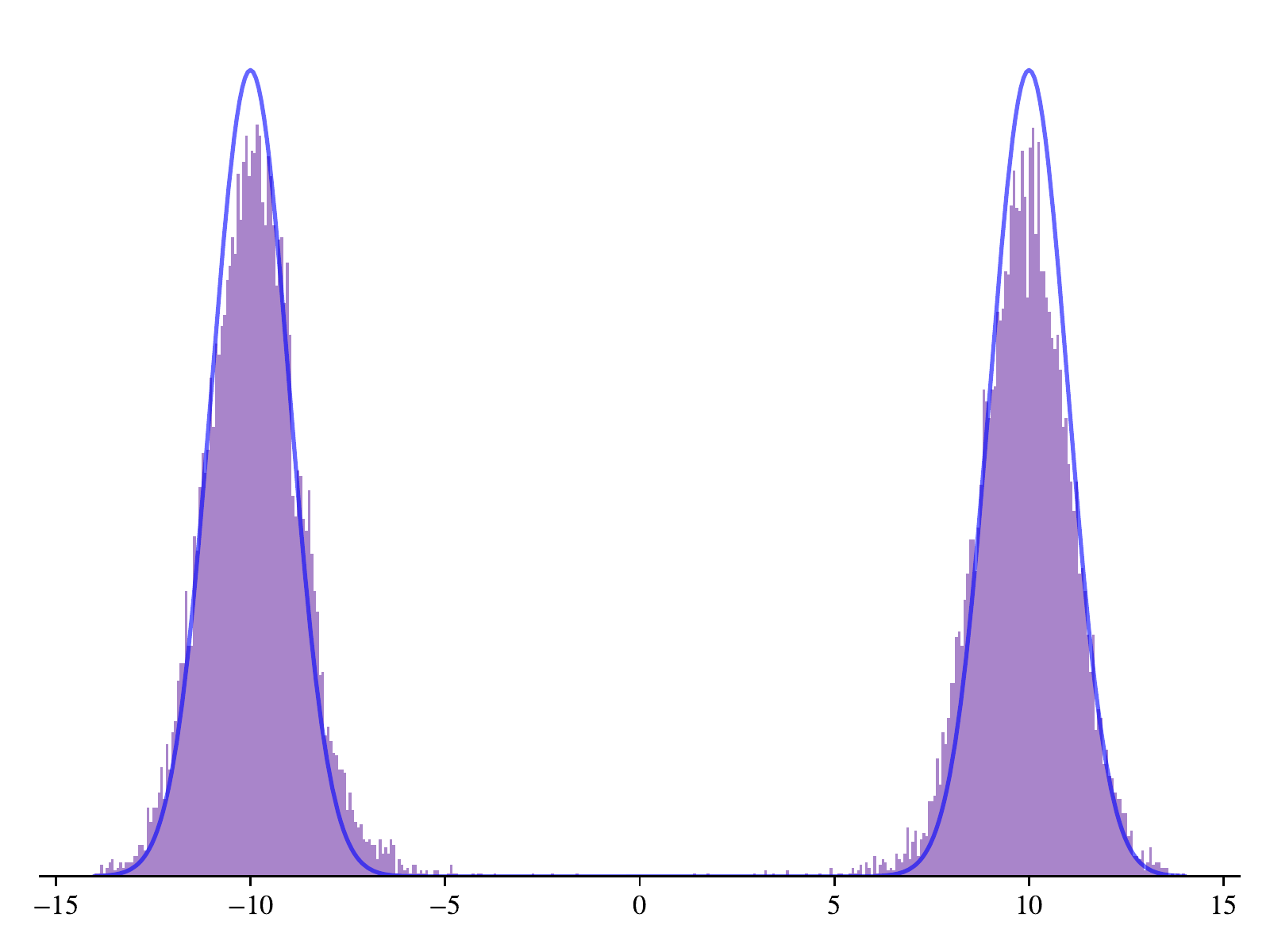}
  \end{tabular}
  \caption{Histograms of distributions generated with VAE (left, in orange), GAN (middle, in green), 
  and with SGM (right, in purple) for  $ m = 10 $. 
  The data distribution densities are plotted in blue.}\label{fig:histo_expe1_dim1}
\end{figure}

\paragraph{Lipschitz constant and mass between modes.}
%\Cref{fig:expe1_dim1} (left) shows the evolution
%of the Lipschitz constants of the three generative models  in function of $ %m $. 
%\Cref{fig:expe1_dim1} (right)  shows the evolution of the mass generated 
%by the models on the interval $ \ccint{-m/2,m/2} $ when $ m $ increases. We % also
%show the lower bounds $ \beta_g(\ocint{-\infty,-m/2},m) $
%of $ \ccint{-m/2,m/2} $  provided by Equation
% \eqref{eq:large_r} of \Cref{thm:perfectapprox} for the VAE and the GAN. 
In \Cref{fig:expe1_dim1} (right), we observe that the GAN generator reaches much larger Lipschitz constants than the VAE decoder. This 
explains the difference of behaviors between GAN and VAE observed in \Cref{fig:histo_expe1_dim1}, as the mapping learned by the VAE is not stiff enough to concentrate the push forward measure on the two modes. 
One possible explanation for the interpolating behavior of the VAE is that the Euclidean norm of the Jacobian of the VAE decoder is implicitly regularized during training, 
as it has been demonstrated in \cite{kumar2020implicit}. 
Both GAN and VAE 
saturate the constraint on $ g_{\#}\mu_p(\ccint{-m/2,m/2}) $ provided by \Cref{thm:perfectapprox}, 
meaning that the generative networks minimize the amount of mass between
modes as much as their Lipschitz constants allow it. Finally, we can observe 
that the score network is able to keep 
a relatively small Lipschitz constant compared to the GAN, while managing to interpolate
less than the latter. 
%When the modes are well separated, the distribution generated
%with the score-based model has almost as little mass as $ \nu $. 
A probable explanation for this follows from the fact 
that the score network is used multiple time during inference. Hence,
the Lipschitz constant of the push-forward mapping (the whole generation dynamic) is likely
much larger than the Lipschitz constant of the neural network itself, and
so the model is able to push-forward a Gaussian distribution into a multimodal
distribution keeping a relatively small Lipschitz constant of the score
network. Finally, in \Cref{fig:expe1_dim1} (left), we observe that  when $ m $ increases, 
the Lipschitz constant of the VAE decoder and the GAN generator
becomes rapidly much smaller than the value of the lower bound 
provided by \Cref{coro:gm1}. This means that for $m$ large enough it is not possible to close the gap between the data distribution and the push-forward distribution. We highlight that this observation does not apply to SGMs since in this setting the network is applied multiple times.
%Another possible explanation is that it is a consequence of the fact that in VAEs, 
%the distribution of the training data at the output of the encoder
%in the latent space is often far from the standard normal distribution 
%used during inference, as it was pointed in \cite{dai2018diagnosing}. 
%Hence, when the true data distribution is multimodal, the latent training data distribution in input of the VAE decoder 
%is likely to be multimodal as well. As an outcome, 
%the decoder doesn't have to reach a large Lipschitz constant in order to separate the modes during training, 
%since the modes are already partly separated. 
%This is not the same for the GAN generator which has to separate a single mode into several different
%modes during training. 

\begin{figure}[h!]
  \centering
  \scalebox{0.95}{
    \begin{tabular}{cc}
    \hspace{1.4em} \vspace{-0.2 em} {\ssmall Lipschitz constants} & \hspace{-1.5em} \vspace{-0.2 em} \hspace{2em} {\ssmall \quad Mass between modes} \\
    \hspace{-1.8em} 
     \includegraphics[width=0.46\textwidth]{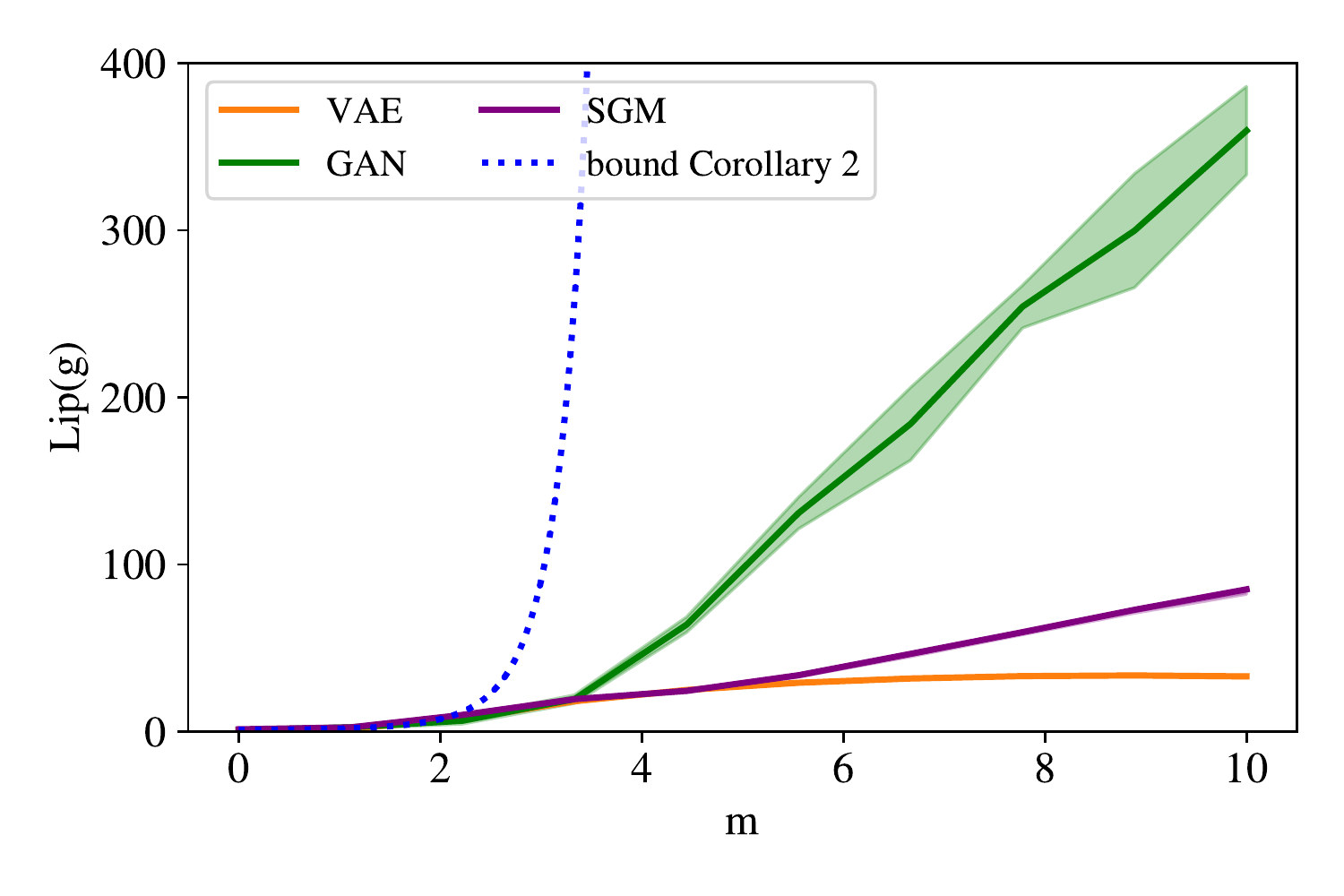}   \hspace{-1.5em} \vspace{-0.5 em}
    & 
    \includegraphics[width=0.46\textwidth]{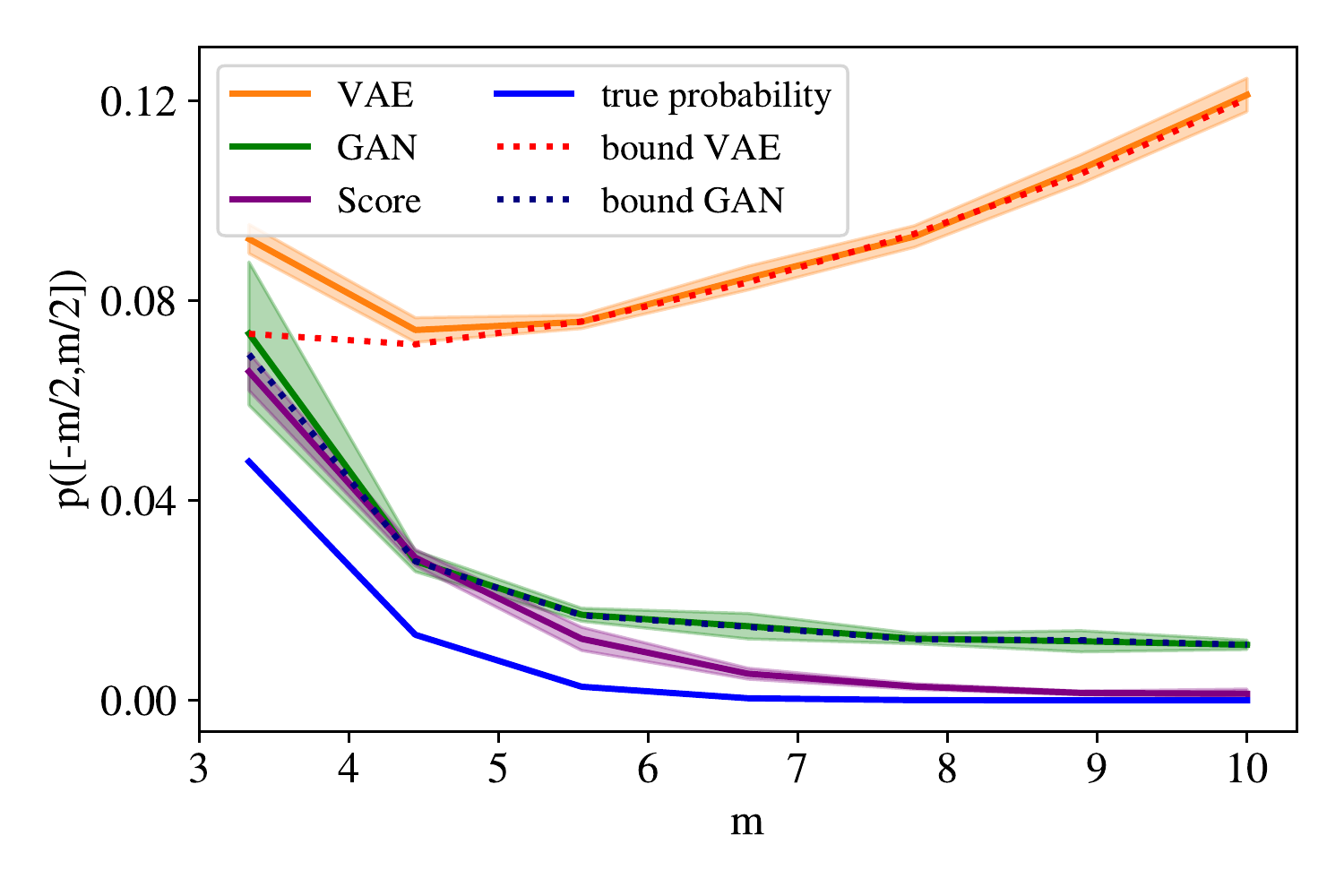}    \vspace{-0.5 em}
    \end{tabular}}
    \caption{Left: evolution of the Lipschitz constants of the three different generative models
    trained on 50000 samples of $(1/2)[\mathrm{N}(-m,1) + \mathrm{N}(m,1)] $, in function of $m$. Right: evolution of  the proportion of samples generated by the three models on the interval $ \ccint{-m/2,m/2} $. We also show on this graph
    the lower bounds predicted by \Cref{thm:perfectapprox} for the VAE and the GAN, as well as  the
    true probability $ \nu(\ccint{-m/2,m/2}) $. Experiments are 
    averaged over $ 10 $ runs and the colored bands correspond to +/- the standard deviation.}\label{fig:expe1_dim1}
\end{figure}

\paragraph{Stability of GAN and mode collapse.} \cite{odena2018generator} suggested
that the magnitude of the norm of the generator jacobian may be causally related 
to instability and mode collapse. This is why many state-of-the-art GANs
apply spectral normalization \cite[]{miyato2018spectral} on their generators. In 
\Cref{fig:gan_SN} (left), we show that this technique cannot be used when training 
GANs on multimodal distributions: since spectral normalization constraints
the Lipschitz constant of the generator to be smaller than $ 1 $, the GAN is trained towards concentrating in one of the modes of the distribution over interpolating massively between them. 
This has been referred to as \emph{mode dropping} by \cite{khayatkhoei2018disconnected}. 
To complete this analysis, we also train the GAN adding an additional 
gradient penalty term  $ 10/L^2 \max_{z \sim \mathrm{N}(0,\Id_p)} (\|\nabla_z g_{\theta}(z)\|_2^2 - L)^2$, 
in the generator loss function, similarly to WP-GAN  \cite[]{gulrajani2017improved}, where
 $ L $ is an hyperparameter corresponding to 
the targeted Lipschitz constant. 
As expected, we can observe in \Cref{fig:gan_SN} (right), that when $ \text{Lip}(g) $ increases, the 
GAN begin to generate both modes but becomes also more and more prone to mode collapse. This illustrates the fundamental trade-off between expressivity and robustness in push-forward generative models.

\begin{figure}[h!]
  \centering
  \begin{tabular}{c|ccc}
    \includegraphics[width=0.21\textwidth]{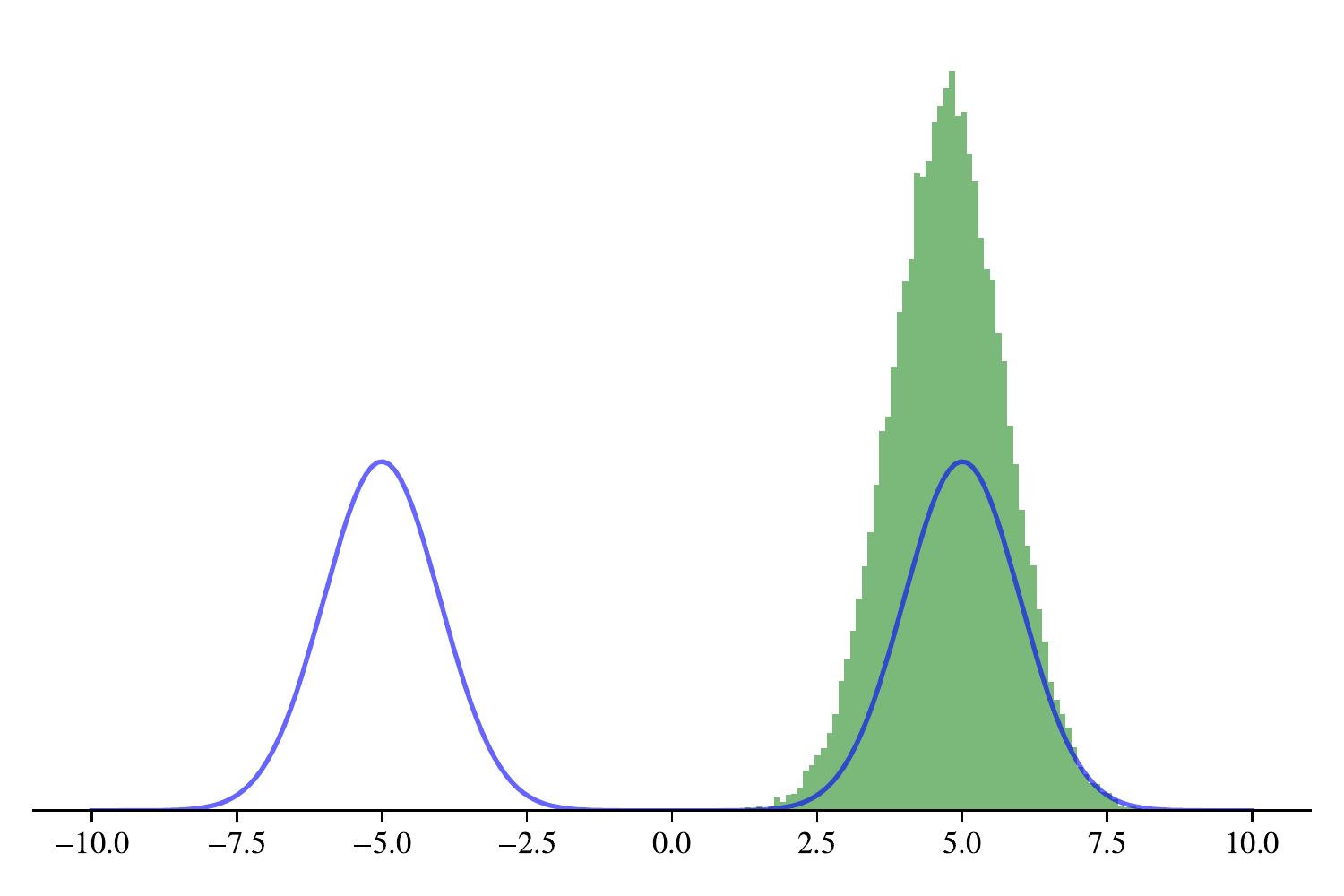} &
    \includegraphics[width=0.21\textwidth]{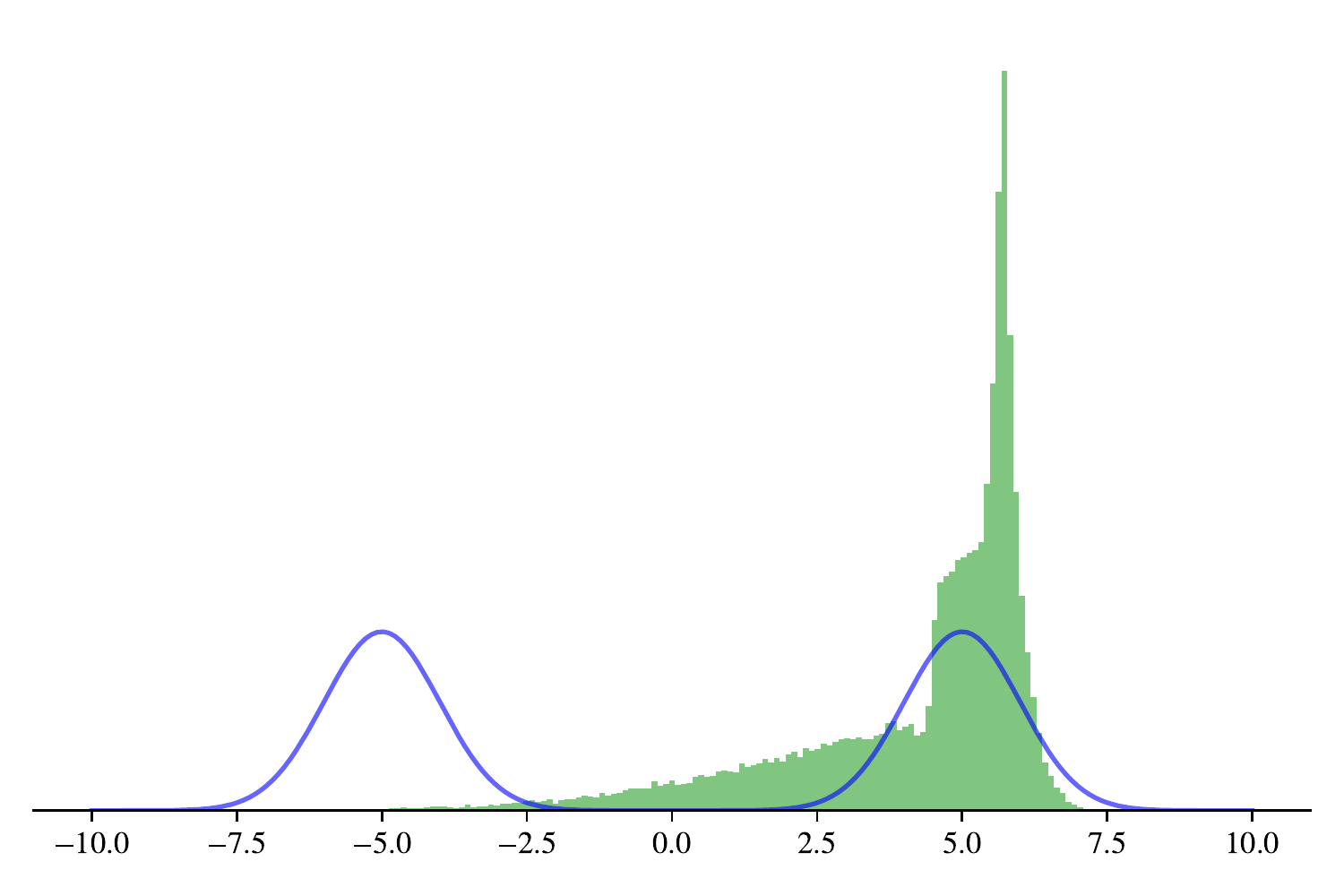} &
    \includegraphics[width=0.21\textwidth]{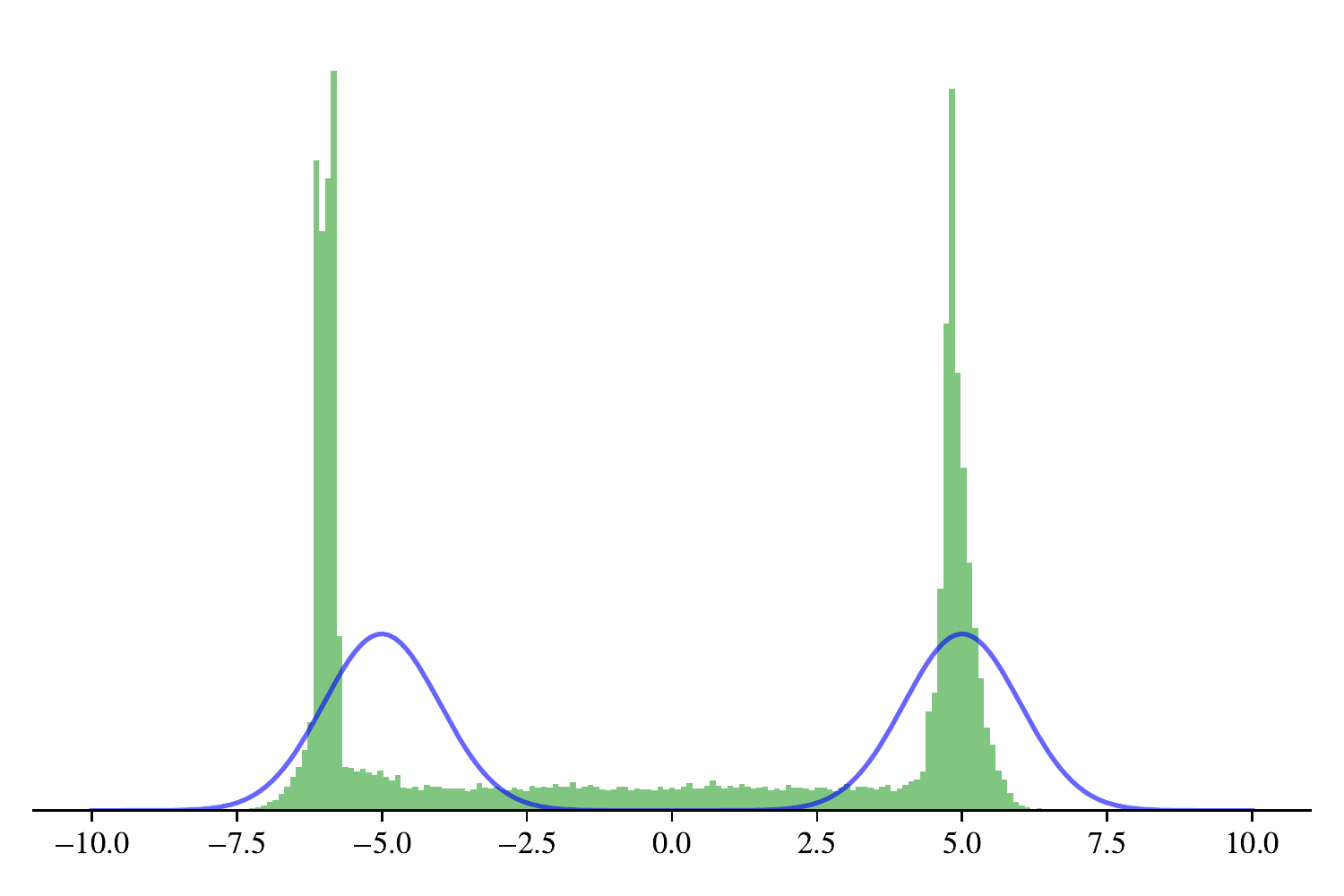} &
    \includegraphics[width=0.21\textwidth]{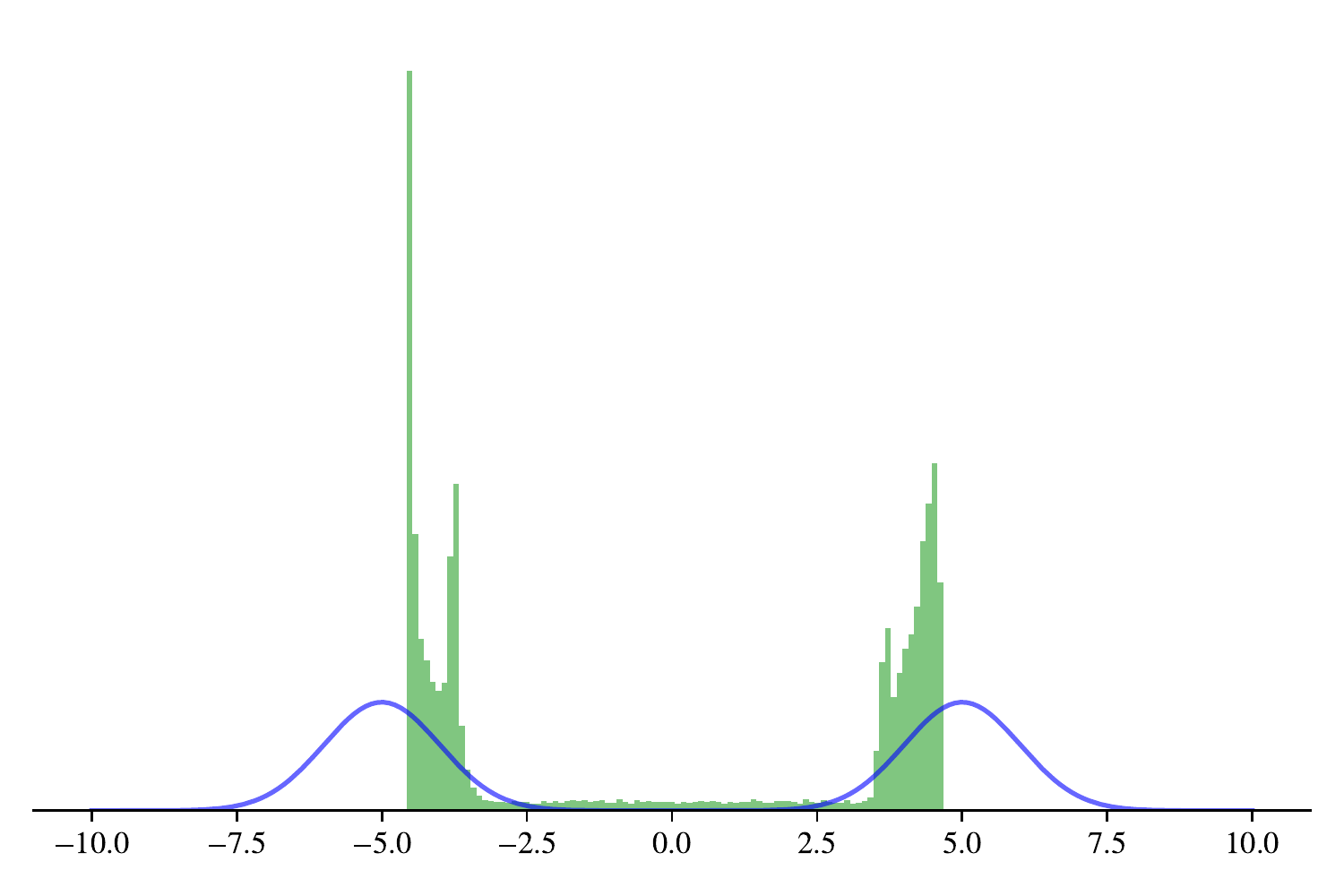}     
  \end{tabular}
\caption{Histograms of distributions generated with GANs with spectral normalization
applied on the generator (left), and with gradient penalty (right) for $ \mathrm{Lip}(g) \approx L  = 5 $,
$ \mathrm{Lip}(g) \approx L  = 15 $ and $ \mathrm{Lip}(g) \approx L = 25 $. The data distribution densities are plotted in blue.}\label{fig:gan_SN}
\end{figure}

\paragraph{Influences of generator depth and time of training.}
In \Cref{fig:n_layer_n_epochs}, we study the effect of increasing the number of layers of the generative network as well as increasing the training time on the value of the Lipschitz constant of the VAE decoder and the GAN generator. In the VAE setting, the Lipschitz constant
increases linearly with the depth of the decoder. This is not the case in the GAN setting, where increasing the size of the model seems to  dramatically affect its stability. For both models, the Lipschitz constants of the generative network grow with the number of epochs. Yet
this growth seems to be logarithmic for the VAE and the GAN seems to becomes
more unstable as the number of epochs increases. 
%This supports the conclusion that increasing the depth and the training time allow for greater expressivity of the model at the cost of stability in the training. 

\begin{figure}[h!]
\centering
  \scalebox{1}{
    \begin{tabular}{cc}

     \includegraphics[width=0.4\textwidth]{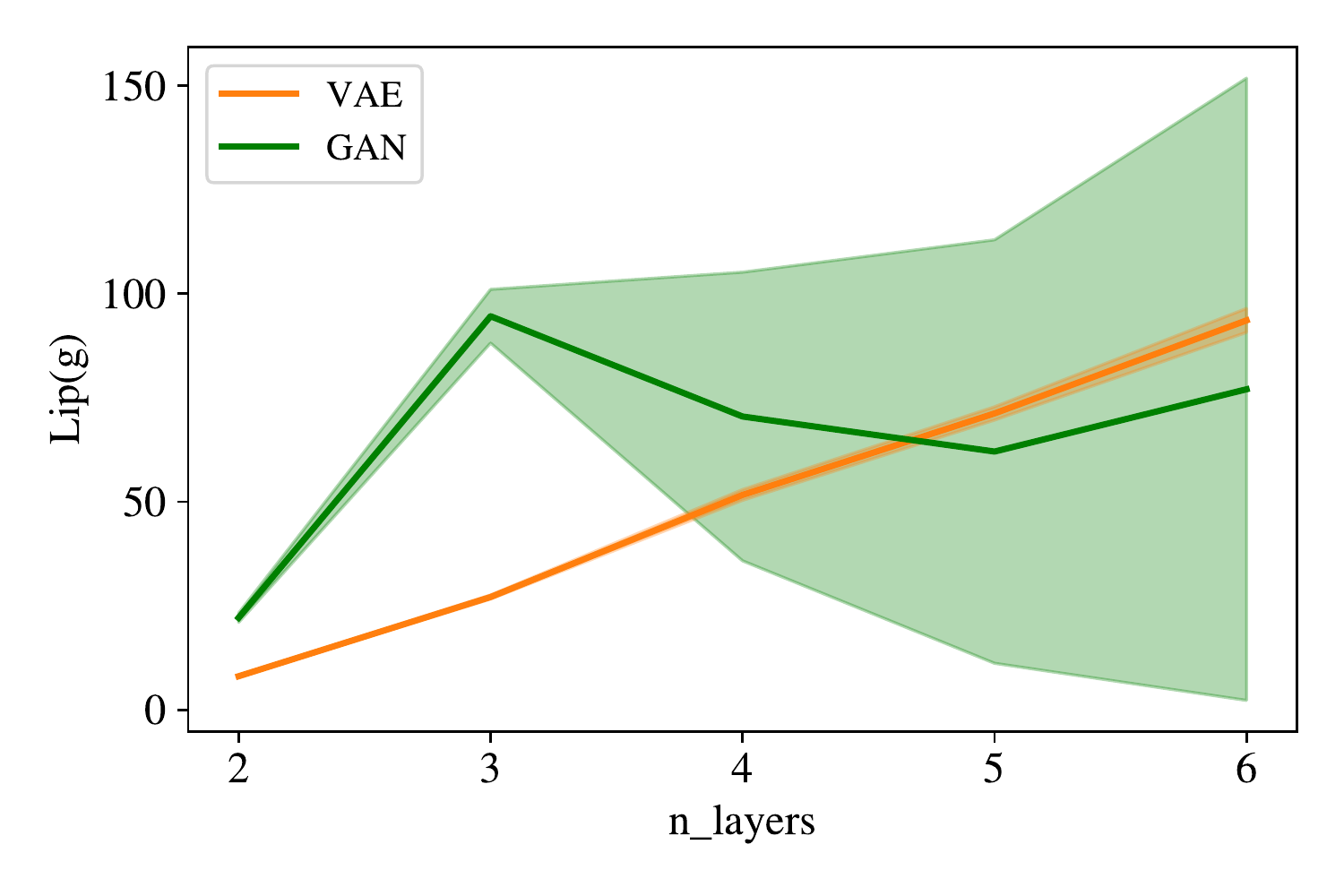}   \hspace{-1.5em} \vspace{-0.5 em}
    & 
    \includegraphics[width=0.4\textwidth]{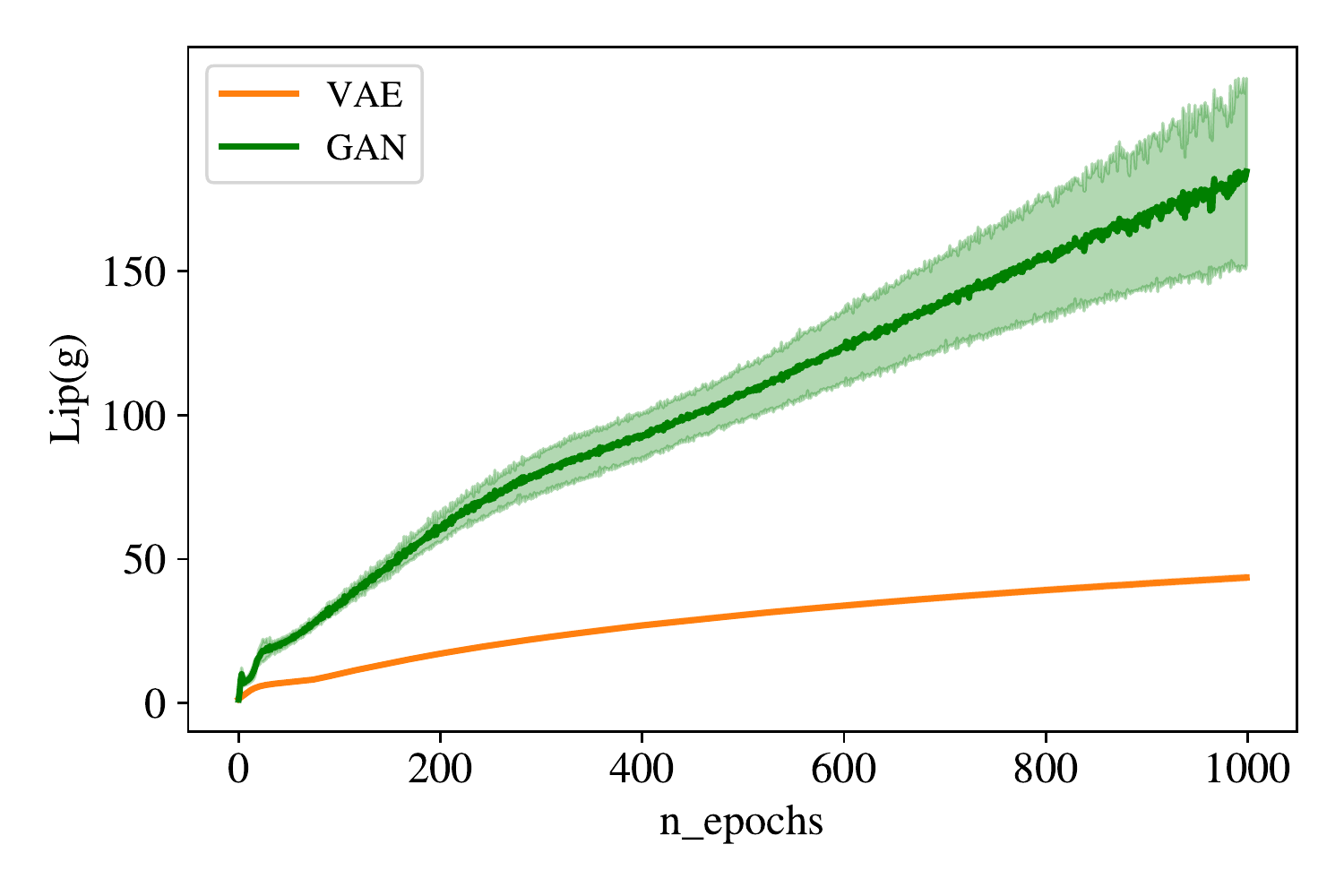}    \vspace{-0.5 em}
    \end{tabular}}
\caption{Evolution of the Lipschitz constant of the generative network with respect to its number of layers (left) and 
of the Lipschitz constant in function of the numbers of epochs (right). The experiments are 
averaged over $ 10 $ runs and the colored bands correspond to +/- the standard deviation.
}\label{fig:n_layer_n_epochs}
\end{figure}

\begin{figure}[h!]
  \centering
  \scalebox{0.95}{
    \begin{tabular}{cc}
    \hspace{1.4em} \vspace{-0.2 em} {\ssmall VAE} & \hspace{-1.5em} \vspace{-0.2 em} \hspace{2em} {\ssmall \quad GAN} \\
    \hspace{-1.8em} 
     \includegraphics[width=0.46\textwidth]{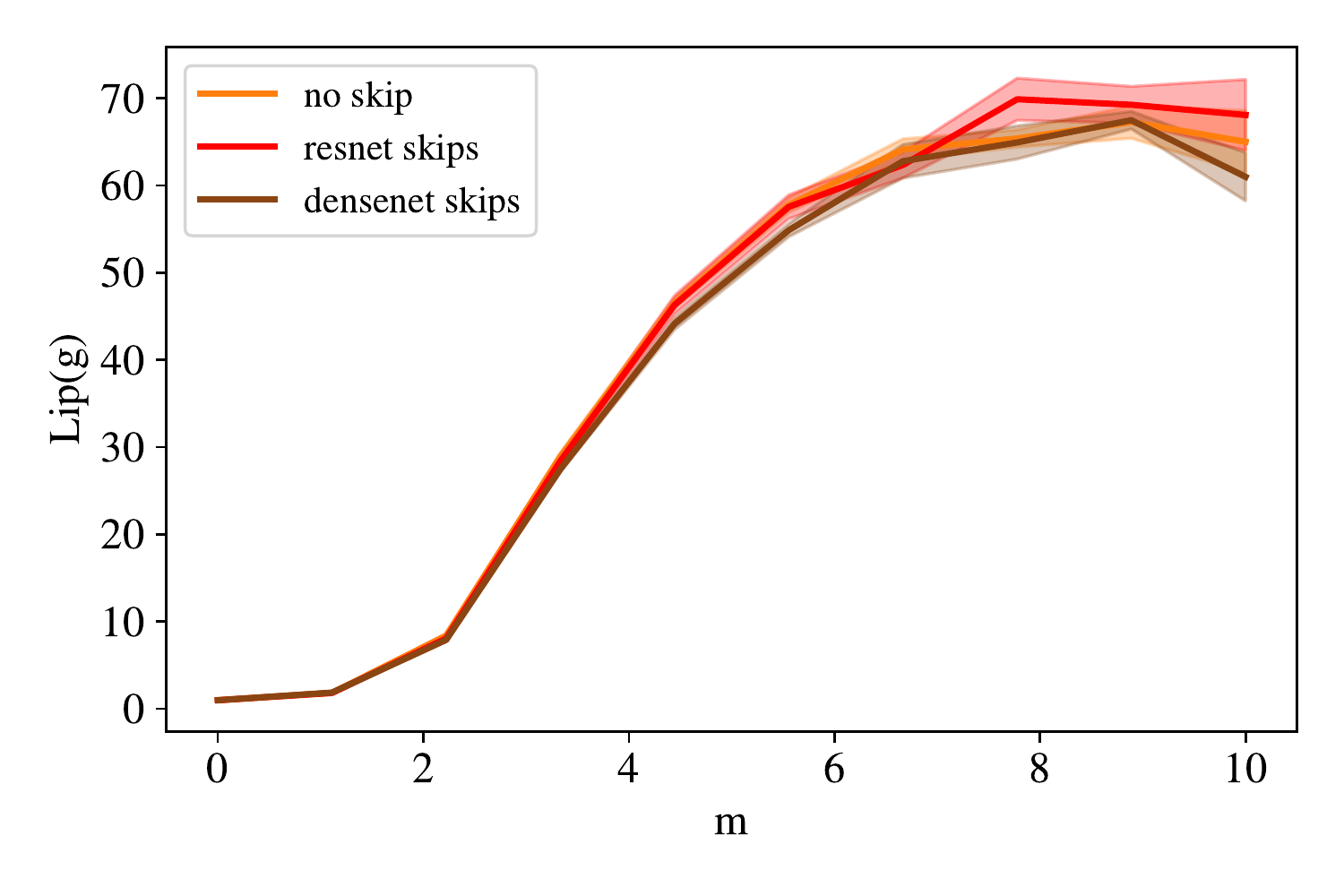}   \hspace{-1.5em} \vspace{-0.5 em}
    & 
    \includegraphics[width=0.46\textwidth]{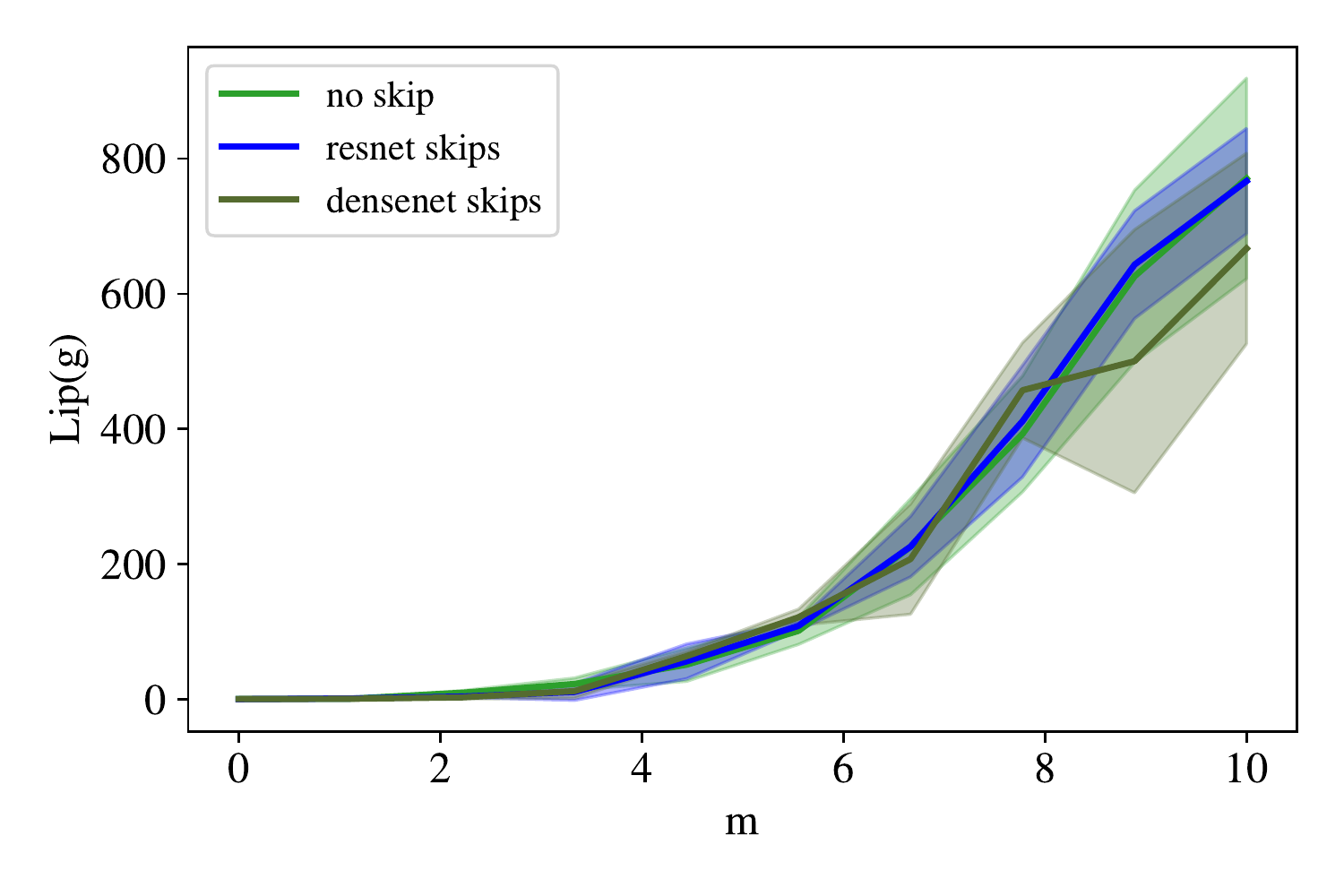}    \vspace{-0.5 em}
    \end{tabular}}
    \caption{Evolution of the Lipschitz constant of the VAE decoder (left) and the GAN generator (right)
    trained on 50000 samples of $(1/2)[\mathrm{N}(-m,1) + \mathrm{N}(m,1)] $ for $ 3 $ different architectures of the generative network: simple feed-forward backbone, backbone with skip-connections of type "resnet", and backbone with skip-connections of type "densenet". Experiments are 
    averaged over $ 5 $ runs and the colored bands correspond to +/- the standard deviation.}\label{fig:expe_archi}
\end{figure}

\paragraph{Influence of generator architecture.}
Finally, we study in \Cref{fig:expe_archi} the impact of the architecture of the generative network (i.e. the VAE decoder and the GAN generator) on its Lipschitz constant as well as on the training stability of the model by comparing three different architectures: first, we use a simple feed-forward network as precedently, then we add additive skip-connections of type "resnet" \cite[]{he2016deep} to the previous backbone, and last we add concatenation skip-connections of type "densenet" \cite[]{huang2017densely} instead of additive skip-connections. For both models, it seems that more expressive decoder architectures do not help to reach larger values of  Lipschitz constant. However, one can observe that in the GAN setting, even if the model remains certainly too unstable for correct distribution generation, adding additive skip-connections seems to stabilize the training a little since the colored bands are narrower than for the two other models. This suggests that some generator architectures may be better than others at learning mappings with large Lipschitz constants while staying stable.

\subsection{Experiments on MNIST}
We train a VAE, a GAN and a SGM on two datasets derived from MNIST \cite[]{lecun1998gradient}:
first, two images of two different digits ($3$ and $ 7$) are chosen and $ 10000 $ noisy versions of theses images are drawn
with a noise amount of $ \sigma = 0.15 $, forming 
a dataset of $ n = 20002 $ independent samples drawn from a balanced mixture of two Gaussian distributions in dimension $ 784 =28 \times 28 $. Second, 
we train the models on the subset of all $ 3 $ and $ 7 $ of MNIST. 
We emphasize that our goal is not reach state-of-the-art performance on this problem but rather to illustrate our theoretical results in a moderate dimensional setting.

\begin{figure}[h!]
  \centering
\includegraphics[width=0.9\textwidth]{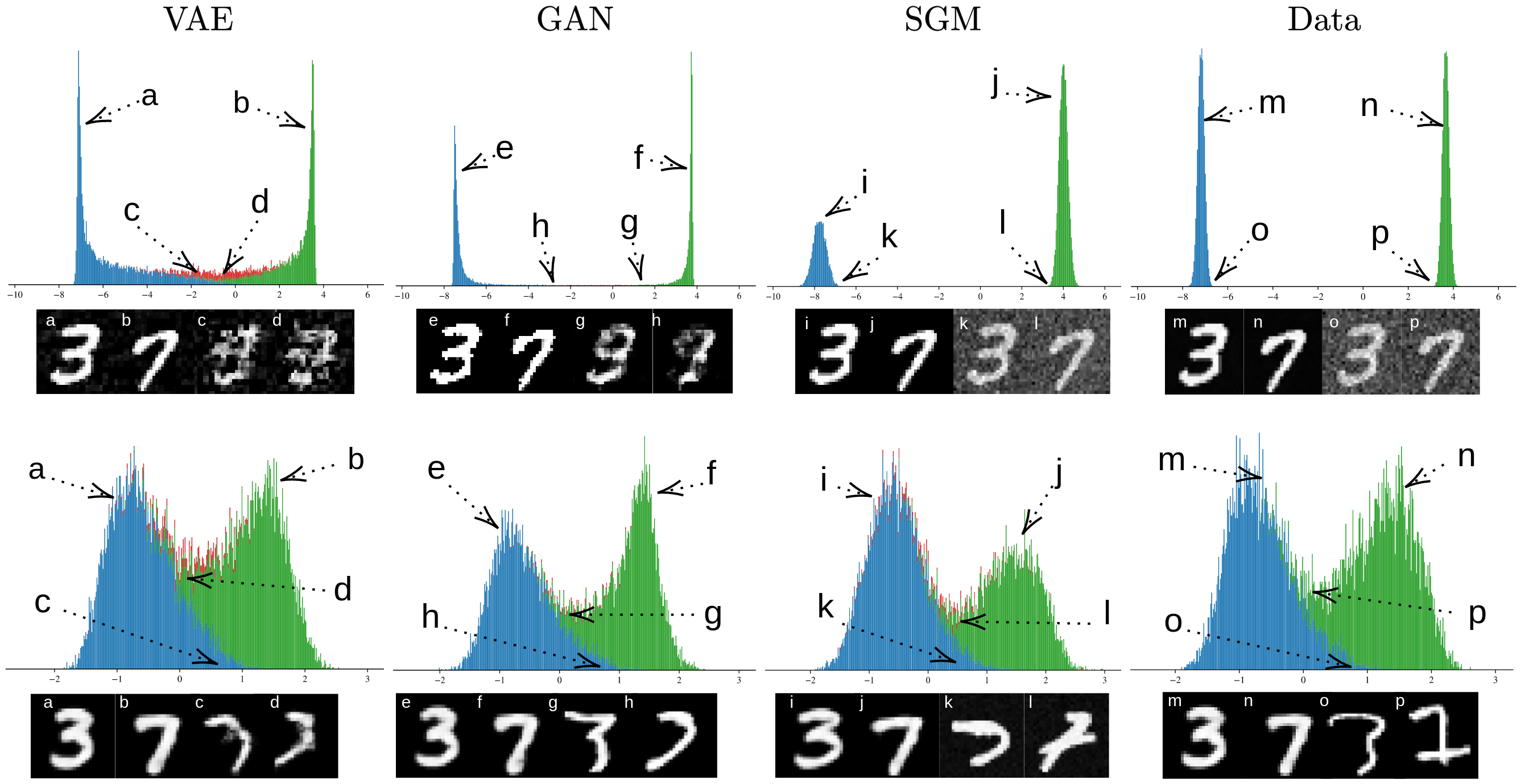}
  \caption{mixture of Gaussians (top): histograms of projections on the line passing through the mean of each Gaussian. 
  Subset of MNIST (bottom): histograms of projections on the line passing through the barycenters of 
  all the $3 $ and $ 7 $ in the deep Wasserstein 
  embedding space. Bins of data are colored in blue if they are classified as $ 3 $, 
  in green if classified as $ 7 $, 
  and in red if classified as another digit.}\label{fig:mnist}
\end{figure}

\paragraph{Mixture of Gaussians.} For 
this experiment, we set the dimension of the  latent space in the GAN and the VAE to $ 784 = 28 \times 28  $ 
since it is the intrinsic dimension of the support of the data distribution.
In order to visualize 
the interpolation between modes, 
we project the data on the line passing through the mean of each Gaussian, 
i.e. the two original clean images, and we plot histograms of the one-dimensional projections. 
In order to understand which bins of data in the histograms correspond to which digit, we
train a classifier and we assign a color in function of which digit the data have been classified as. 
Results can be found in \Cref{fig:mnist} top. Moreover,
GAN and VAE both fail to generate noisy versions of the images. 
As in the univariate case, the SGM is able to not interpolate between modes and seem to retrieve the Gaussian structure of the modes. 
This suggests that while direct push-forward models fail at representing multimodal distributions, considering stacked models with noise input at each step (as in SGM) might help to close the gap between the generated and the data distributions. 
However SGM does not manage to retrieve the 
right modes proportions. This is a well-known shortcoming of score-based models which has been studied in \citep{wenliang2020blindness}. 
%We observed that increasing the size of the score network 
%seems to be a way to reduce this issue (see Appendix \ref{sec:expeplus}).

\paragraph{Subset of MNIST.} 
Finally 
we train the three different models on the subset of MNIST composed of all $ 3 $ and $ 7 $ (no Gaussian noise was added).
%Since it have been estimated by \cite{pope2020intrinsic} 
%that the intrinsic dimension of MNIST was around $ 10 $, 
We choose a latent dimension of $  20 $ for the VAE and the GAN. 
Since the Euclidean distance is not a meaningful metric to compare the different digits of MNIST,
we use the deep Wasserstein embedding proposed by \cite{courty2018learning}: an autoencoder is 
learned in a supervised fashion such that the Euclidean distance in the latent space 
approximates the Wasserstein distance between pairs of images of MNIST. 
In the learned Wasserstein space, 
we project data on the line passing through the Euclidean barycenters of all $3 $ and $ 7 $ 
and plot histograms of projections, using the same classifier
as before. Results can be found in  \Cref{fig:mnist} (bottom). Note that the distribution does not exhibit strong multimodality features contrary to the mixture of Gaussians settings, see Figure \ref{fig:mnist}. As before, the VAE
interpolate between modes, the GAN manages to not interpolate but generate a narrower histogram, 
and the score-based model does not interpolate and seems to recover the structure of
the distribution, but doesn't retrieve the right modes proportions. However, we emphasize that all these models seem to perform better than on the previous dataset. A possible explanation
of this is that the modes are less separated than in the Gaussian mixtures and therefore the model is easier to train.

\section{Discussion}
\label{sec:disc}
In this work, given a Lipschitz mapping $ g $ and a measure $ \nu $, 
we derived lower bounds on the total variation distance 
and the Kullback-Leibler divergence between the push-forward measure $ g_{\#}\mu_p $
and $ \nu $ depending on the Lipschitz constant
of the mapping $g$. 
These bounds indicate 
how the mass between the modes of the push-forward measure depends 
on the Lipschitz constant of the push-forward mapping. They highlight the trade-off between 
the ability of VAEs and GANs to 
fit multimodal distributions and the stability of their training.

A common assumption in the imaging literature, validated empirically by \cite{pope2020intrinsic}, is that distributions of natural images live on low dimensional manifolds. Understanding whether these distributions are composed of separated modes or not remains, to the best of our knowledge, an open problem. To that extent, the fact that unsupervised push-forward generative models perform well on datasets such as CelebA \cite[]{liu2015deep} could possibly be, in regard of our work, an indicator that the data distributions of those datasets are unimodal, or at least not composed of well separated modes.

%The hypothesis of the multimodality of the data distribution made in the theorems is difficult to check on real datasets. Indeed, if the common assumption that distributions of natural images live on very low dimensional manifolds has been validated empirically by \cite{pope2020intrinsic}, understanding whether these distributions are composed of separated modes or not remains to the best of our knowledge an open problem. To that extent, the fact that unsupervised push-forward generative models perform well on datasets such as CelebA \cite[]{liu2015deep} could possibly be, in regard of our work, an indicator that the data distributions of those datasets are unimodal, or at least not composed of separated modes {\color{red} pas sûr de cette phrase}. 

Several techniques have been proposed in the literature to fit 
data distributions on disconnected manifolds. Most of them consist
in overparametrizing the model, either by using stacked generative networks
\cite[]{khayatkhoei2018disconnected,mehr2019disconet} or by learning 
a more complex latent distribution than the standard Gaussian \cite[]{gurumurthy2017deligan,rezende2015variational,kingma2016improved,luise2020generalization}. 
Other methods consist in rejecting a posteriori samples associated 
to large values of the Jacobian generator \cite[]{tanielian2020learning,issenhuth2020learning}.
In this work, we empirically showed that score-based models seemed to be able to fit separated manifolds
without model overparametrization or additional posterior sample rejection scheme. This suggests that the structure 
of the generation dynamic in these models is particularly adapted to (indirectly) learn mappings with large Lipschitz constants. 
Their good performance on multimodal distributions might follow from the fact that these models do not optimize directly the push-forward mapping
itself and/or that noise is injected at each step during the generation process. Hence, a future perspective of work would be to study what are the structural aspects of diffusion models that play a significant role in their expressivity. 
% {\color{red} pas sûr de ca: 
% est ce quand on utilise le solver numérique de la sde (par exemple langevin) pour générer on est
% pas en train de faire un pushforward avec une fonction deterministe dont l'entrée est tous les bruits qui interviennent 
% dans la dynamique?}.
%To that extent, in the context of Normalizing Flows \cite[]{rezende2015variational}, \cite{cornish2020relaxing} showed 
%that injecting stochasticity in the model improved the expressivity of the model. Hence,
%a future perspective of work would be to study how the introduction of stochasticity 
%affect the expressivity of generative models. 

A possible limitation of this work is that  
the bounds derived on the Kullback-Leibler divergence and total variation distance are not tight (see Appendix \ref{sec:expeplus}), mainly because 
they take no account of the fact that when interpolating,
$  g_{\#}\mu_p $ has automatically less mass than $ \nu $ on the modes
since a significant amount of its total mass is between them. In future work, we plan to tighten the gap between our bounds and the true distance.

%+ OTHER DISTANCES LIKE WASSERSTEIN

%While we focused in this paper on the mass generated between modes for multimodal distributions, another interesting direction would be to focus on distribution tails...

%\begin{figure}
%\centering
%\begin{tabular}{cc}
%\includegraphics[width=0.47\textwidth]{figure_neurips/expe_synthetic_gm/mnist_synthetic_gm.png}&
%\includegraphics[width=0.47\textwidth]{figure_neurips/expe_mnist_ 3_7/mnist_3_7.png}
%\end{tabular}
%\caption{}\label{fig:mnist_vae_gan}
%\end{figure}

\bibliographystyle{apalike}
\bibliography{article}

\appendix

\theoremstyle{plain}
\newtheorem{unlemma}{Lemma S}
\newtheorem{unproposition}{Proposition S}
\newtheorem{uncorollary}{Corollary S}
\newtheorem{untheorem}{Theorem S}

\setcounter{equation}{0}
\setcounter{figure}{0}
\setcounter{table}{0}
\makeatletter
\renewcommand{\theequation}{S\arabic{equation}}
\renewcommand{\thefigure}{S\arabic{figure}}
\renewcommand{\thetheorem}{S\arabic{theorem}}
\renewcommand{\thedefinition}{S\arabic{definition}}
\renewcommand{\thelemma}{S\arabic{lemma}}
\renewcommand{\thesection}{S\arabic{section}}
\renewcommand{\theremark}{S\arabic{remark}}
\renewcommand{\theproposition}{S\arabic{proposition}}
\renewcommand{\thecorollary}{S\arabic{corollary}}
\setcounter{tocdepth}{1}

\newpage
\section{Organization of the supplementary}
\label{sec:organ-suppl}

The supplementary is organized as follows. First, in \Cref{sec:diff}, we give details on why the generated distribution in each aforementioned generative model is a push-forward of a standard Gaussian distribution. In Appendix \ref{sec:proof-}, we give the full proofs of all the theoretical results of the paper.  In Appendix \ref{sec:theoplus}, we show a generalization of \Cref{coro:discomanifold} in the case there is more than two disconnected manifolds. In
Appendix \ref{sec:expedetails}, we give details on the experiments.
Finally, in Appendix \ref{sec:expeplus}, we provide additional experimental results and additional visualizations of histograms of
generated distributions for the univariate case, and generated 
data for the experiments on MNIST. 

\section{More details on push-forward generative models}
\label{sec:diff}
In this section, we specify why in each aforementioned model the generated distribution is of the form $ g_{\theta\#}\mu_p $ with $ \mu_p = \mathrm{N}(0,\Id_p) $ being the standard Gaussian distribution in dimension $ p $ and $ g_{\theta} : \rset^p \rightarrow \rset^d$ being a deterministic mapping of parameter $ \theta $. 

\subsection{Direct push-forward models}
In GANs, the generated distribution is trivially of the form $ g_{\theta\#}\mu_p $, where 
$ g_{\theta} $ is the generator and $ p $ is the dimension of the latent space. This is also the case in most of normalizing flow models, where $ g_{\theta} $ is a composition of neural flows and $ p $ is automatically equal to $ d $ since the networks must be invertible. In the Gaussian-VAE model described in \cite{kingma2013auto}, the conditional probability $ p_{\theta}(x|z) $, with $ z \in \rset^p $ and $ x \in \rset^d $ being respectively the latent and observable variables, is of the 
form $ p_{\theta}(x|z) = \mathrm{N}(f_{\theta}(z),h_{\theta}^2(z)\Id_d) $, where $ f_{\theta}(z)  $ and $g_{\theta}) $ are the outputs of the decoder. This is often simplified in practice to $ p_{\theta}(x|z) = \mathrm{N}(f_{\theta}(z),c^2\Id_d) $ with $ c > 0 $ being an hyperparameter of the model. Thus the generated distribution in Gaussian-VAE is of the form
$ g_{\#}\mu_p $, where $ g_{\theta} : \rset^{p+d} \rightarrow \rset^d $ is the neural network defined as
\begin{equation}
g_{\theta}(z,z') = \left\{
    \begin{array}{ll}
        f_{\theta}(z) + h_{\theta}^2(z)z' & \mbox{if }  p_{\theta}(x|z) = \mathrm{N}(f_{\theta}(z),h_{\theta}^2(z)\Id_d)  \\
        f_{\theta}(z) + c^2z' & \mbox{if } p_{\theta}(x|z) = \mathrm{N}(f_{\theta}(z),c^2\Id_d) \eqsp.
    \end{array}
\right.
\end{equation}
Moreover,  if relaxing the conditional probability to $ p_{\theta}(x|z) = \mathrm{N}(f_{\theta}(z),h_{\theta}^2(z)\Id_d) $ or $ p_{\theta}(x|z) = \mathrm{N}(f_{\theta}(z),c^2\Id_d) $ instead of
$ \delta_{f_{\theta}(z)} $ is crucial for the training of VAEs, this is not particularly relevant during inference since the relaxed conditional distribution is simply a noisy version of the push-forward measure $ f_{\theta\#}\mu_p $. For this reason, the inference in VAEs is often done by
simply sampling from the single output $ f_{\theta}(z) $ of the decoder, with $ z \sim \mu_p $, and so the generated distribution is trivially of the form $ f_{\theta\#}\mu_p $ in that case.

\subsection{Score-based generative models}
In diffusion models, the generation process is an Euler-Maruyama discretization of the reverse-time denoising diffusion \cite[]{song2020score}

\begin{equation}
\rmd X_t = \left(f(X_t,t) - g^2(t)\nabla \log p_{t}(X_t)\right)\rmd t + g(t)\rmd B_t \eqsp,
\end{equation}

where $ (X_t)_{t \in [0,T]} $ is a random process on $ \rset^d $, 
$ (B_t)_{t \in [0,T]} $ is a Brownian motion, and where $ f(.,t) : \rset^d \rightarrow \rset^d $ is a vector-valued function called the drift operator, $ g : \rset \rightarrow \rset $ is a real-valued function called the diffusion coefficient and $ \nabla \log p_{t}(X_t) $  is the score of the marginal law of $ X_t $, which is approximated by a neural network $ s_{\theta}(X_t,t) $. This discretization yields for instance, for an appropriate choice of $ f $ and $ g $, to the annealed Langevin 
dynamic \cite[]{song2019generative}:

\begin{equation}
 \left\{
    \begin{array}{ll}
        x_0 = z_0 \quad (z_0 \sim  \mathrm{N}(0,\Id_d)) \\
        x_{k+1} = x_k + (\alpha_k/2)s_{\theta}(x_k,\sigma_k) + \sqrt{\alpha_k}z_{k+1} \quad (z_{k+1} \sim \mathrm{N}(0,\Id_d))  \eqsp,
    \end{array}
\right.
\end{equation}
where $ \alpha_k = \varepsilon\sigma^2_k/\sigma^2_0 $ with 
$ \varepsilon > 0 $ being an hyperparameter of the model and 
$ (\sigma_k)_{k \geq 0} $ is such that it exists $ K $ (another hyperparameter of the model) such that $ (\sigma_{Ki})_{i\geq 0} $
is a geometric progression, and for all $ k \geq 0 $, $ \sigma_k = \sigma_{K\floor{k/K}} $. Denoting for $ k \geq 1 $, $ h^k_{\theta} : \rset^{2d} \rightarrow \rset^d $ the function defined as
\begin{equation}
h^k_{\theta}(x,y) = x + (\alpha_k/2)s_{\theta}(x,\sigma_k) + \sqrt{\alpha_k}y  \eqsp,
\end{equation} 
it follows that when the data are generated with a Langevin dynamic 
of $ N $ iterations, the generated distribution is of 
the form $ g_{\theta\#}\mu_p $ with $ p = d(N+1) $ and where 
$ g_{\theta} : \rset^{d(N+1)} \rightarrow \rset^d $ is the function defined as
\begin{equation}
g_{\theta}(z_0,z_1,z_2,\dots,z_{N-1},z_N) = h^N_{\theta}\left(h^{N-1}_{\theta}\left(\dots(h_{\theta}^{2}\left(h_{\theta}^{1}(z_0,z_1),z_2\right),\dots),z_{N-1}\right),z_{N}\right) \eqsp.
\end{equation}

\section{Proofs of the theoretical results}
\label{sec:proof-}

\subsection{Proof of \Cref{thm:perfectapprox}}
\label{sec:proof-crefthm:perfectapprox}

We start by recalling the Gaussian isoperimetric inequality \cite{sudakov1978extremal}.

\begin{lemma}\label{lem:gaussiso}
  Let $\msa \in \mcb{\rset^p}$ and $\mu_p = \mathrm{N}(0,\Id_p) $. Then we have
\begin{equation}
\mu_p^+(\partial \msa) \geq \varphi(\Phi^{-1}(\mu_p(\msa)) ) \eqsp ,
\end{equation}
where $ \varphi(x) = (2 \uppi)^{-1/2}\exp[-x^2/2] $ and $ \Phi(x) = \int_{-\infty}^{x}\varphi(t) \rmd t $.
\noindent Equivalently, for all $ r \geq 0 $
\begin{equation}
\mu_p(\msa_r) \geq \Phi(r + \Phi^{-1}(\mu_p(\msa))) \eqsp . 
\end{equation}
\end{lemma}

In particular, using \Cref{lem:gaussiso}, one can show that among
all sets of given Gaussian measure $ \mu_p $, half-spaces have the minimal
$\mu_p$-surface area. We are now ready to turn to
the proof of \Cref{thm:perfectapprox}.

\begin{proof}[Proof of \Cref{thm:perfectapprox}]
  Let $\msa \in \mcb{\rset^d}$ such that $ g_\#\mu_p(\msa) > 0 $ (note that if
  $g_\#\mu_p(\msa) = 0$ then the result is trivial).  First, we show that for
  any $ \vareps > 0$,
  $g((g^{-1}(\msa))_{\varepsilon/ \mathrm{Lip}(g)}) \subset \msa_{\varepsilon}
  $.  Let $x$ be in $g((g^{-1}(\msa))_{\varepsilon/ \mathrm{Lip}(g)})$.
  There exists $ z_1 \in (g^{-1}(\msa))_{\varepsilon/ \mathrm{Lip}(g)} $ such that
  $ g(z_1) = x $.  There also
  exists $ z_2 \in g^{-1}(\msa) $ such that
\begin{equation}
\| z_1  - z_2 \| \leq \varepsilon / \mathrm{Lip}(g) \eqsp . 
\end{equation}
Hence, we have that 
\begin{equation}
\| x - a \| \leq \mathrm{Lip}(g)\| z_1 - z_2 \| \leq \varepsilon \eqsp ,
\end{equation}
where $ a = g(z_2) $. Since $ z_2 \in g^{-1}(\msa) $, $ a \in \msa$, and
therefore $ x \in \msa_\varepsilon $. Using this result, the fact that
$g_\#\mu_p(\msb) = \mu_p(g^{-1}(\msb))$ and $\msb \subset g^{-1}(g(\msb))$ for
any $\msb \in \mcb{\rset^d}$, we have
\begin{align}
\liminf_{\varepsilon \rightarrow 0^+} \{g_\#\mu_p(\msa_{\varepsilon}) - g_\#\mu_p(\msa)\}/ \varepsilon &\geq \liminf_{\varepsilon \rightarrow 0^+} \{g_\#\mu_p(g((g^{-1}(\msa))_{\varepsilon/\mathrm{Lip}(g)})) - g_\#\mu_p(\msa)\}/\varepsilon \\
             & \geq \liminf_{\varepsilon \rightarrow 0^+} \{\mu_p((g^{-1}(\msa))_{\varepsilon/\mathrm{Lip}(g)}) - \mu_p(g^{-1}(\msa))\}/\varepsilon \eqsp .  \label{eq:intermediate_pushforward_ineq}
\end{align}
Using \Cref{lem:gaussiso}, we have 
\begin{equation}\label{eq:isogmu}
\mathrm{Lip}(g) \liminf_{\varepsilon \rightarrow 0^+} \{(\mu_p((g^{-1}(\msa))_{\vareps / \mathrm{Lip}(g)}) - \mu_p(g^{-1}(\msa)))\}/\varepsilon \geq \varphi(\Phi^{-1}(\mu_p(g^{-1}(\msa)))) \eqsp , 
\end{equation}
Combining this result and \eqref{eq:intermediate_pushforward_ineq}, we get that
\begin{equation}
\mathrm{Lip}(g)(g_\#\mu_p)^+(\partial \msa) \geq \varphi(\Phi^{-1}(g_\#\mu_p(\msa))) \eqsp . 
\end{equation}
In addition, using \Cref{lem:gaussiso}, we have for all $ r \geq 0$
\begin{equation}
\mu_p((g^{-1}(\msa))_{r / \mathrm{Lip}(g)}) \geq \Phi(r / \mathrm{Lip}(g) + \Phi^{-1}(\mu_p(g^{-1}(\msa)))) \eqsp .
\end{equation}
Using this result and that
$ g((g^{-1}(\msa))_{r / \mathrm{Lip}(g)}) \subset \msa_{r} $, we have for any $r \geq 0$
\begin{equation}
  g_\#\mu_p(\msa_r) = \mu_p(g^{-1}(\msa_r)) \geq \mu_p((g^{-1}(\msa))_{r / \mathrm{Lip}(g)}) \geq \Phi(r / \mathrm{Lip}(g) + \Phi^{-1}(g_\#\mu_p(\msa))) \eqsp . 
\end{equation}
\end{proof}

\subsection{Proof of \Cref{coro:gm1}}
\label{sec:proof-crefcoro:gm1}
We prove the corollary when $ \nu = \lambda\mathrm{N}(-m,\sigma^2\Id_d) + (1 - \lambda)\mathrm{N}(m,\sigma^2\Id_d) $
since the problem can always be reduced to that case by translation and
setting $ m = (m_2 - m_1)/2 $.
Let $\msh$ be defined by $\msh = \{ x \in \rset^d | m^Tx \geq 0 \}$. Note that
for any $x \in \partial \msh$, $\normLigne{x - m} = \normLigne{x + m}$.  Since
the problem is invariant by rotation, we can consider without any loss of
generality that $ m = (\|m\|,0,\dots,0) $. In that case, we have
$\nu = \nu_1 \otimes \mathrm{N}(0,\sigma^2\Id_{d-1})$, where
$ \nu_1 = \lambda\mathrm{N}(-\|m\|,\sigma^2) + (1 - \lambda)\mathrm{N}(\|m\|,\sigma^2) $,
and $ \otimes $ is the tensor product between measures.  In this case, we have
that $\msh = \{x_1 \geq 0\} \times \rset^{d-1}$. Therefore, we have 
\begin{align}
  \nu^+(\partial \msh) &= \textstyle{\liminf_{\varepsilon \rightarrow 0^+}\{(\int_{\msh_\varepsilon}p_{\nu}(x)\rmd x - \int_{\msh}p_{\nu}(x)\rmd x )\}/\vareps} \eqsp, \\
                       &= \textstyle{\liminf_{\varepsilon \rightarrow 0^+}\{(\int_{-\varepsilon}^{+\infty}\int_{\rset^{d-1}}p_{\nu_1}(x_1)h(y)\rmd x_1\rmd y - \int_{0}^{+\infty}\int_{\rset^{d-1}}p_{\nu_1}(x_1)h(y)\rmd x_1\rmd y )\}/\vareps  \eqsp ,}
\end{align}
where $ p_{\nu} $ and $  p_{\nu_1} $ are the respective densities  of $ \nu $ and $ \nu_1 $, and $ h $ is the density of $ \mathrm{N}(0,\sigma^2I_{d-1}) $. It follows
that 
\begin{align}
\nu^+(\partial \msh) &= \textstyle{\liminf_{\varepsilon \rightarrow 0^+}(1/\varepsilon) \int_{-\varepsilon}^0 p_{\nu_1}(x_1) (\int_{\rset^{d-1}} h(y)\rmd y) \rmd x_1} \\
&= \textstyle{\liminf_{\varepsilon \rightarrow 0^+}(1/\vareps) \int_{-\varepsilon}^0 p_{\nu_1}(x_1)\rmd x_1= p_{\nu_1}(0) = (2\uppi\sigma^2)^{-1/2}\exp[-\|m\|^2 / (2\sigma^2)] \eqsp . }
\end{align}

Applying \Cref{thm:perfectapprox}, we get that
\begin{equation}\label{eq:gmlipcond}
\mathrm{Lip}(g) \geq \varphi(\Phi^{-1}(\nu(\msh))) / \nu^+(\partial \msh) \eqsp . 
\end{equation}
Furthermore, one can derive that
\begin{align}
\nu(\msh) &= \lambda(1 - \Phi(m/\sigma)) + \Phi(m/\sigma)(1 - \lambda) \\
&= \lambda(1 - 2\Phi(m/\sigma)) + \Phi(m/\sigma) \eqsp.
\end{align}
Observing that $ \lambda - \nu(\msh) $ is an increasing function of $ \lambda $ and
$ \lambda - \nu(\msh) = 0 $ if $ \lambda = 1/2 $, we get that
$ \lambda \leq \nu(\msh) $ if $ \lambda \leq 1/2 $ and $ \lambda \geq \nu(\msh) $ if $ \lambda \geq 1/2 $.
Since $ \varphi \circ \Phi^{-1} $ reaches its maximum in $1/2$, it follows that for any $ \lambda \in (0,1) $
we have 
\begin{equation}
\varphi(\Phi^{-1}(\nu(\msh))) \geq \varphi(\Phi^{-1}(\lambda)) \eqsp,
\end{equation}
and thus 
\begin{align}
\mathrm{Lip}(g) &\geq (2 \uppi)^{1/2}\sigma\varphi(\Phi^{-1}(\lambda))\exp[\|m\|^2/(2\sigma^2)] \\
&\geq \sigma \exp[\|m\|^2/(2\sigma^2)-(\Phi^{-1}(\lambda))^2/2] \eqsp,
\end{align}
which concludes the proof.

\begin{figure}[!h]
\centering
\includegraphics[width=0.5\textwidth]{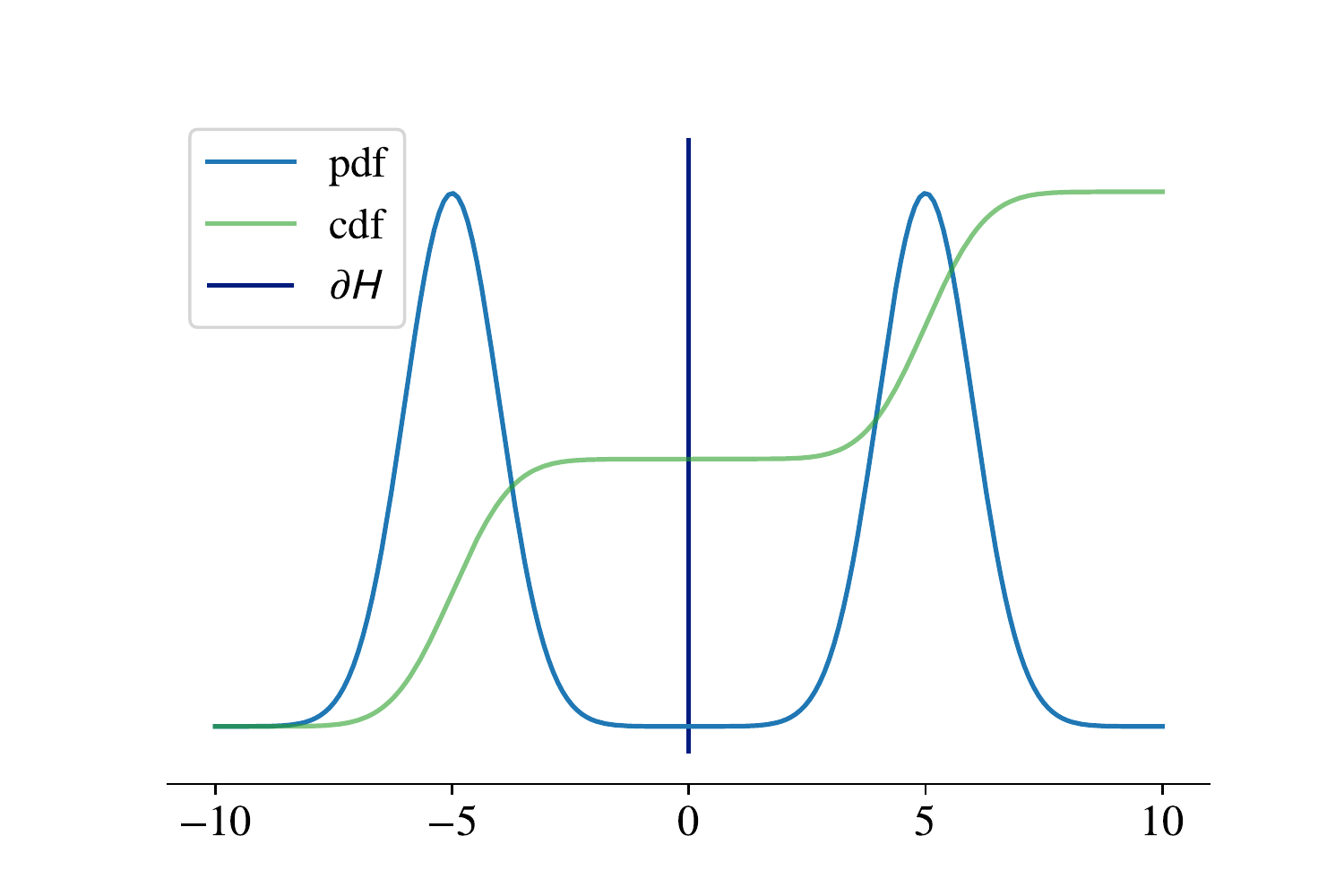} 
\caption{The hypersurface $ \partial \msh $ in the univariate case.}\label{fig:hyperplan}
\end{figure}

\subsection{Proof of \Cref{coro:mongemap}}
\label{sec:proof-crefc}

Since $ \nu $ admits a density $ p_{\nu} $ with respect to the Lesbegue measure,
it follows that $ \Phi_{\nu} $ is absolutely continuous and therefore
differentiable almost everywhere w.r.t. the Lebesgue measure using the Lebesgue
differentiation theorem.
%(alternatively, since $d=1$, we can also remark that
%$\Phi_\nu$ is non-decreasing and therefore differentiable almost everywhere
%w.r.t. the Lebesgue measure using the Lebesgue differentiation theorem for
%monotone mappings).  
Moreover, since $ \mathrm{supp}(\nu) = \rset $, it follows
that $ \Phi_{\nu}: \rset \to \ooint{0,1}$ is increasing and therefore is
bijective, and so $ T_{\mathrm{OT}} = \Phi^{-1}_\nu \circ \Phi $ is also
differentiable almost everywhere w.r.t. the Lebesgue measure and bijective, with inverse
$ T_{\mathrm{OT}}^{-1} = \Phi^{-1} \circ \Phi_\nu $,
using \cite[Remark 2.29]{OT}. Therefore, for any $x \in \rset$ we have 
\begin{align}
T'_{\mathrm{OT}}(x) &= \varphi(x)/p_{\nu}(T_{\mathrm{OT}}(x)) \\
&= \varphi(\Phi^{-1}(\Phi_\nu(T_{\mathrm{OT}}(x))) / p_{\nu}(T_{\mathrm{OT}}(x)) \eqsp .
\end{align}

Let $y \in \rset$. Using \Cref{thm:perfectapprox} with
$\msa = \ocint{-\infty, y}$ we get that for any $g: \ \rset^p \to \rset$, Lipschitz
such that $g_\# \mu_p = \nu$, 
\begin{equation}
\textstyle{
  \mathrm{Lip}(g) \geq \sup_{y \in \rset}  \varphi(\Phi^{-1}(\Phi_\nu(y))/p_{\nu}(y) \eqsp  ,
  }
\end{equation}
and so, since $ T_{\mathrm{OT}} $ is bijective
\begin{equation}
  \textstyle{
    \mathrm{Lip}(g) \geq \sup_{x \in \rset}   |T'_{\textup{OT}}(x)| \eqsp  ,
    }
  \end{equation}
which concludes the proof.

\subsection{Proof of \Cref{thm:lowerboundtv}}
\label{sec:proof-crefthm:l}

Let $ \msa \in \mcb{\rset^d} $ and let $ r > 0 $. We have on one hand 
\begin{align}
|g_{\#}\mu_p(\msa_r \setminus \msa)| &\leq |g_{\#}\mu_p(\msa_r \setminus \msa) - \nu(\msa_r \setminus \msa)| + |\nu(\msa_r \setminus \msa)| \\
  &\leq d_{\mathrm{TV}}(g_{\#}\mu_p,\nu) + \nu(\msa_r \setminus \msa) \eqsp .
\end{align}
Using \Cref{thm:perfectapprox}, we get
\begin{equation}
\textstyle{
|g_{\#}\mu_p(\msa_r \setminus \msa)| = g_{\#}\mu_p(\msa_r) - g_{\#}\mu_p(\msa) \geq \Phi\left(r/\mathrm{Lip}(g) + \Phi^{-1}(g_{\#}\mu_p(\msa))\right) - g_{\#}\mu_p(\msa) \eqsp ,
}
\end{equation}
and so 
\begin{equation}
d_{\mathrm{TV}}(g_{\#}\mu_p,\nu) \geq \alpha_g(\msa,r) - g_{\#}\mu_p(\msa) - \nu(\msa_r \setminus \msa) \eqsp ,
\end{equation}
where $ \alpha_g(\msa,r) = \Phi\left(r/\mathrm{Lip}(g) + \Phi^{-1}(g_{\#}\mu_p(\msa))\right) $. On the other hand, we have 
\begin{align}
|g_{\#}\mu_p(\msa_r)| &\leq |g_{\#}\mu_p(\msa_r) - \nu(\msa_r)| + |\nu(\msa_r)| \\
    &\leq d_{\mathrm{TV}}(g_{\#}\mu_p,\nu) + \nu(\msa_r \setminus \msa) + \nu(\msa)\eqsp .
\end{align}
Using \Cref{thm:perfectapprox}, we get
\begin{equation}
\textstyle{
|g_{\#}\mu_p(\msa_r)|  \geq \Phi\left(r/\mathrm{Lip}(g) + \Phi^{-1}(g_{\#}\mu_p(\msa))\right)  \eqsp ,}
\end{equation}
and so 
\begin{equation}
d_{\mathrm{TV}}(g_{\#}\mu_p,\nu) \geq \alpha_g(\msa,r) - \nu(\msa) - \nu(\msa_r \setminus \msa) \eqsp ,
\end{equation}
which concludes the proof.

\subsection{Proof of \Cref{coro:discomanifold}}
\label{sec:proof-crefcoro:discomanifold}
To prove \Cref{coro:discomanifold}, we will need the following lemma:

\begin{lemma}\label{lem:inclusion}
Let $ \msa \in \mcb{\rset^d} $ and $ r > 0 $. We denote $ \mathsf{B} = (\msa_r)^c $. Then
\begin{equation}
\mathsf{B}_r \subset \bar{\msa^c} \eqsp ,
\end{equation}
where $ \bar{\msa^c} $ denotes the closure of the complementary of $ \msa $. 
\end{lemma}

\begin{proof}
Let $ x \in \mathsf{B}_r $.
There exists $ b \in \mathsf{B}$ such that $ \|x - b\| \leq r $. 
Moreover, since $ \mathsf{B} = (\msa_r)^c $, it follows 
that for all $ a \in \msa $,
\begin{equation}
\|b - a \| > r \eqsp.
\end{equation}
Then
\begin{equation}
r < \| b - x \| + \| x - a \| \eqsp,
\end{equation}
and so, it follows that for all $ a \in \msa $,
\begin{equation}
\| x - a \| > 0 \eqsp.
\end{equation}
Thus  $ x \in \bar{\msa^c} $.
\end{proof}

Now we are ready to turn to the proof of \Cref{coro:discomanifold}.

\begin{proof}[Proof of \Cref{coro:discomanifold}]
We set $ r = d(\msm_1,\msm_2)/2 $ and $ \msa = (\msm_1)_r $. 
Using \Cref{thm:lowerboundtv}, we have
\begin{equation}
\textstyle{d_{\mathrm{TV}}(g_\#\mu_p,\nu) \geq \alpha_g(\msa,r) - \min\{g_{\#}\mu_p(\msa),\nu(\msa)\} - \nu(\msa_r \setminus \msa) \eqsp .}
\end{equation}
First we suppose that $ g_{\#}\mu_p(\msa) \geq \nu(\msa) $:
since $ \Phi $ is a non-decreasing function, it follows that
\begin{equation}
\textstyle{\alpha_g(\msa,r) = \Phi\left(r/\mathrm{Lip}(g) + \Phi^{-1}(g_{\#}\mu_p(\msa))\right) \geq \Phi\left(r/\mathrm{Lip}(g) + \Phi^{-1}(\nu(\msa))\right) \eqsp .} 
\end{equation}
Moreover $ \min\{g_{\#}\mu_p(\msa),\nu(\msa)\} = \nu(\msa) = \lambda = \Phi(\Phi^{-1}(\lambda))$ and so it follows
\begin{equation}
\textstyle{d_{\mathrm{TV}}(g_\#\mu_p,\nu)} \geq \textstyle{\Phi\left(d(\msm_1,\msm_2)/(2\mathrm{Lip}(g)) + \Phi^{-1}(\lambda)\right) - \Phi(\Phi^{-1}(\lambda))} 
\geq\textstyle{\int_{\Phi^{-1}(\lambda)}^{r/\mathrm{Lip}(g)+\Phi^{-1}(\lambda)}\varphi(t)dt}\eqsp,
\end{equation}
since $ \nu $ has no mass on $ \msa_r \setminus \msa $. 
Now we suppose that $ g_{\#}\mu_p(\msa) \leq \nu(\msa) $: we then set $ \mathsf{B} = \msa^c $. Since
 $ g_{\#}\mu_p(\msa) \leq \nu(\msa)$, 
we have $ g_{\#}\mu_p(\mathsf{B}) \geq \nu(\mathsf{B}) $. Applying \Cref{thm:lowerboundtv}, and the 
same reasoning as before we get 
\begin{align}
d_{\mathrm{TV}}(g_\#\mu_p,\nu) &\geq \alpha_g(\mathsf{B},r) - \min\{g_{\#}\mu_p(\mathsf{B}),\nu(\mathsf{B})\} - \nu(\mathsf{B}_r \setminus \mathsf{B}) \\
&\geq \Phi\left(d(\msm_1,\msm_2)/(2\mathrm{Lip}(g)) + \Phi^{-1}(1 -\lambda)\right) - \Phi(\Phi^{-1}((1 - \lambda)) - \nu(\mathsf{B}_r \setminus \mathsf{B}) \eqsp.
\end{align}
Using \Cref{lem:inclusion}, we get that $ \nu(\mathsf{B}_r \setminus \mathsf{B}) \leq \nu(\bar{\msa^c} \setminus (\msa_r)^c)$
but $ \nu(\bar{\msa^c} \setminus (\msa_r)^c) = 0 $ since $ \nu $ has no mass on $ \bar{\msa^c} \setminus (\msa_r)^c $ 
except on its boundary and so its follows that
\begin{align}
\textstyle{d_{\mathrm{TV}}(g_\#\mu_p,\nu)} &\geq \textstyle{\Phi\left(d(\msm_1,\msm_2)/(2\mathrm{Lip}(g)) + \Phi^{-1}(1 -\lambda)\right) - \Phi(\Phi^{-1}((1 - \lambda)) } \\
&\geq \textstyle{\Phi\left(d(\msm_1,\msm_2)/(2\mathrm{Lip}(g)) - \Phi^{-1}(\lambda)\right) - \Phi(-\Phi^{-1}(\lambda))} \\
&\geq \textstyle{\int_{-\Phi^{-1}(\lambda)}^{r/\mathrm{Lip}(g)-\Phi^{-1}(\lambda)}\varphi(t)dt } \eqsp,
\end{align}
since $ \Phi^{-1}(1-\lambda) = -\Phi^{-1}(\lambda) $. Since $ \lambda \geq 1/2 $, it follows that $ \Phi^{-1}(\lambda) \geq 0 $
and so 
\begin{equation}
\textstyle{\int_{-\Phi^{-1}(\lambda)}^{r/\mathrm{Lip}(g)-\Phi^{-1}(\lambda)}\varphi(t)dt} \geq  \textstyle{\int_{\Phi^{-1}(\lambda)}^{r/\mathrm{Lip}(g)+\Phi^{-1}(\lambda)}\varphi(t)dt} \eqsp ,
\end{equation}
which concludes the proof. 
\end{proof}

\subsection{Proof of \Cref{coro:gm2}}
\label{sec:proof-crefcoro:gm2}
As previously,
we prove the corollary when $ \nu = (1/2)[\mathrm{N}(-m,\sigma^2\Id_d) + \mathrm{N}(m,\sigma^2\Id_d)] $
since the problem can always be reduced to that case by translation and
setting $ m = (m_2 - m_1)/2 $.
Since the problem is invariant by rotation, we can assume without any loss of
generality that $ m = (\normLigne{m},0,\dots,0) $.  Let $ \msh $ be the half-space of
$ \rset^d$ defined by $\msh = \ocint{-\infty,0} \times \rset^{d-1}$ and we set $ r = \normLigne{m}/2\sigma $. 
First we suppose that $ g_{\#}\mu_p(\msh) \geq \nu(\msh) $: using Theorem \ref{thm:lowerboundtv}, we get that
\begin{equation}
\textstyle{d_{\mathrm{TV}}(g_\#\mu_p,\nu) \geq \alpha_g(\msh,r) - \min\{g_{\#}\mu_p(\msh),\nu(\msh)\} - \nu(\msh_r \setminus \msh) \eqsp , }
\end{equation}
with $\msh_r = \ocint{-\infty,\|m\|/2\sigma} \times \rset^{d-1}$. 
On one hand we have
that $\nu = \nu_1 \otimes \mathrm{N}(0,\sigma^2 \Id_{d-1})$, where
$ \nu_1 = (1/2)[\mathrm{N}(-\|m\|,\sigma^2)+\mathrm{N}(\|m\|,\sigma^2)]$ and
so $ \nu(\msh_r \setminus \msh) = \nu_1(\ccint{0,\|m\|/2\sigma}) $. On 
the other hand we have that $ \min\{g_{\#}\mu_p(\msh),\nu(\msh)\} = \nu(\msh) $ 
and $ g_{\#}\mu_p(\msh) \geq 1/2 $ since $ g_{\#}\mu_p(\msh) \geq \nu(\msh) $. Hence it follows that 
\begin{equation}
\textstyle{d_{\mathrm{TV}}(g_\#\mu_p,\nu) \geq \Phi(r/\mathrm{Lip}(g)) - 1/2 - \nu_1(\ccint{0,\|m\|/2\sigma}) \eqsp . }
\end{equation}
Now we suppose that $ g_{\#}\mu_p(\msh) \leq \nu(\msh) $: we then
set $ \msh_2 = \ocint{0,+\infty} \times \rset^{d-1}$. Since  
$ g_{\#}\mu_p(\msh) \leq 1/2 $, we get that $ g_{\#}\mu_p(\msh_2) \geq 1/2 $ 
and so $ g_{\#}\mu_p(\msh_2) \geq \nu(\msh_2) $. Hence we retrieve the previous case and so it follows that 
\begin{equation}
\textstyle{d_{\mathrm{TV}}(g_\#\mu_p,\nu) \geq \Phi(r/\mathrm{Lip}(g)) - 1/2 - \nu_1(\ccint{-\|m\|/2\sigma,0}) \eqsp . }
\end{equation}
Since $ \nu_1(\ccint{-\|m\|/2\sigma,0}) = \nu_1(\ccint{0,\|m\|/2\sigma}) $, we get in both cases
\begin{equation}
\textstyle{d_{\mathrm{TV}}(g_\#\mu_p,\nu) \geq \Phi(r/\mathrm{Lip}(g)) - 1/2 - \nu_1(\ccint{0,\|m\|/2\sigma}) \eqsp . }
\end{equation}
Now we derive the value of $ \nu_1(\ccint{0,\|m\|/2\sigma}) $:
\begin{align}
\nu_1(\ccint{0,\|m\|/2\sigma})  &= \textstyle{(1/2)\int_0^{m/2\sigma}(2\uppi\sigma^2)^{-1/2}\exp[-(x+m)^2/2\sigma^2]dx} \\
& \textstyle{\qquad + (1/2)\int_0^{m/2\sigma}(2\uppi\sigma^2)^{-1/2}\exp[-(x-m)^2/2\sigma^2]dx } \\
&= \textstyle{(1/2)\int_{-m/2\sigma}^{m/2\sigma}(2\uppi\sigma^2)^{-1/2}\exp[-(x+m)^2/2\sigma^2]dx } \\
&= \textstyle{(1/2)\int_{\|m\|(2\sigma -1)/2\sigma^2}^{\|m\|(2\sigma+1)/2\sigma^2}\varphi(x)dx } \eqsp,
\end{align}
which concludes the proof.

\subsection{Proof of \Cref{thm:lowerboundkl}}
\label{sec:proof-crefthm:l-1}

Let $ \msa \in \mcb{\rset^d} $, $ r > 0 $ and $ \zeta > 0 $. 
We set for any $ x \in \rset^d $
$ f(x) = \zeta\chi_{\msa_r \setminus \msa}(x) $, 
where $ \chi_\msa $ denotes the characteristic function of the set $ \msa $.
Since $ f $ is bounded, it follows that
\begin{align}
d_{KL}(g_{\#}\mu_p||\nu) &\geq \textstyle{\int_{\rset^d}f(x)dg_{\#}\mu_p(x) - \mathrm{log}\left(\int_{\rset^d} e^{f(x)} d\nu(x)\right)} \\
  &\geq \textstyle{\zeta g_{\#}\mu_p(\msa_r \setminus \msa) -  \mathrm{log}\left(1 + (e^{\zeta} - 1)\nu(\msa_r \setminus \msa)\right)} \eqsp .
\end{align}
Using \Cref{thm:perfectapprox}, we get
\begin{equation}
\textstyle{
g_{\#}\mu_p(\msa_r \setminus \msa) = g_{\#}\mu_p(\msa_r) - g_{\#}\mu_p(\msa) \geq \Phi\left(r/\mathrm{Lip}(g) + \Phi^{-1}(g_{\#}\mu_p(\msa))\right) - g_{\#}\mu_p(\msa) \eqsp .
}
\end{equation}
Thus we get
\begin{equation}
d_{\mathrm{KL}}(g_{\#}\mu_p||\nu) \geq \sup\ensembleLigne{J(\zeta,\msa,r)}{\zeta \in \reel,\msa \in \mcb{\reel^d}, r > 0}  \eqsp,
\end{equation}
where the functional $ J $ is defined by
\begin{align}
J(\zeta,\msa,r) &= \textstyle{ \zeta\left(\Phi\left(r/\mathrm{Lip}(g) + \Phi^{-1}(g_{\#}\mu_p(\msa))\right) - g_{\#}\mu_p(\msa)\right)} \\
& \quad \textstyle{- \mathrm{log}\left(1 + (e^{\zeta} - 1)\nu(\msa_r \setminus \msa)\right) \eqsp .}
\end{align}
Differentiating $ J $ with respect to $ \zeta $, we get that
  \begin{equation}
\nabla_\zeta J(\zeta,\msa,r) =  \beta_g(\msa,r) - (e^\zeta \nu(\msa_r \setminus \msa))/(1 + (e^\zeta - 1)\nu(\msa_r \setminus \msa)) \eqsp ,
\end{equation}
where $ \textstyle{\beta_g(\msa,r) = \Phi\left(r/\mathrm{Lip}(g) + \Phi^{-1}(g_{\#}\mu_p(\msa))\right) - g_{\#}\mu_p(\msa)} $ .
Applying the first order condition, we get that:
\begin{equation}
\zeta^* = \mathrm{log}[\beta_g(\msa,r)(1-\nu(\msa_r \setminus \msa))] - \mathrm{log}[\nu(\msa_r \setminus \msa)(1-\beta_g(\msa,r))] \eqsp .
\end{equation}
By re-injecting the value of $ \zeta^* $, we get 
\begin{align}
\textstyle{\zeta^*\beta_g(\msa,r) - \log\left(1 + (e^{\zeta^*} -1)\nu(\msa_r \setminus \msa)\right)} &= \textstyle{\beta_g(\msa,r)\log\left(\frac{\beta_g(\msa,r)\left(1-\nu(\msa_r \setminus \msa)\right)}{\nu(\msa_r \setminus \msa)\left(1-\beta_g(\msa,r)\right)}\right)} \\ 
 & \qquad \textstyle{- \log\left(\frac{1 - \nu(\msa_r \setminus \msa)}{1 - \beta_g(\msa,r)}\right)} \\
 &= \textstyle{\beta_g(\msa,r)\log\left(\frac{\beta_g(\msa,r)}{\nu(\msa_r\setminus \msa)}\right)} \\
 & \qquad \textstyle{
 + \left(1 - \beta_g(\msa,r)\right)\log\left(\frac{1 - \beta_g(\msa,r)}{1 - \nu(\msa_r \setminus \msa)}\right)} \eqsp ,
\end{align}
which concludes the proof.

\subsection{Proof of \Cref{coro:kl_gm}}
As previously, we prove the corollary when $ \nu = (1/2)[\mathrm{N}(-m,\sigma^2\Id_d) + \mathrm{N}(m,\sigma^2\Id_d)] $
since the problem can always be reduced to that case by translation and
setting $ m = (m_2 - m_1)/2 $.
Since the problem is invariant by rotation, we can assume without any loss of generality that $ m = (\normLigne{m},0,\dots,0) $. Furthermore, 
observe that the half-space $ \ensembleLigne{(m_2 - m_1)^T\left(x - (m_1+m_2)/2\right) \leq 0}{x \in \rset^d} $ becomes $ \ocint{-\infty,0} \times \rset^{d-1} $ in that case, and that the condition 
$ \lambda \in \ocint{0,1/2} $ is indeed non-restrictive since the problem is invariant by rotation. 
We set as before $ \msh = \ocint{-\infty,0} \times \rset^{d-1} $ and $ r = \normLigne{m}/2\sigma $.

Applying \Cref{thm:lowerboundkl}, we get
\begin{equation}
\textstyle{
d_{\mathrm{KL}}(g_{\#}\mu_p||\nu) \geq  \beta_g(\msh,r)\log\left(\frac{\beta_g(\msh,r)}{\nu(\msh_r\setminus \msh)}\right) + \left(1 - \beta_g(\msh,r)\right)\log\left(\frac{1 - \beta_g(\msh,r)}{1 - \nu(\msh_r \setminus \msh)}\right) \eqsp .
}
\end{equation}
On one hand, we get
\begin{align}
\beta_g(\msh,r) &= \Phi\left(r/\mathrm{Lip}(g) + \Phi^{-1}(g_{\#}\mu_p(\msh) \right) - g_{\#}\mu_p(\msh) \\
&= \Phi\left(r/\mathrm{Lip}(g) + \Phi^{-1}(g_{\#}\mu_p(\msh) \right) - \Phi\left(\Phi^{-1}(g_{\#}\mu_p(\msh)\right) \\
&= \textstyle{\int_{\Phi^{-1}(\lambda)}^{\normLigne{m}/ 2\sigma\mathrm{Lip}(g) + \Phi^{-1}(\lambda)} \varphi(t) \rmd t} \\ 
&= \textstyle{\int_{-\Phi^{-1}(1-\lambda)}^{\normLigne{m}/ 2\sigma\mathrm{Lip}(g) - \Phi^{-1}(1-\lambda)} \varphi(t) \rmd t} \eqsp,
\end{align}
noting $ \lambda = g_{\#}\mu_p(\msh) $. We replaced $ \Phi^{-1}(\lambda)  $ by $ - \Phi^{-1}(1-\lambda) $ in order to emphasize that $ \Phi^{-1}(\lambda) \leq 0 $ since $ \lambda \leq 1/2 $. Observe that if we supposed $ \lambda \geq 1/2 $, we would have $ \beta_g(\msh^c,r) \geq \beta_g(\msh,r) $ and so the bound that we would have found by reasoning on $ \msh $ would have been
sub-optimal.
On the other hand, observing as before
that $\nu = \nu_1 \otimes \mathrm{N}(0,\sigma^2 \Id_{d-1})$, where
$ \nu_1 = (1/2)[\mathrm{N}(-\|m\|,\sigma^2)+\mathrm{N}(\|m\|,\sigma^2)]$, we get that  
\begin{align}
\nu(\msh_r \setminus \msh) &= \nu_1(\ccint{0,\|m\|/2\sigma})\\
&= \textstyle{(1/2)\int_{\normLigne{m}(2\sigma - 1)/2\sigma^2}^{\normLigne{m}(2\sigma  + 1)/2\sigma^2}\varphi(t) \rmd t} \eqsp, 
\end{align} 
which concludes the proof.

\section{Additional theoretical result}
\label{sec:theoplus}
In this section we derive a generalization of \Cref{coro:discomanifold} when 
$ \nu $ is a distribution whose support is composed of more than two disconnected  manifolds.

\begin{corollary}
Let $ \nu $ be a measure on $ \rset^d $ on $ N $ disconnected manifolds $ (\msm_1,\dots,\msm_N) $, and let
$ g : \rset^p \rightarrow \reel^d $ be a Lipschitz function. Then,
\begin{equation}
\textstyle{d_{\mathrm{TV}}(g_\#\mu_p,\nu) \geq \max\limits_{I\subset \llbracket 1,N \rrbracket} \int_{\Phi^{-1}(\lambda)}^{d(\bigsqcup\limits_{i \in I}\msm_i,\bigsqcup\limits_{j \in \llbracket 1, N \rrbracket \setminus I}\msm_j)/ 2 \mathrm{Lip}(g) + \Phi^{-1}(\lambda)} \varphi(t)dt\eqsp , } 
\end{equation}
where for $ \msa, \msb \in \mcb{\rset^d} $,
$ d(\msa,\msb)  = \inf \ensembleLigne{\normLigne{a - b}}{a \in \msa, b \in \msb} $,
and $ \lambda = \textstyle{\nu\left(\bigsqcup\limits_{i \in I}\msm_i\right)} $  if $ \textstyle{\nu\left(\bigsqcup\limits_{i \in I}\msm_i\right)} \geq 1/2 $ and $ \lambda = 1 - \textstyle{\nu\left(\bigsqcup\limits_{i \in I}\msm_i\right)} $ otherwise.
\end{corollary} 

\begin{proof}
Let $ I \subset \llbracket 1, N \rrbracket$.  First, we suppose that $ \textstyle{\nu\left(\bigsqcup\limits_{i \in I}\msm_i\right)} \geq 1/2 $. Since $ \nu $ can be seen as a bi-modal distribution on the 
two disconnected sets $ \bigsqcup\limits_{i \in I}\msm_i $ and $ \bigsqcup\limits_{j \in \llbracket 1, N \rrbracket \setminus I}\msm_j $, we can apply \Cref{coro:discomanifold}. Thus we get 
\begin{equation}
\textstyle{d_{\mathrm{TV}}(g_\#\mu_p,\nu) \geq \int_{\Phi^{-1}(\lambda)}^{d(\bigsqcup\limits_{i \in I}\msm_i,\bigsqcup\limits_{j \in \llbracket 1, N \rrbracket \setminus I}\msm_j)/ 2 \mathrm{Lip}(g) + \Phi^{-1}(\lambda)} \varphi(t)dt\eqsp . } 
\end{equation}\label{eq:discomanifoldn}
If $ \textstyle{\nu\left(\bigsqcup\limits_{i \in I}\msm_i\right)} \leq 1/2 $, we
can still apply \Cref{coro:discomanifold} by interchanging the roles of $ \bigsqcup\limits_{i \in I}\msm_i $ and $ \bigsqcup\limits_{j \in \llbracket 1, N \rrbracket \setminus I}\msm_j $, thus we get also Inequality \eqref{eq:discomanifoldn} in that case, which concludes the proof.
\end{proof}

\section{Experimental details}
\label{sec:expedetails}
We detail our experiments in dimension $ 1 $ in Appendix \ref{sec:expedetuni}. In Appendix
\ref{sec:expedetgm}, we give details on our experiment on the synthetic mixture of two Gaussians 
derived from MNIST. Finally, we detail the experiment on the subset of all $ 3 $ and $ 7 $
of MNIST in Appendix \ref{sec:expedetmnist}. We trained our models using 
$ 2 $ NVIDIA Titan Xp from the proprietary server of our institution with an estimated 
total training time of approximately $ 175 $ GPU hours. Code is available  \href{https://github.com/AntoineSalmona/Push-forward-Generative-Models}{here}
\footnote{\href{https://github.com/AntoineSalmona/Push-forward-Generative-Models}{https://github.com/AntoineSalmona/Push-forward-Generative-Models}}.

\subsection{Univariate case}
\label{sec:expedetuni}
In the univariate case we use a simple $ 3$-layer Multi Layer Perceptron (MLP) 
of shape $ (1,128,256,1) $ as decoder for the VAE and as generator for the GAN. The network 
has a total of $ 33537 $ learnable parameters. The score network uses also a
a $ 3$-layer MLP block, this time of shape $ (1,96,196,1) $, in which at
each layer is injected the noise information transformed by a  
positional encoding \cite[]{vaswani2017attention} and then by another MLP block
size $ (16,32,64) $, see \Cref{fig:archiscore}. The score network 
has a total of $ 34665 $ learnable parameters.
In all three models, we use LeakyReLU \cite[]{maas2013rectifier} as non-linearity
with a negative slope of $ 0.2 $. The three models 
are trained during $ 400 $ epochs with a batch size of $ 1000 $ using ADAM \cite[]{kingma2014adam} with 
a momentum of $ 0.9 $ and a learning rate of $ 10^{-4} $. In the following,
we give more specific details for each model.

\begin{figure}[h!]
\centering
\includegraphics[width=0.9\textwidth]{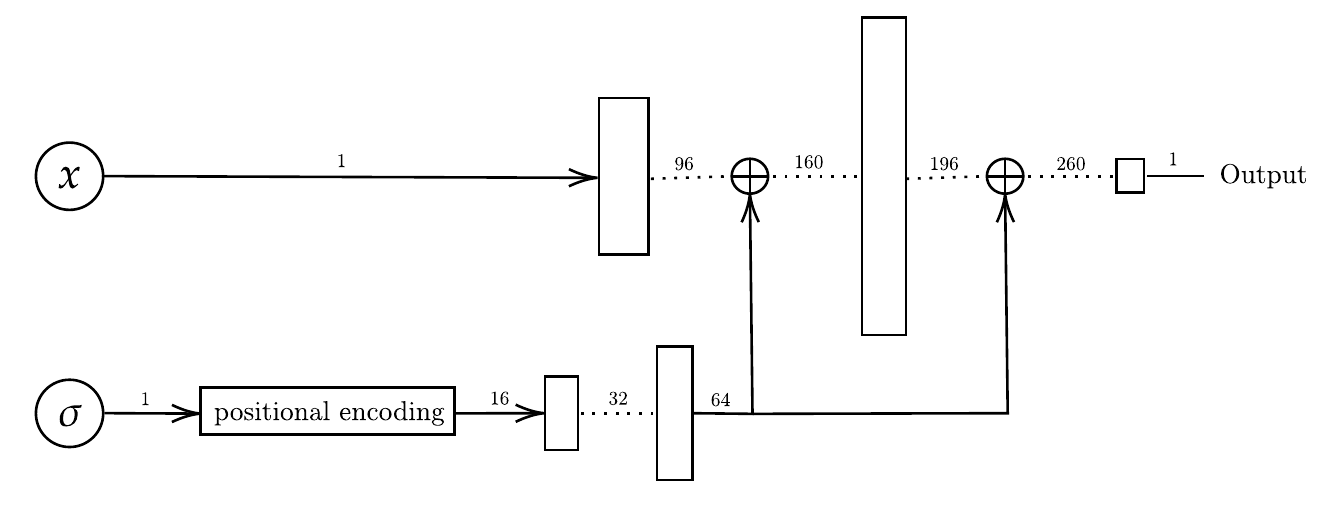}
    \caption{Architecture of the score network used for the univariate
    experiments. The "positional encoding" block applies the sine transform described
in \cite{vaswani2017attention}. $ \bigoplus $ corresponds to concatenation, the vertical blocks 
correspond to the fully connected layers and the numbers over the arrows correspond to the size of the vectors.} 
    \label{fig:archiscore}
\end{figure}

\paragraph{Variational autoencoder.} We use the vanilla VAE model 
as described in \cite{kingma2013auto}. In
the following, we denote $ \theta $ and $ \phi $ the respective parameters of the decoder and the encoder. 
The decoder $ f_{\phi} $ is composed of an MLP block of size $ (1,256,128) $ followed by 
two parallel fully connected layers of shape $ (128,1) $ which 
gives two outputs $  f_{1\phi}(x) $ and $ f_{2\phi}(x) $. Then 
the input $ z $ of the decoder $ g_{\theta} $ is obtained by the so-called
reparametrization trick, which consists in sampling
$ z \sim q_{\phi}^{z|x} $, where $  q_{\phi}^{z|x}  = \mathrm{N}(f_{1\phi}(x),\exp[f_{2\phi}(x)]) $.
During training, the model minimizes the following loss function:
\begin{equation}
\textstyle{
\mathcal{L}_{\mathrm{VAE}}(\theta,\phi) = \mathbb{E}_{x \sim \nu}[\mathrm{ELBO}_{\theta,\phi}(x,q_{\phi}^{z|x},p_{\theta}^{x|z})] \eqsp ,}
\end{equation}
 where $ p_{\theta}^{x|z} = \mathrm{N}(g_{\theta}(z),c^2\Id_d) $ and $ \mathrm{ELBO} $ is the Evidence Lower Bound \cite[]{blei2017variational}, defined as follows:
\begin{equation}
\textstyle{
\mathrm{ELBO}_{\theta,\phi}(x,q_{\phi}^{z|x},p_{\theta}^{x|z}) = \mathbb{E}_{z \sim q_{\phi}^{z|x}}[\log(p_{\theta}(x|z))] - d_{\mathrm{KL}}(q_{\phi}^{z|x}||\mathrm{N}(0,\Id_p)) \eqsp.}
\end{equation}
The standard deviation $ c $ in $ p_{\theta}^{x|z} $ is an hyperparameter of the model. For our experiments, we observed
that $ c = 0.1 $ gave good results.

\paragraph{Generative adversarial network.} As for the VAE, we
use the vanilla GAN model as described in \cite{goodfellow2014generative}. 
The discriminator is $ 4 $-layer MLP of shape $ (1,512,256,128,1) $ with spectral
normalization \cite[]{miyato2018spectral} in order 
to reduce as much as possible mode collapse.
We train the model using the vanilla adversarial loss, that the 
discriminator $ d_\phi $ tries to maximize and that the generator $ g_{\theta} $ tries to minimize
\begin{equation}
\mathcal{L}_{\mathrm{GAN}}(\theta,\phi) = \mathbb{E}_{x\sim \nu}[\log(d_\phi(x)] + \mathbb{E}_{z \sim \mathrm{N}(0,\Id_p)}[\log(1 - d_\phi(g_\theta(z)))] \eqsp. 
\end{equation}
We also tried with the hinge version of the adversarial loss, as  proposed in
\cite{lim2017geometric} and \cite{tran2017deep} and we obtained similar results.

\paragraph{Score-based generative modeling.} Our diffusion model
is similar to the model introduced by \cite{song2019generative}. 
The neural network $ s_{\theta} $ learns to approximate, for a
given $ x $ and a given $ \sigma $,
the score $ \nabla_x p_{\nu}(x,\sigma) $  of the data distribution convoluted
with a Gaussian distribution of standard deviation $ \sigma $. This is done by first defining 
a geometrical progression $ \{\sigma_i\}^L_{i = 1} $
where $ L = 10 $ and where the ratio is chosen such that $ \sigma_L \approx 0.01 $, and then
minimizing the Fischer divergence \cite[]{vincent2011connection}
\begin{equation}
\textstyle{
\mathcal{L}_{\mathrm{SGM}}(\theta) = \mathbb{E}_{\sigma \sim 1/L\sum \delta_{\sigma_i}} \left[\sigma^2
\mathbb{E}_{x \sim \nu}\left[\mathbb{E}_{y \sim \mathrm{N}(x,\sigma^2\Id_d)}  \left [ \left\| s_{\theta}(y,\sigma)  + (y - x)/\sigma^2 \right\|^2  \right]\right]\right]} \eqsp . 
\end{equation}
Then, in order to generate data, we use an annealed Langevin dynamic scheme 
as defined in \cite{song2019generative}. In the Langevin dynamic, 
we set the step size to $ 2 \times 10^{-5} $ and the number of step for each 
value of $ \sigma $ to $ 100 $  as in \cite{song2019generative}. 

\paragraph{Influence of generator depth.}
For this experiment, we increase the number of layers of the VAE decoder and the GAN generator
from $ 2 $ to $ 6 $. At each new layer, we double the number 
of neurons at the previous layer. For instance, the 
generative network with $ 2 $ layers is thus an MLP of shape $ (1,128,1) $
and the one with $ 6 $ layers is an MLP of shape $ (1,128,256,512,1024,2048,1) $. 
Specifically to the GAN model, we also increase the number of layers in the discriminator
in order to keep the dynamic between this latter and the generator balanced. As in the $ 3$-layers
case, the discriminator is one layer deeper than the generator. 
For instance, the discriminator associated to the generator with $ 2 $ layers 
is an MLP of shape $ (1,256,128,1) $.

\paragraph{Influence of generator architecture.}
For this experiment, we use a feed-forward MLP of 
shape $ (1,256,256,256,1) $ as backbone. Then we add two additive pre-activation skip-connections of type "resnet"  between the first and the second hidden layers and between the second and the third hidden layers. Finally, we replace the two previous additive skip-connections of type "resnet" by concatenation pre-activation skip-connections of type "densenet".

\subsection{Synthetic mixture of Gaussians on MNIST}
\label{sec:expedetgm}

\paragraph{Models details.}
We adapt our three models to MNIST, changing mainly 
the networks architectures and making small modifications 
that we describe in what follows. We base the architecture of the GAN and the VAE 
on DCGAN \cite[]{radford2015unsupervised},
using the generator as decoder and the discriminator as encoder for our VAE. This 
is done by doubling the last layer of the discriminator in order 
that the VAE encoder has two outputs as in the univariate case. For the 
GAN model, we replaced the convolutional discriminator by 
a simple MLP of shape $ (784,512,256,128,1) $ because the dynamic between the generator 
and the discriminator seemed unbalanced otherwise. 
We also update our GAN model
using some features of SAGAN \cite[]{zhang2019self}: 
applying spectral normalization 
on the discriminator and using the unconditional hinge version of the adversarial loss 
\cite[]{lim2017geometric,tran2017deep}:
\begin{align}
\mathcal{L}^{d_{\phi}}_{\mathrm{GAN}} &= - \mathbb{E}_{x \sim \nu}[\min\{0,-1+d_{\phi}(x)\}] - \mathbb{E}_{z \sim \mathrm{N}(0,\Id_p)}  [\min\{0,-1-d_{\phi} (g_{\theta(z)}\}]  \eqsp, \\
\mathcal{L}^{g_{\theta}}_{\mathrm{GAN}} &= - \mathbb{E}_{z \sim \mathrm{N}(0,\Id_p)}[d_{\phi}(g_{\theta}(z))] \eqsp.
\end{align}
Such loss function is equivalent to minimize the Kullback-Leibler divergence between 
the generated distribution and the data distribution. 
%During training, we observed 
%that the convolutional discriminator converged too quickly to let the 
%generator become good at generation so we replaced it by the previous MLP discriminator 
%of the univariate case and we got better results.
The VAE decoder and the GAN generator have $ 1713088 $ learnable parameters. 
For the score network architecture, we use the vanilla U-Net 
architecture \cite[]{ronneberger2015u} in which we double the number of channels at each layer, 
we add group normalization \cite[]{wu2018group} after each 
convolution and we replace the ReLU non-linearies by
SiLU \cite[]{elfwing2018sigmoid}. As in the univariate case, 
we use positional encoding \cite[]{vaswani2017attention} followed by
a MLP block of shape $ (1,16,32) $ to incorporate the noise information
at each layer. The score network has $ 1607392 $ learnable parameters. 
For inference, we use the same Langevin dynamic scheme as above
with the same hyperparameters as in the univariate case. The three models are trained during $ 100 $ epochs with
a batch size of $ 128 $ using ADAM with 
a momentum of $ 0.9 $ and a learning rate of $ 2\times10^{-4} $. 

\paragraph{Additional details.}
The histograms of projection on the line passing through the mean of each Gaussians are
obtained using $ 20000 $ generated samples. To assign a color to each bin of the histograms, we
train a simple MLP of shape $(784,1024,50,10)$ as classifier on MNIST. The classifier is 
trained during $ 10 $ epochs using again ADAM with a momentum of $ 0.9 $ and a learning rate of $ 2\times10^{-4} $ and 
reaches an accuracy of $ 0.98 $ on the test set.

\subsection{Subset of MNIST} 
\label{sec:expedetmnist}

\paragraph{Models details.} Since the dataset is more complex than before, we use bigger models. For the score
network, we use the architecture defined in \cite{ho2020denoising}, in which
we set the number of channels to $ 64 $ instead to $ 128 $ and we remove 
the self attention layers \cite[]{wang2018non} for computational resource purposes. The score
network has $ 6072065 $ learnable parameters. Again, we use an annealed Langevin dynamic scheme for inference with 
the same hyperparameters as before. For the VAE and the GAN, we use the same architecture as before, using 
this time the convolutional discriminator of DCGAN, and quadrupling
the number of channels at each layer. This is mainly done in order to 
scale the generator/decoder to the score network. Hence the VAE decoder/GAN 
generator has $ 7151104 $ learnable parameters. We train all three models
during $ 600 $ epochs with a batch size of $128 $ using ADAM with  a momentum of $ 0.9 $ and a learning rate of $ 2\times10^{-4} $.

\paragraph{Additional details.} 
We use the deep Wasserstein embedding proposed by \cite{courty2018learning} in order to visualize
histograms of projection in the Wasserstein space. We use the exact same network 
architecture and the same training procedure that in \cite{courty2018learning}: first, one million 
pairs of digits of MNIST are chosen randomly, in which $ 700000 $ are kept for the training set, $ 200000 $
for the test set, and $ 100000 $ for the validation set. We normalize each image in order to consider it 
as a two-dimensional distribution and we compute the $ 1 $-Wasserstein distance
for each pair. Then, we train an autoencoder in a supervised manner in a way
that the images at output of the autoencoder are close to the images in input, and that 
the euclidean distance between two vectors in the latent space is close to the 
$ 1 $-Wassertein distance between the two corresponding images of MNIST. As in \cite{courty2018learning},
the latent Wasserstein space is of dimension $ 50 $ and the autoencoder is trained 
during $ 100 $ epochs with a batch size of $ 100 $ and with an early stopping criterion. Again, 
we use ADAM with a momentum of $ 0.9 $ and a learning rate of $ 10^{-3} $. We use
the same classifier as before to assign color to each bin of the histograms. Finally,
the histograms of projection on the line passing through the deep Wasserstein barycenters of all $ 3 $ and $ 7 $ are
obtained using $ 20000 $ generated samples.

\section{Additional experimental results}
\label{sec:expeplus}
In the following, we provide additional experimental results. First, we compare estimates of the bounds of \Cref{thm:lowerboundtv}, \Cref{coro:gm2}, and \Cref{thm:lowerboundkl} to estimates of the total variation distance 
and the Kullback-Leibler divergence in the univariate case. Then 
we study the possible correlation between the size of the score network 
and the tendency of the score-based model to generate unbalanced modes. 
Finally, we provide additional visualizations of histograms of 
generated distributions for the univariate case and generated samples
for the experiments on MNIST.

\subsection{Bounds on TV distance and KL divergence in the univariate case}
\begin{figure}[h!]
  \centering
    \begin{tabular}{cc}
    \hspace{1.4em} \vspace{-0.2 em} {\small $d_{\mathrm{TV}}(g_{\#}\mu_p,\nu) $} & \hspace{-1.5em} \vspace{-0.2 em} \hspace{2em} {\small \quad $ d_{\mathrm{KL}}(g_{\#}\mu_p||\nu) $} \\
    \hspace{-1.8em} 
     \includegraphics[width=0.46\textwidth]{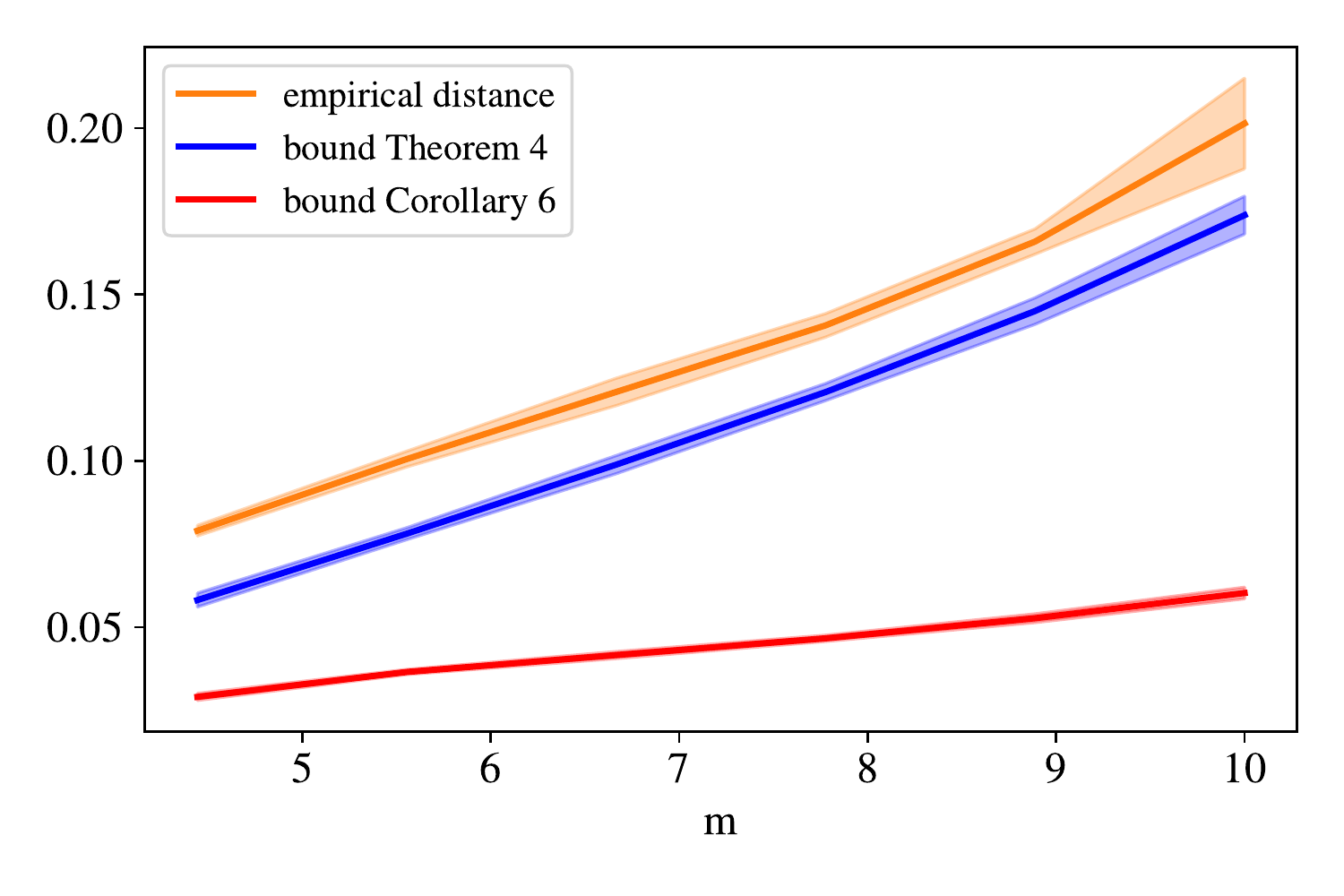}   \hspace{-1.5em} \vspace{-0.5 em}
    & 
    \includegraphics[width=0.46\textwidth]{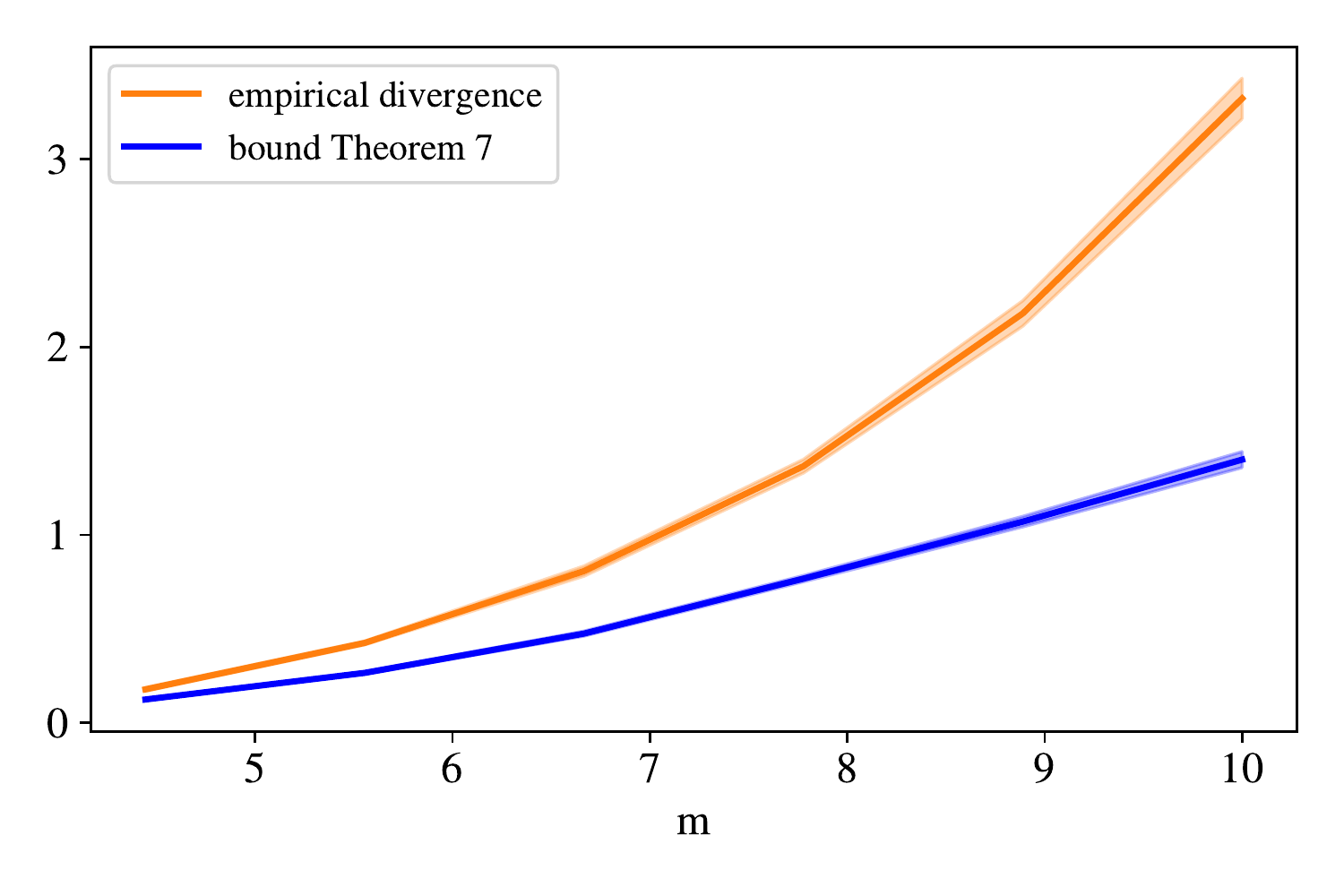}    \vspace{-0.5 em}
\end{tabular}
  \caption{total variation distance (left) and Kullback-Leibler divergence (right) for the
  VAE (in orange) and estimates of the respective lower bounds from \Cref{thm:lowerboundtv}
  and \Cref{thm:lowerboundkl} in blue. The lower bound of \Cref{coro:gm2} is also plotted
  in red for the total variation. The experiments are 
  averaged over $ 10 $ runs and the colored bands correspond to +/- the standard deviation.}\label{fig:expe1_dim1_dist}
 \end{figure}
 
In this experiment, we compare estimates of the bounds of \Cref{thm:lowerboundtv}, \Cref{coro:gm2}, and 
\Cref{thm:lowerboundkl} to estimates of the total variation distance 
and the Kullback-Leibler divergence. We only provide results for the VAE since
the bounds are not interesting for the GAN since they are consequences of interpolation between modes due to a 
small Lipschitz constant of the generative network. Yet this latter in the GAN case achieves a
large Lipschitz constant so does not interpolate significantly. To estimate
empirically the total variation distance and the Kullback-Leibler divergence,
we used their respective analytical formula
\begin{align}
\textstyle{
d_{\mathrm{TV}}(g_{\#}\mu_p,\nu)} &= (1/2)\textstyle{\int_{\rset}|p_{g_{\#}\mu_p}(x)-p_{\nu}(x)|dx \eqsp,} \\
\textstyle{d_{\mathrm{KL}}(g_{\#}\mu_p||\nu)} &= 
\textstyle{\int_{\rset}p_{g_{\#}\mu_p}(x)\log\left(p_{g_{\#}\mu_p}(x)/p_{\nu}(x)\right)dx \eqsp,}
\end{align}
where $ p_{g_{\#}\mu_p} $ and $ p_{\nu} $ are the respective densities of $ g_{\#}\mu_p $ and $ \nu $.
In order to estimate 
the lower bounds of \Cref{thm:lowerboundtv} and \Cref{thm:lowerboundkl}, we
set $ \msa $ of the form $ \ocint{-\infty,-r/2} $ and we perform a grid search on $ r $. 
In \Cref{fig:expe1_dim1_dist}, we can observe 
that the  estimates of the bounds provided by \Cref{thm:lowerboundtv} and \Cref{thm:lowerboundkl}
are not tights. This is possibly because we selected a sub-optimal $ \msa $ but it
most likely follows from the fact that the bounds don't take into 
account that $ g_{\#}\mu_p $ has automatically less mass on the modes than $ \nu $ 
since a significant amount of its total mass is between them. 
One can also observe that the explicit lower bound of \Cref{coro:gm2} 
is much smaller than the bound of \Cref{thm:lowerboundtv}. 
This can be explained by the facts that $ \|m\|/2\sigma $ is probably a sub-optimal 
choice of $r$ and that the bound of \Cref{coro:gm2} minimizes the interpolation between modes
over all the mappings with Lipschitz constant $ \text{Lip}(g) $, regardless 
whether these mappings approximate well $ \nu $ on its modes or not. Since
there is less interpolation if the modes are unbalanced (see \Cref{sec:lower-bound-distance}), 
it is likely that the mappings $ g $ such that $ g_{\#}\mu_p $ is unbalanced 
are affecting the value of this bound in a bad way.

\subsection{Additional examples}

\subsubsection{Univariate histograms}
We provide additional visualizations of histograms of generated data with the three models
for various values of $ m $.

\begin{figure}[h!]
  \centering
  \begin{tabular}{cccc}
    \includegraphics[width=0.22\textwidth]{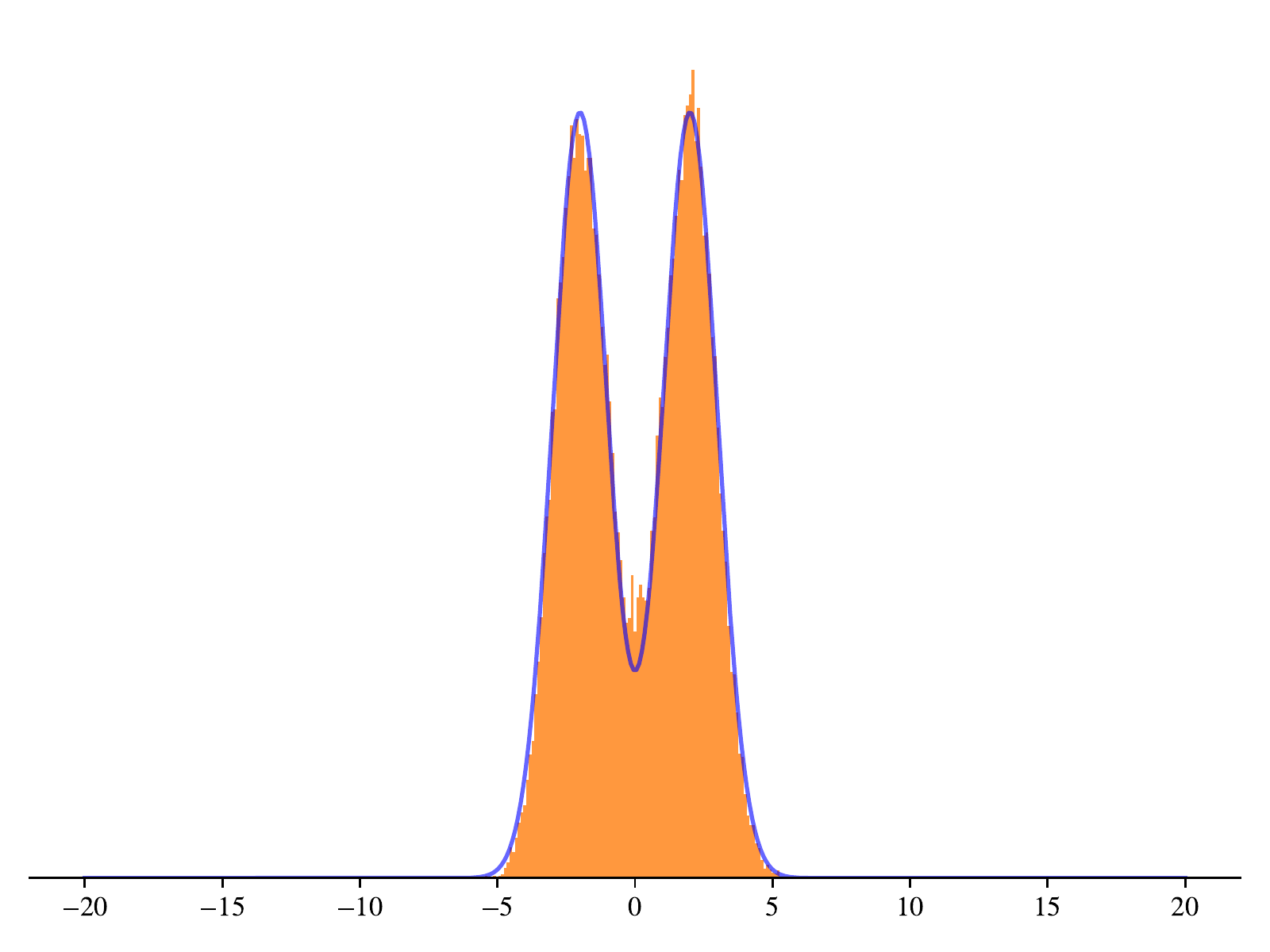} &
    \includegraphics[width=0.22\textwidth]{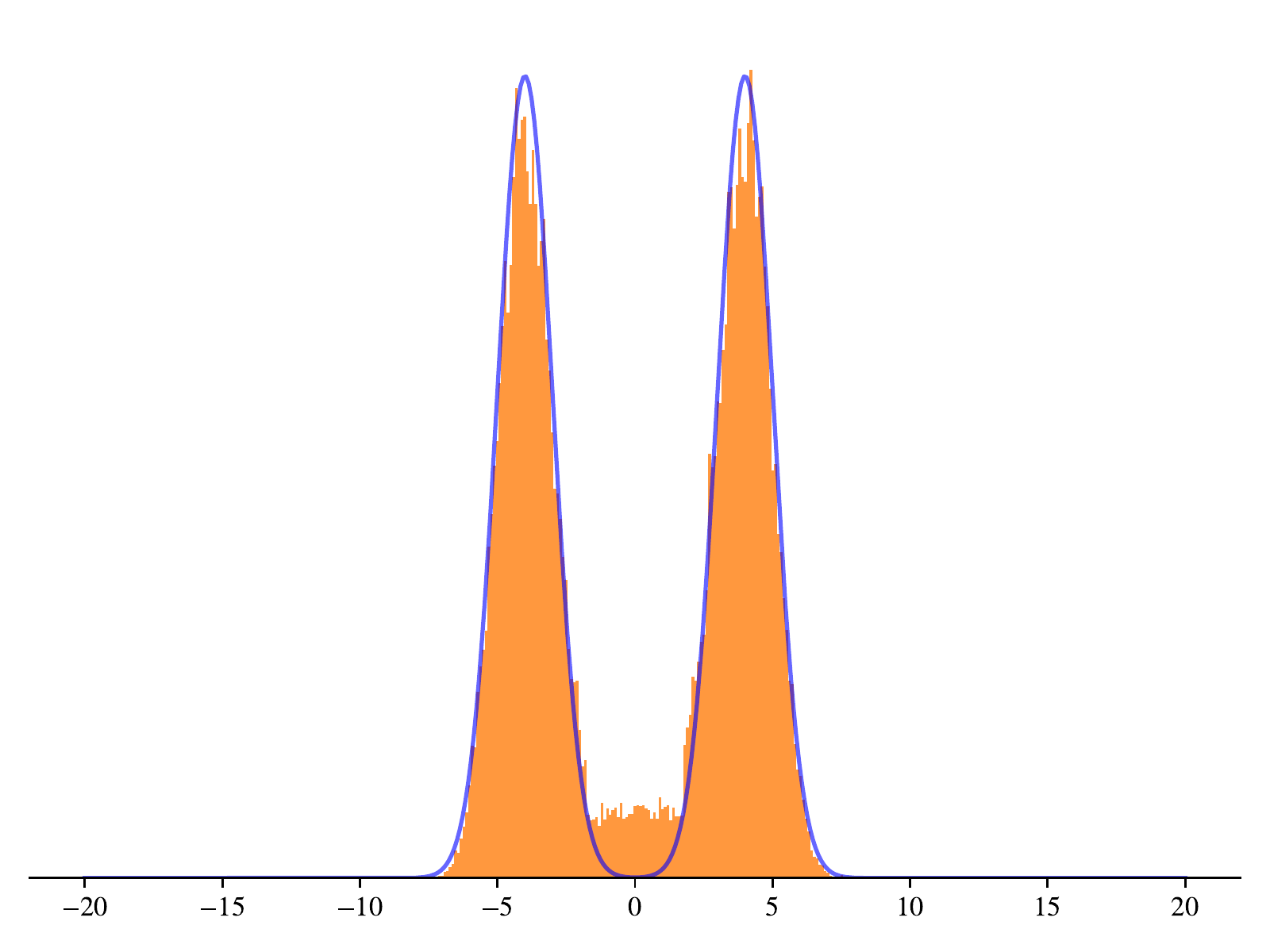} &
    \includegraphics[width=0.22\textwidth]{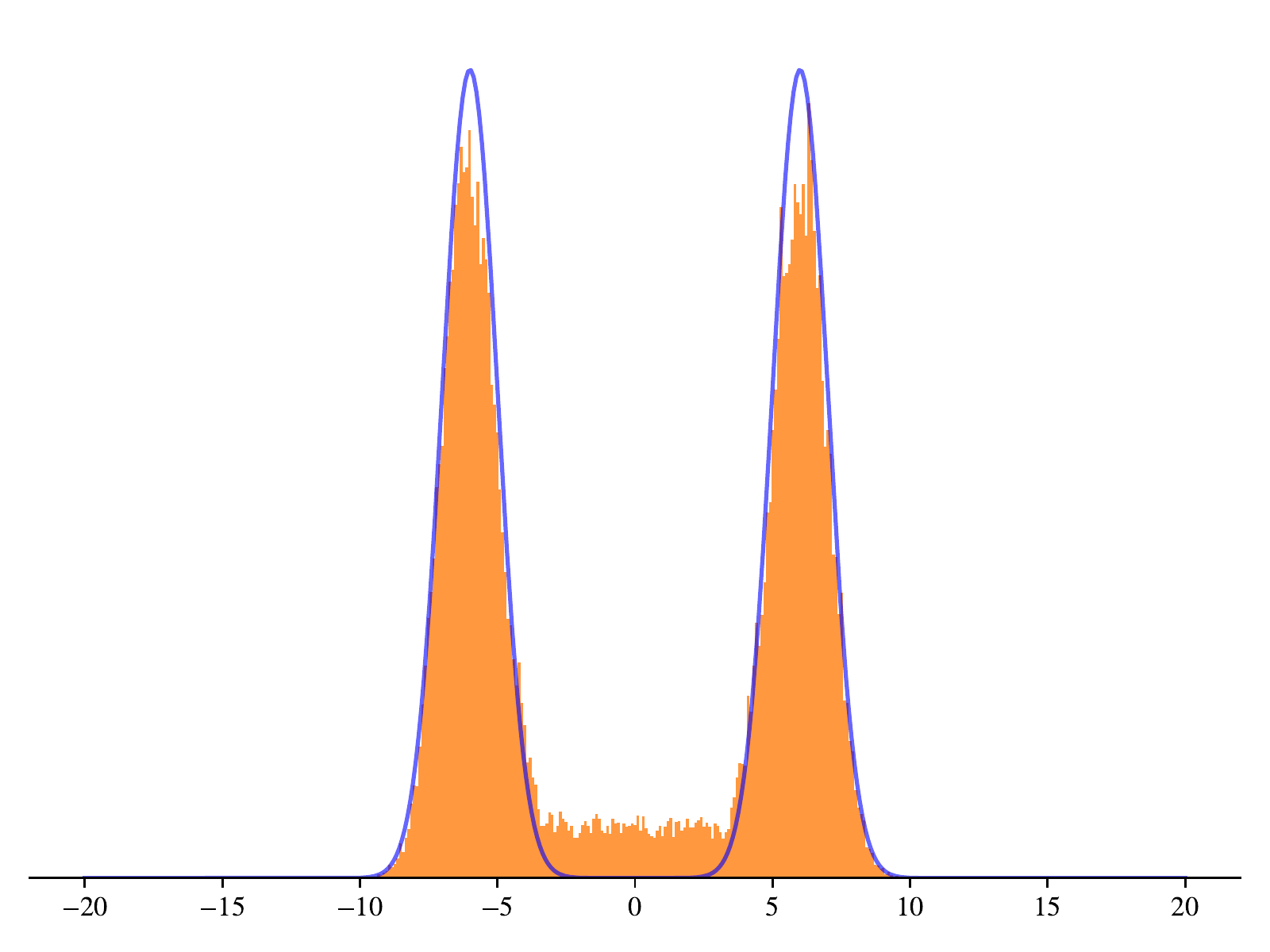} &
    \includegraphics[width=0.22\textwidth]{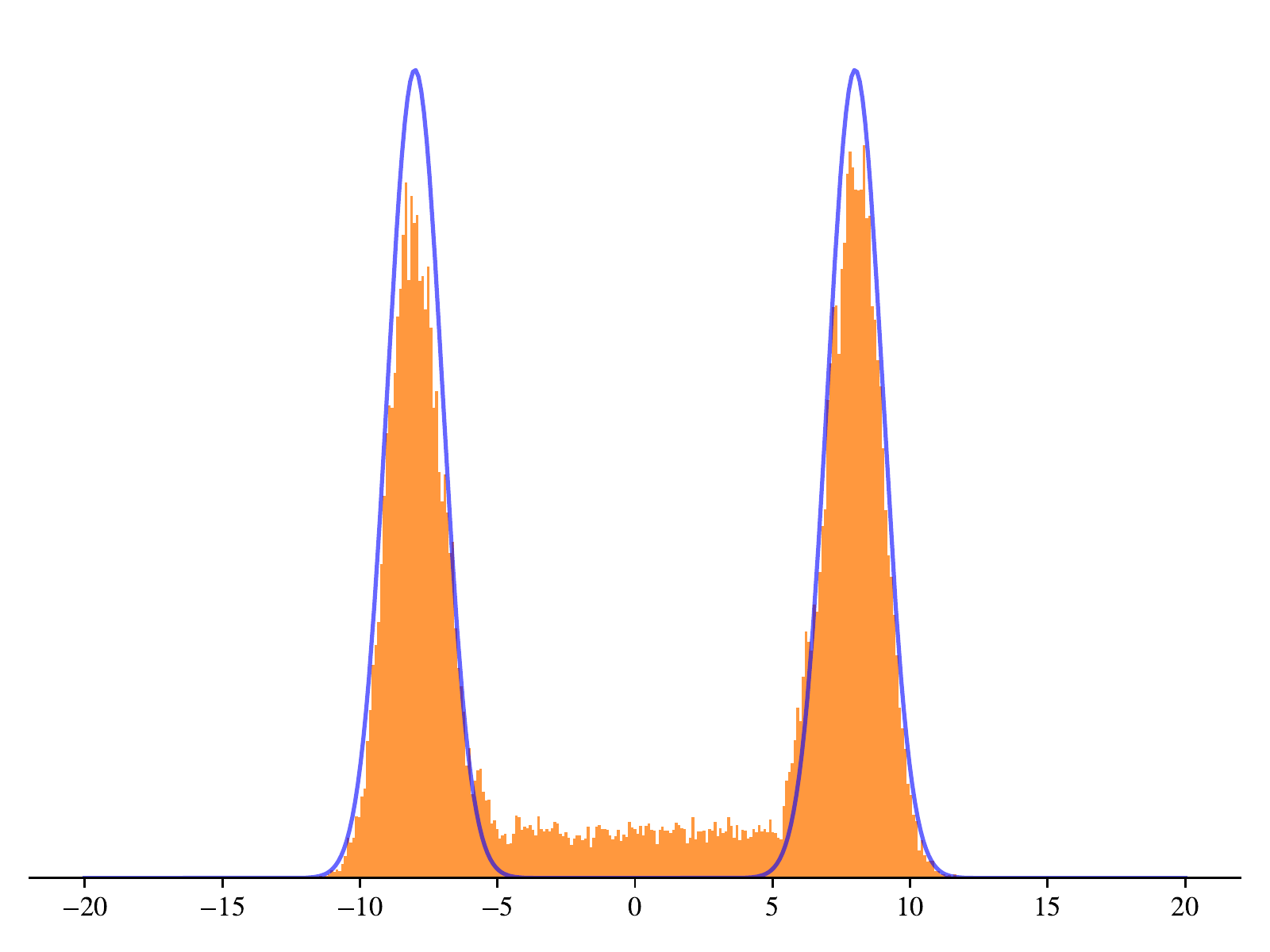}   \\  
    \includegraphics[width=0.22\textwidth]{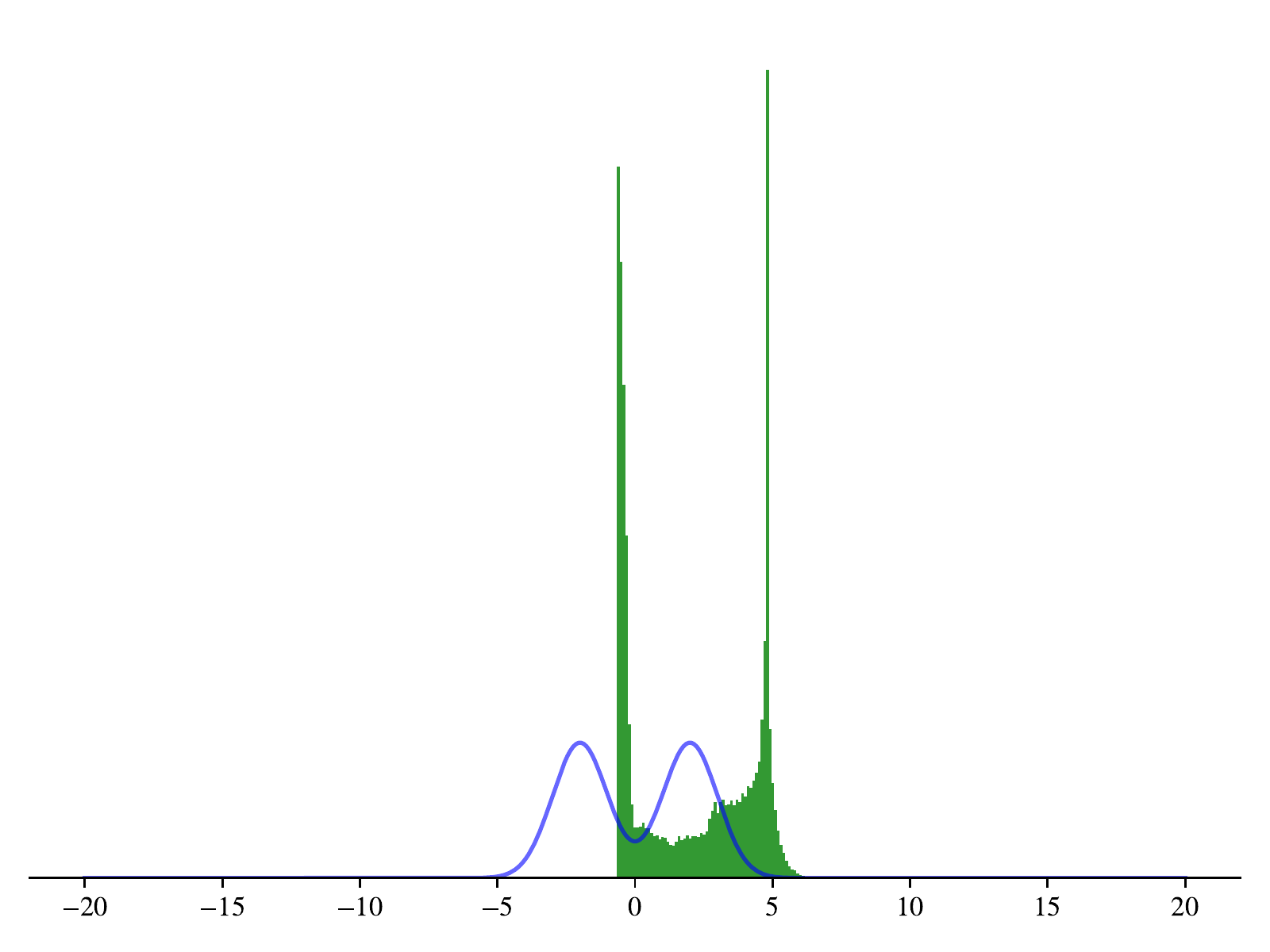} &
    \includegraphics[width=0.22\textwidth]{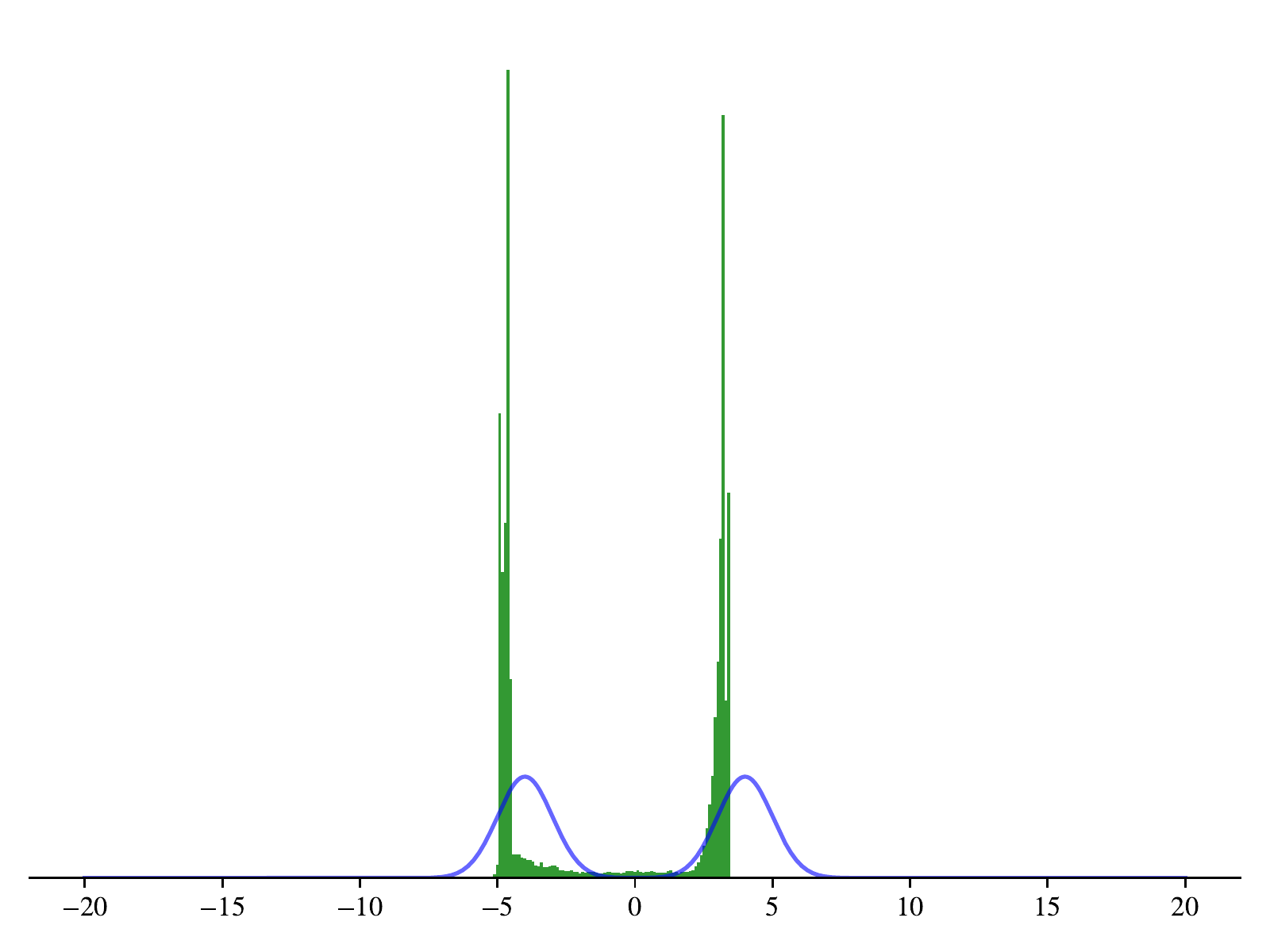} &
    \includegraphics[width=0.22\textwidth]{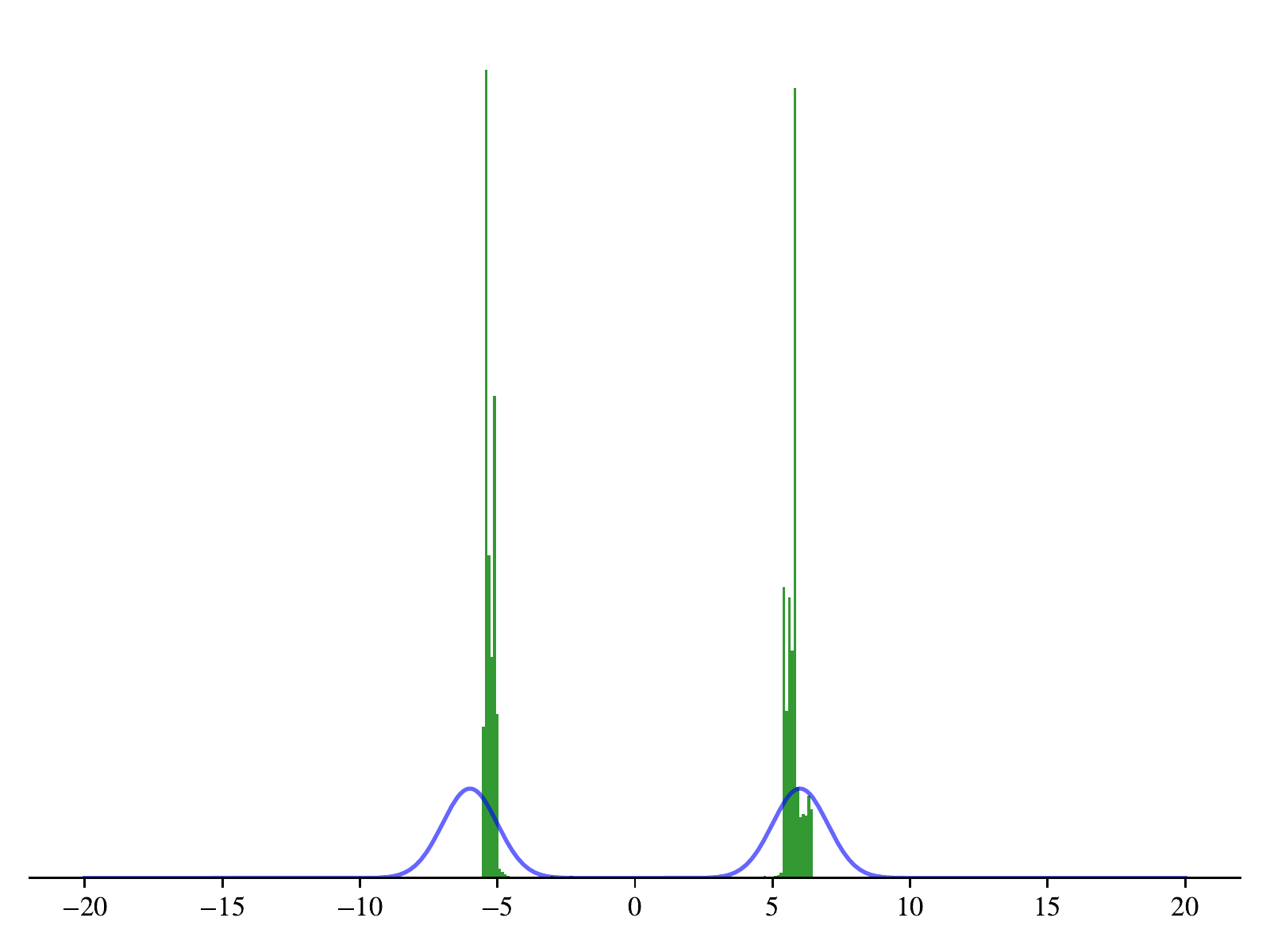} &
    \includegraphics[width=0.22\textwidth]{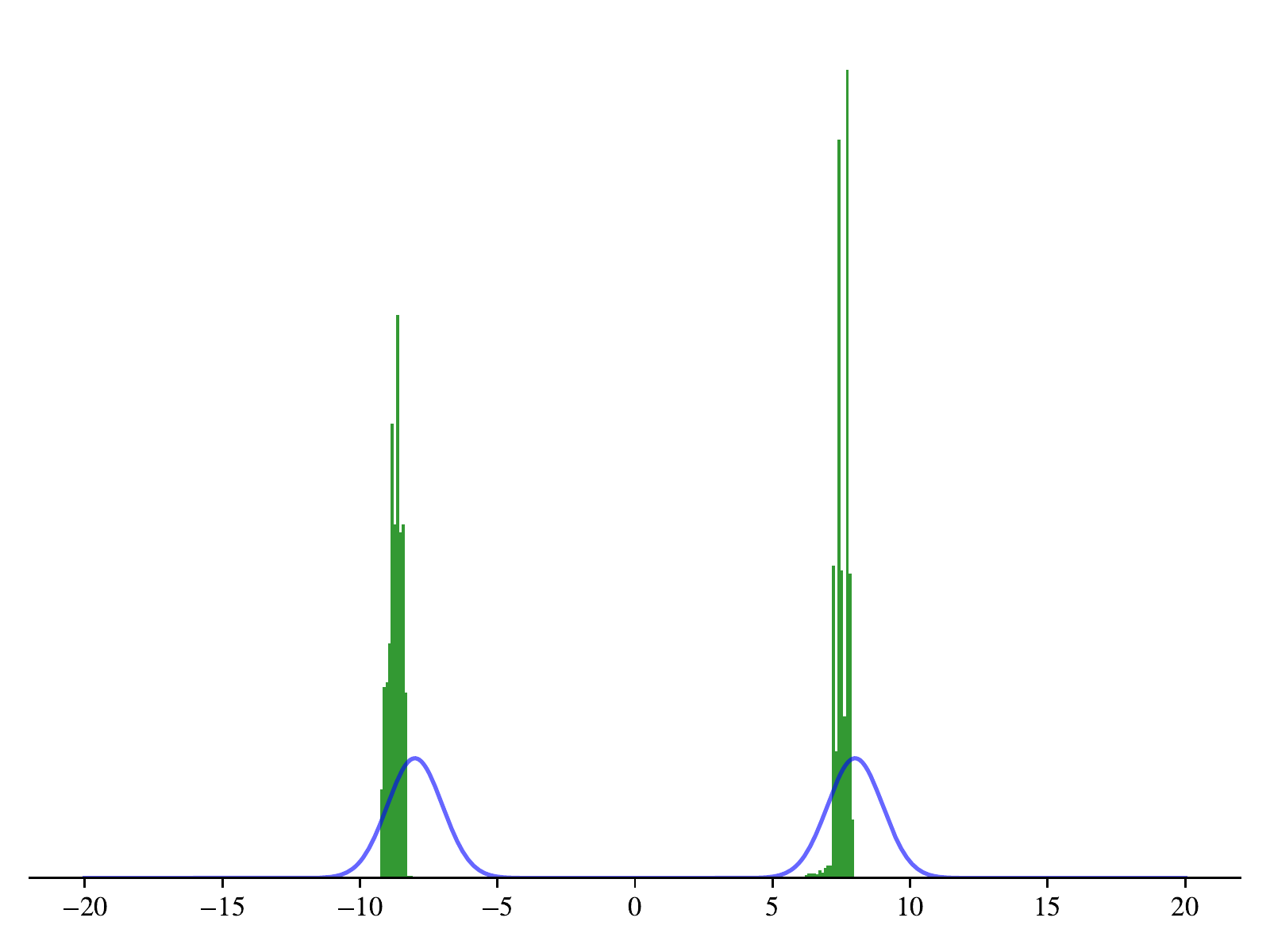}   \\  
    \includegraphics[width=0.22\textwidth]{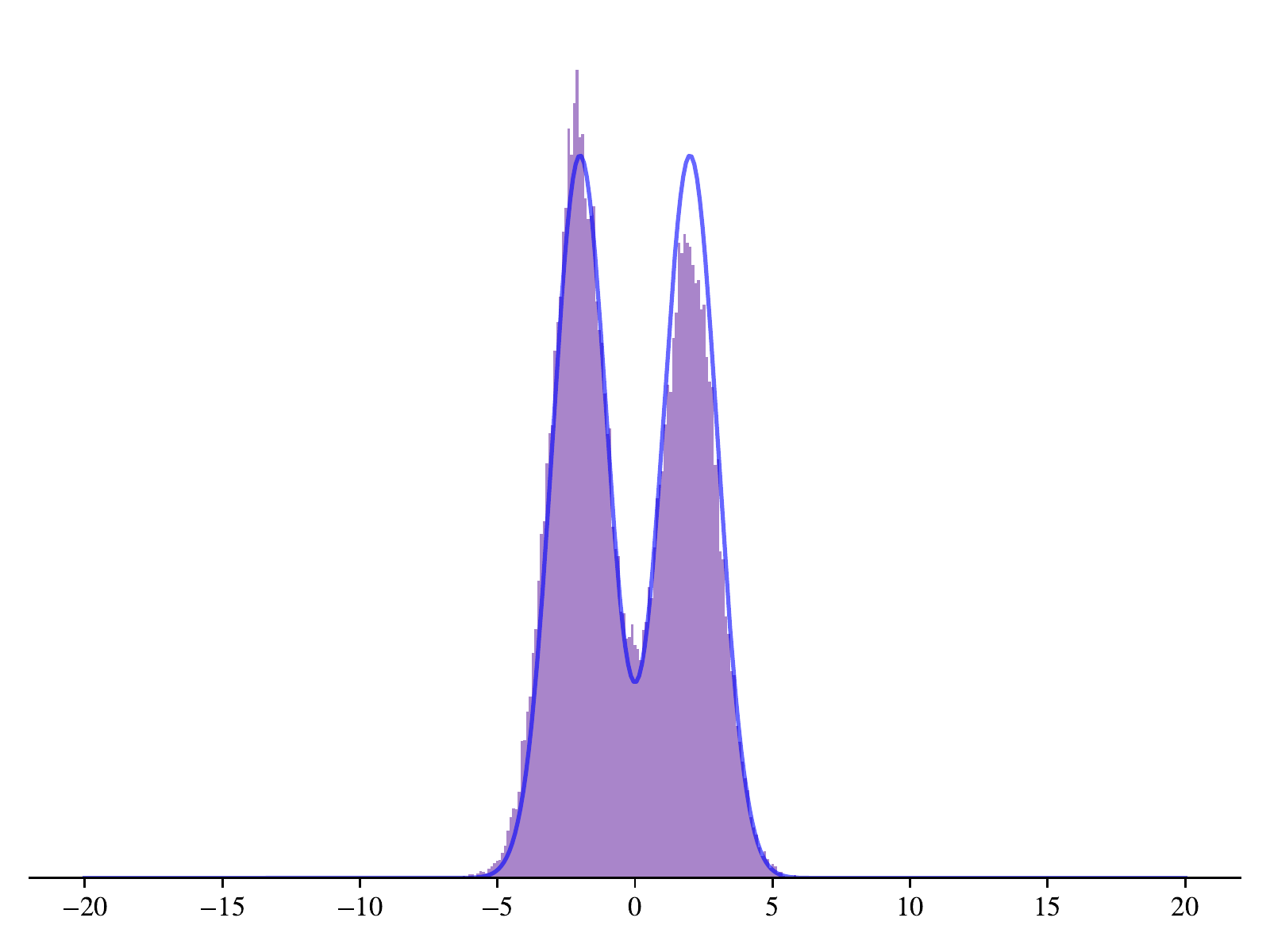} &
    \includegraphics[width=0.22\textwidth]{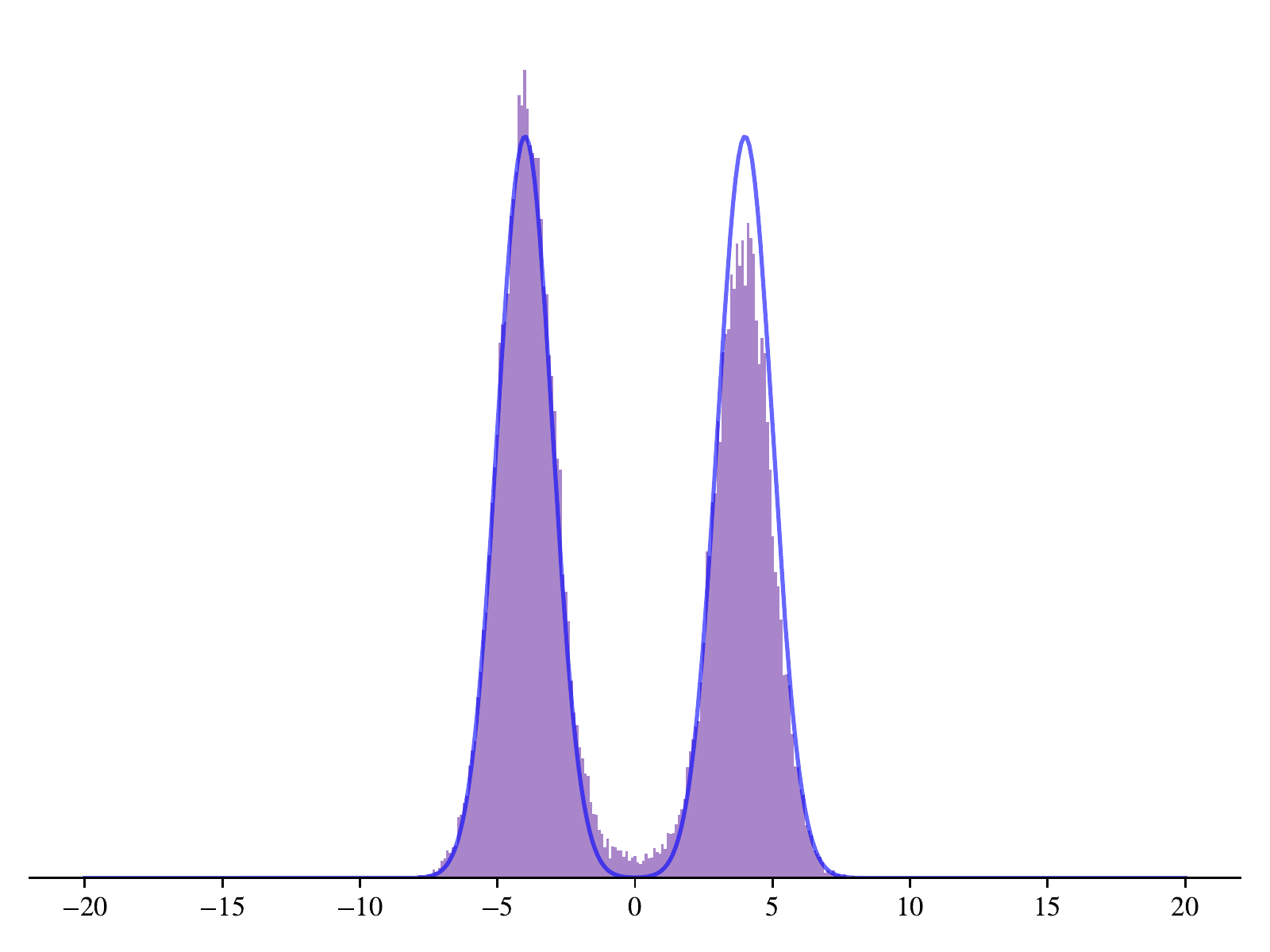} &
    \includegraphics[width=0.22\textwidth]{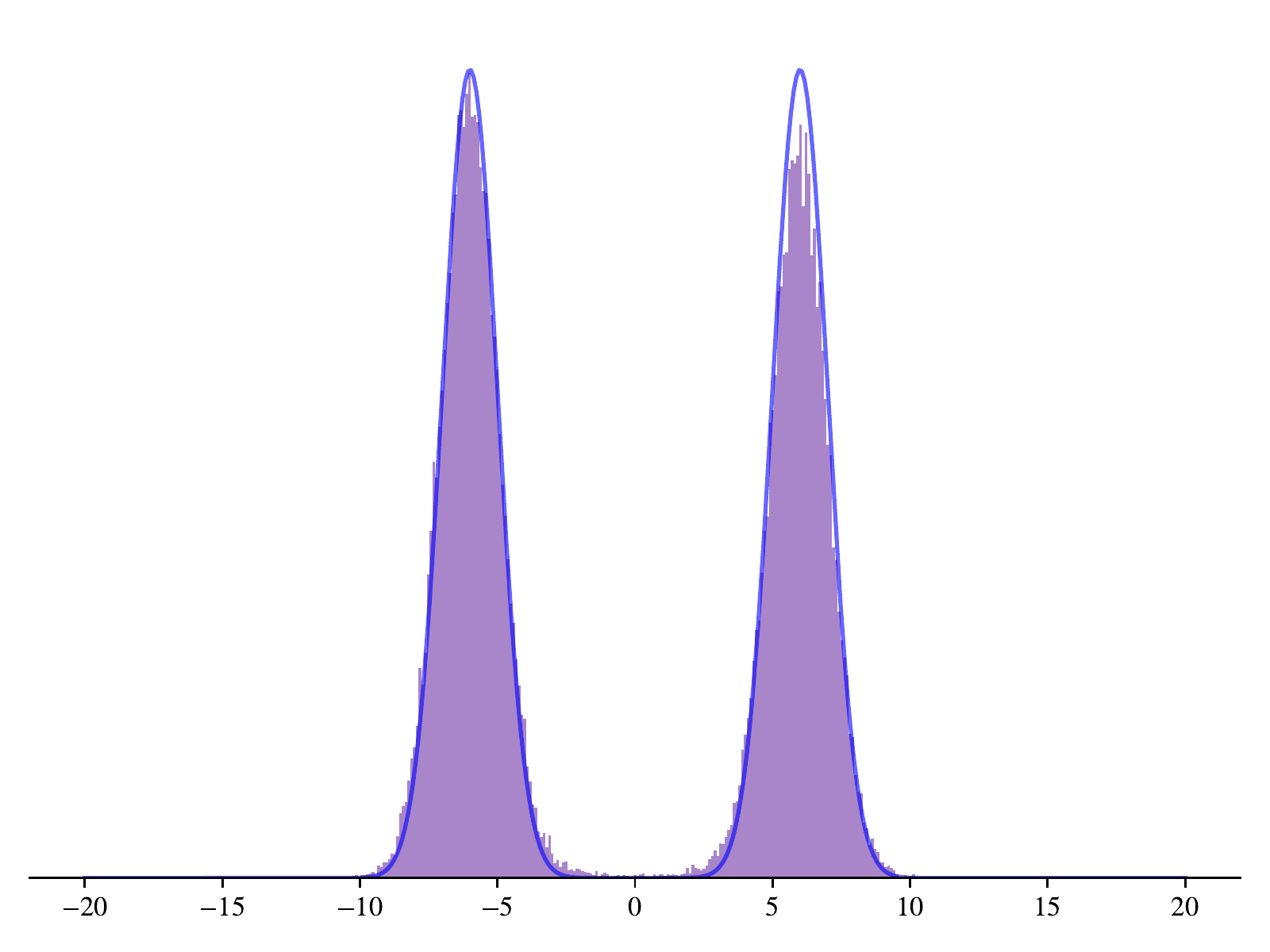} &
    \includegraphics[width=0.22\textwidth]{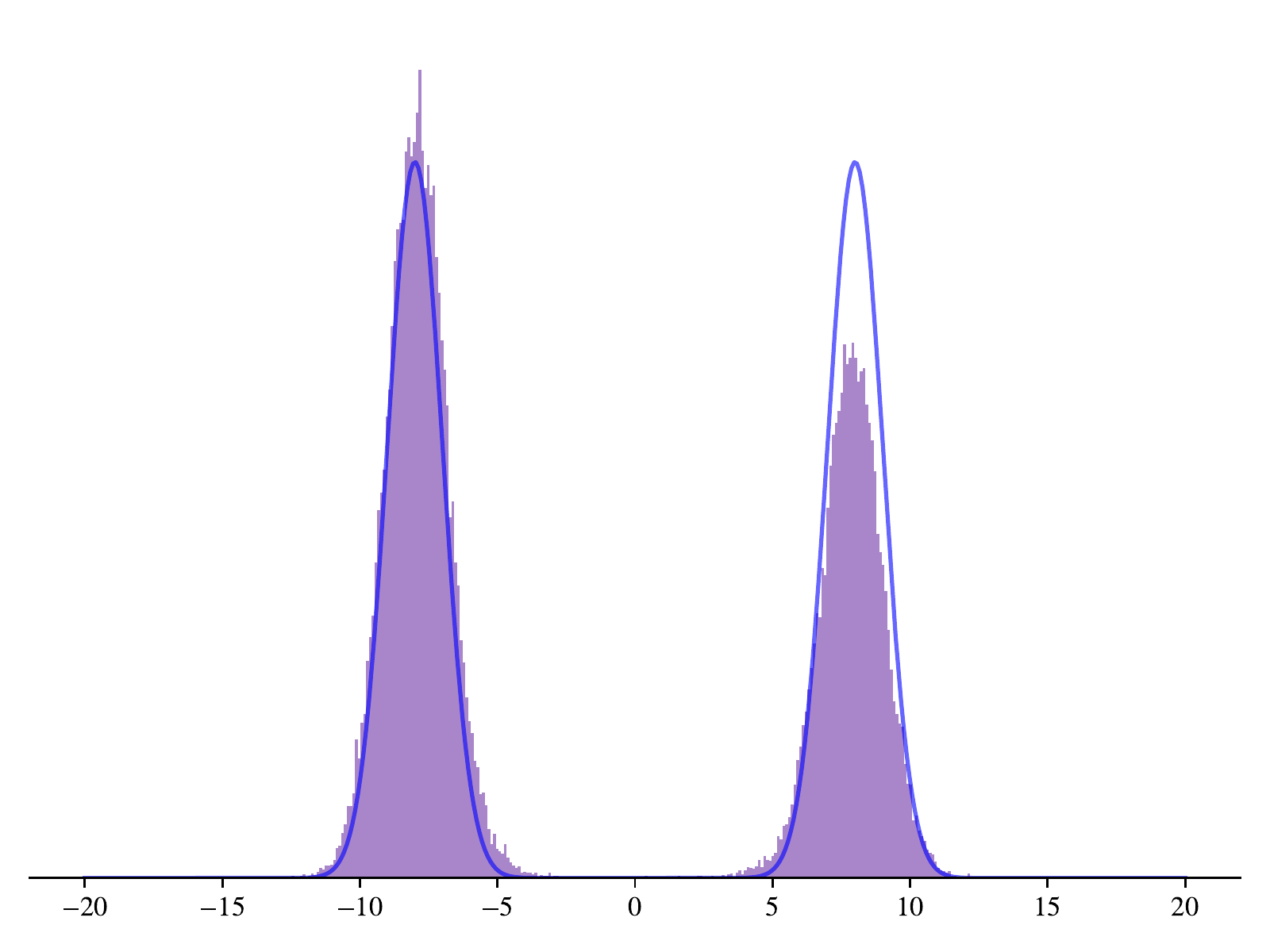}   \\  
  \end{tabular}
\caption{Histograms of distributions generated with VAE (top, in orange), GAN (middle, in green), 
  and with SGM (bottom, in purple) for  $ m = 2$, $ m = 4$, $ m = 6 $ and $ m =8 $.
biblographie stylebiblographie stylebiblographie style  The data distribution densities are plotted in blue.}\label{fig:histo_dim1_2}
\end{figure}

We can observe that the score-based model already generates unbalanced modes,
but the phenomenon is globally less visible than in higher dimensions.
Secondly, we provide additional visualizations of histograms of generated data with GANs
trained with an additional gradient penalty term in the generator loss for 
various values of $ L \approx \mathrm{Lip}(g) $.

\begin{figure}[h!]
  \centering
  \begin{tabular}{cccc}
    \includegraphics[width=0.21\textwidth]{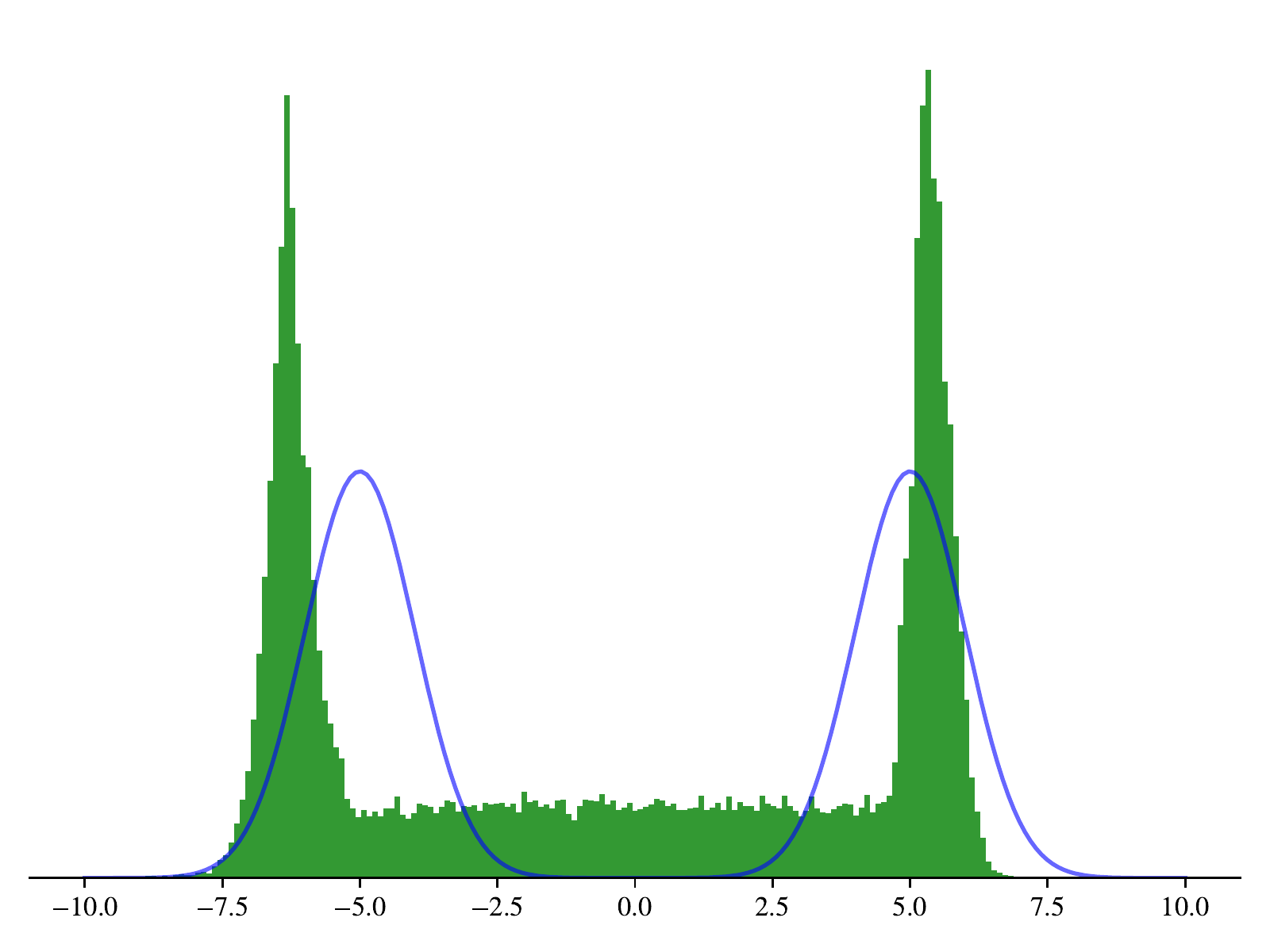} &
    \includegraphics[width=0.21\textwidth]{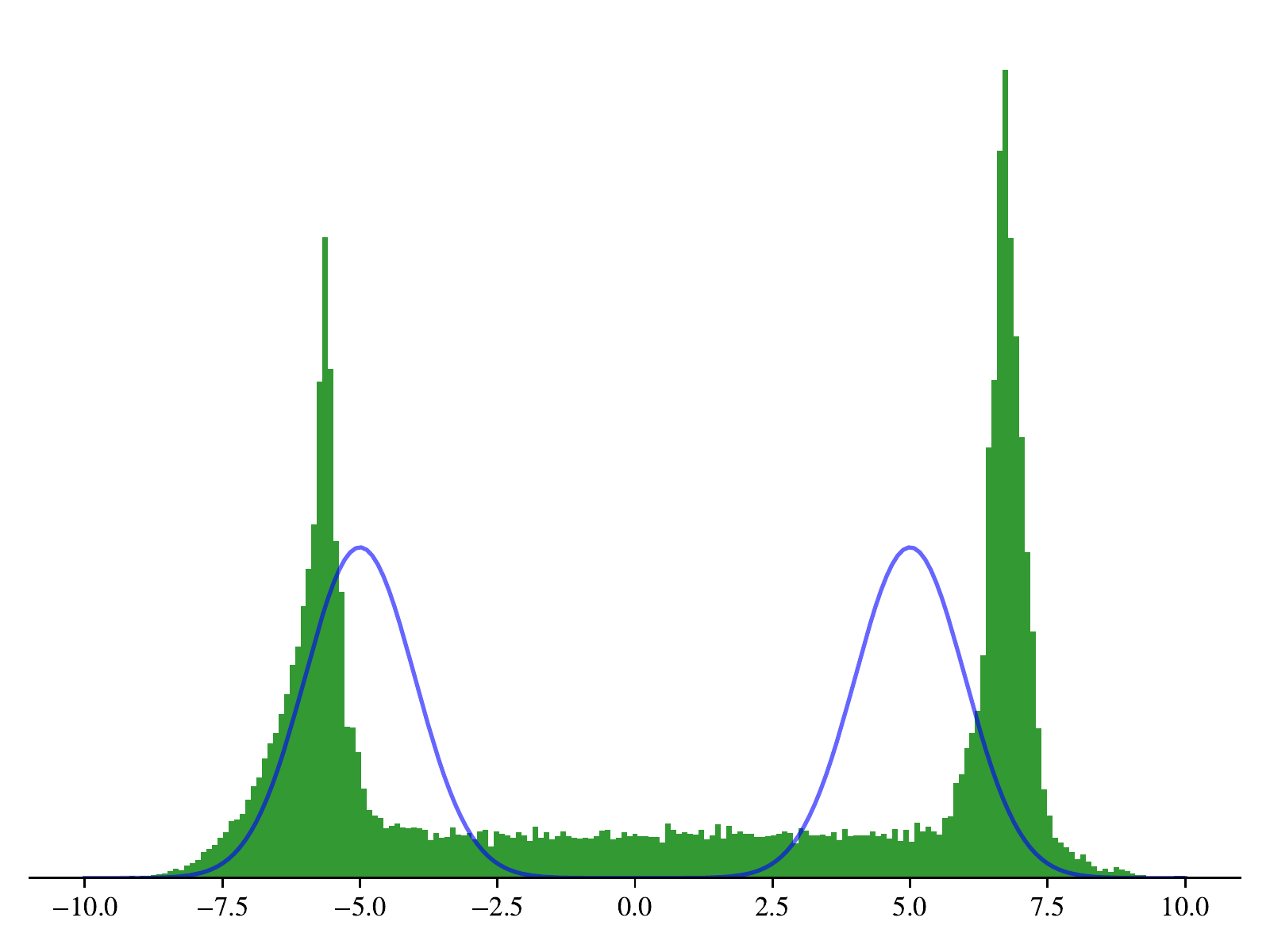} &
    \includegraphics[width=0.21\textwidth]{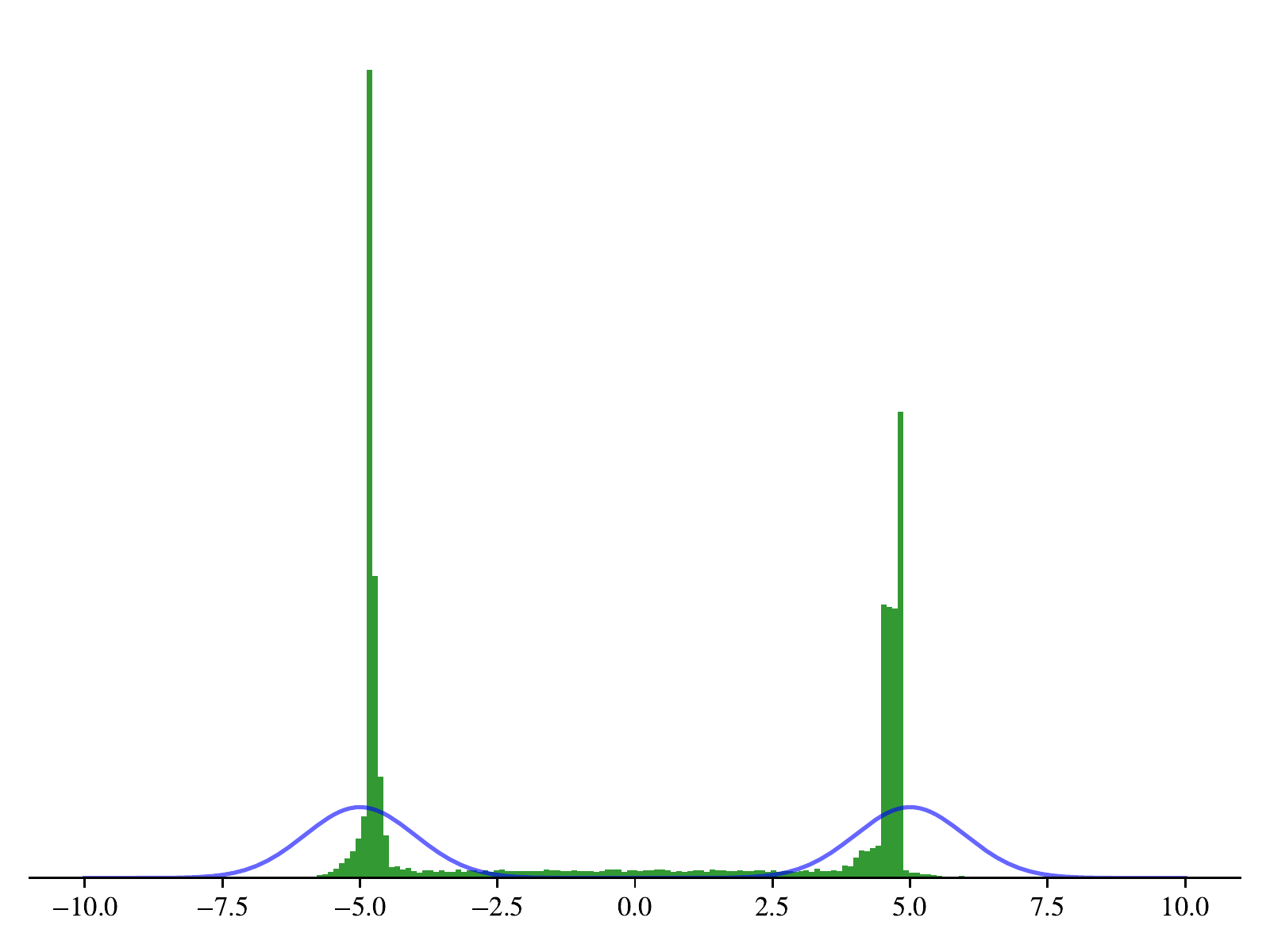} &
    \includegraphics[width=0.21\textwidth]{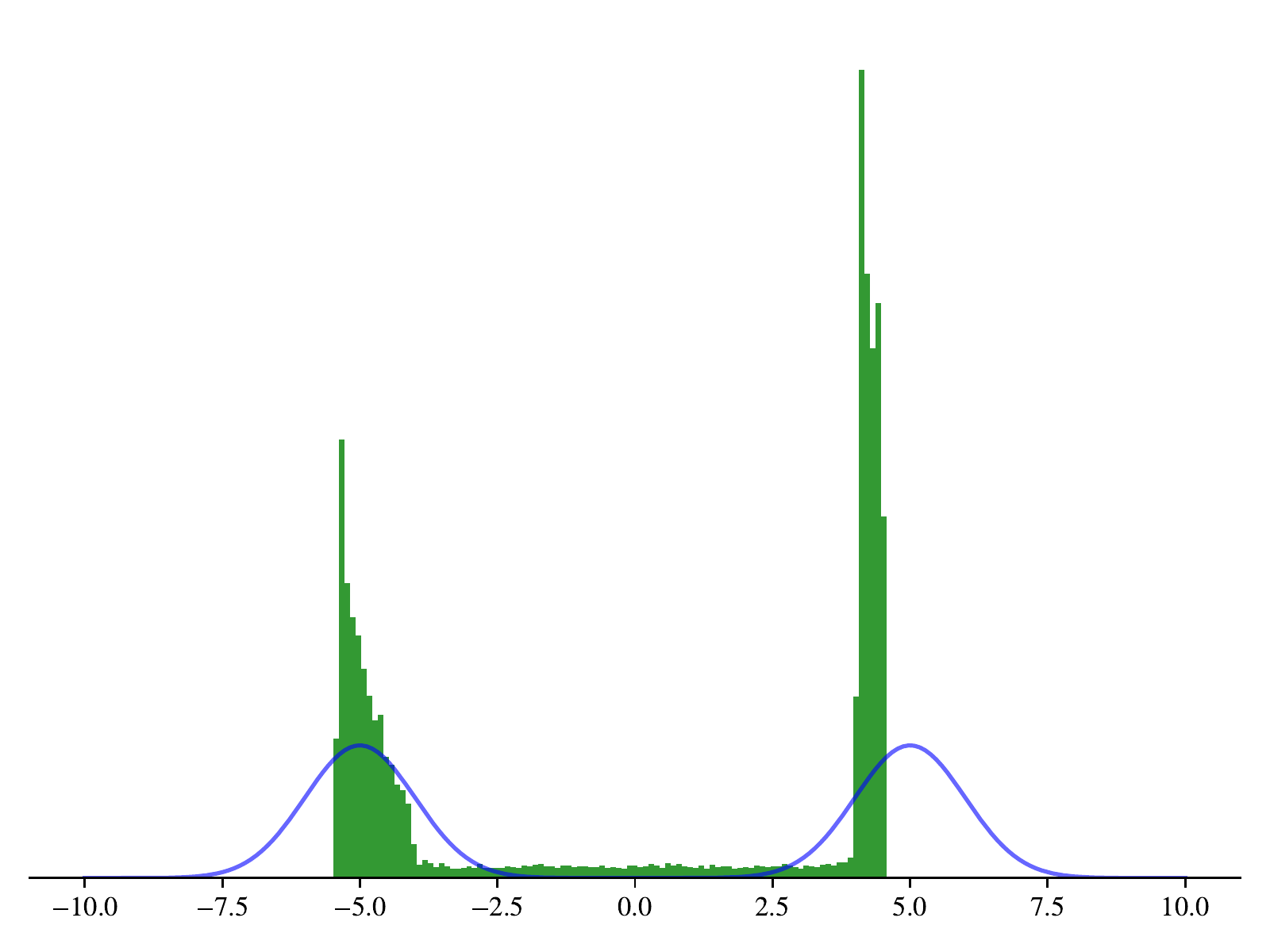}   \\  
  \end{tabular}
\caption{histograms of distributions generated with GANs with  with gradient penalty  for $ \mathrm{Lip}(g) \approx L  = 11 $,
$ \mathrm{Lip}(g) \approx L  = 15 $, $ \mathrm{Lip}(g) \approx L = 19 $ and 
$ \mathrm{Lip}(g) \approx L = 23 $. The data distribution densities are plotted in blue.}\label{fig:gan_SN_2}
\end{figure}
\newpage

\subsubsection{Visualization of generated data}

Finally, we show randomly chosen generated samples with  VAE, GAN and SGM 
on the synthetic mixture of Gaussian on MNIST  and the subset of all $ 3$ and $ 7 $ of MNIST.

\begin{figure}[h!]
    \centering
    \begin{tabular}{ccc}
    VAE  & GAN & SGM \\ 
    \includegraphics[width=0.31\textwidth]{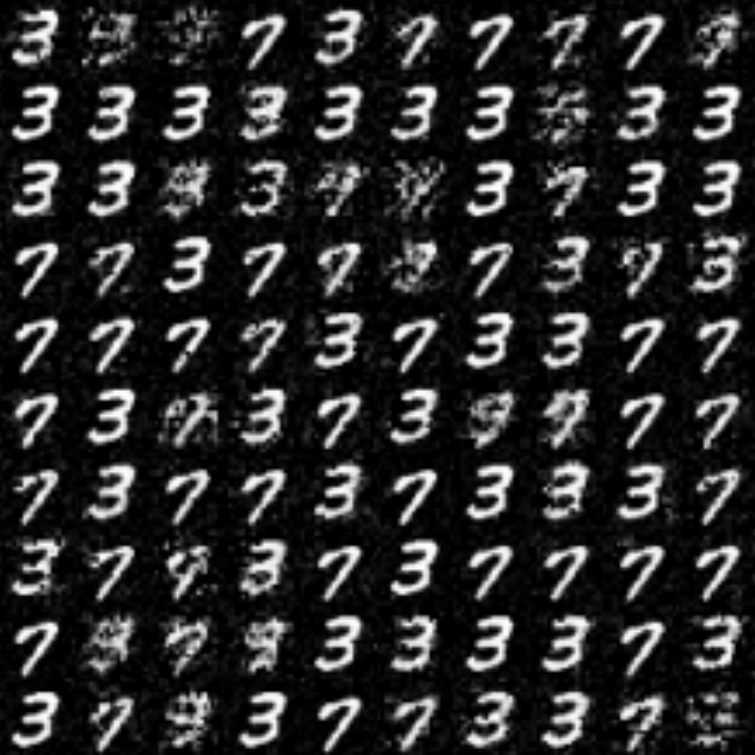} &
    \includegraphics[width=0.31\textwidth]{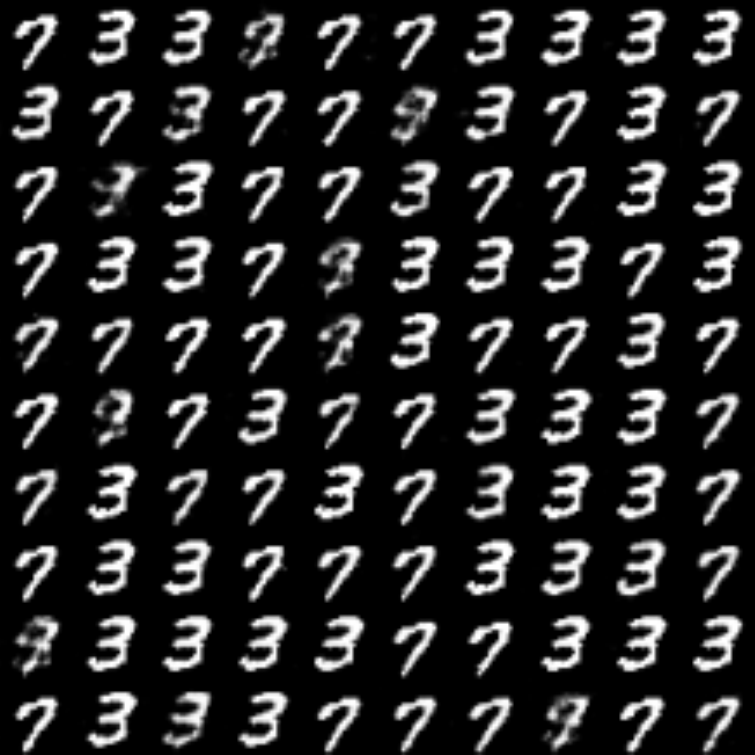} &   
    \includegraphics[width=0.31\textwidth]{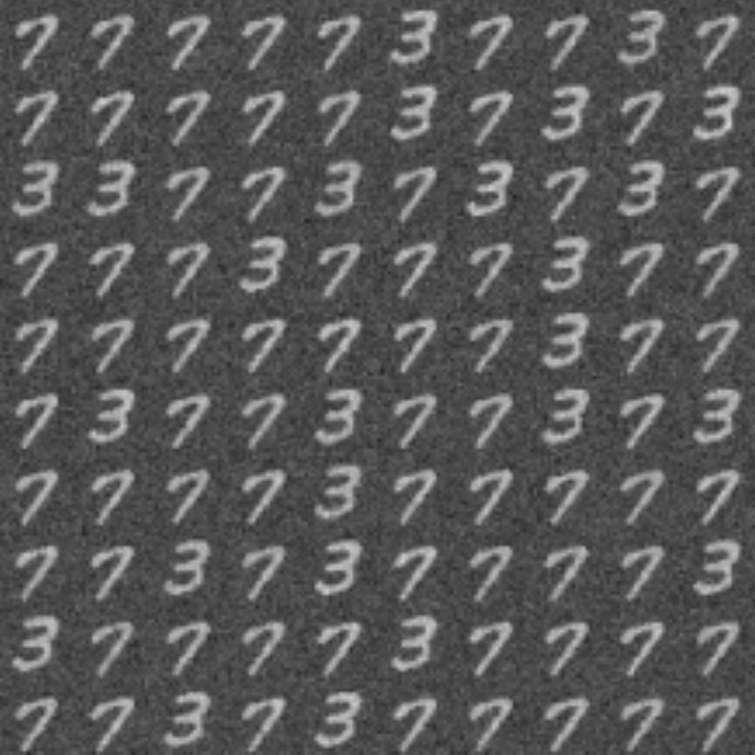} \\
    \includegraphics[width=0.31\textwidth]{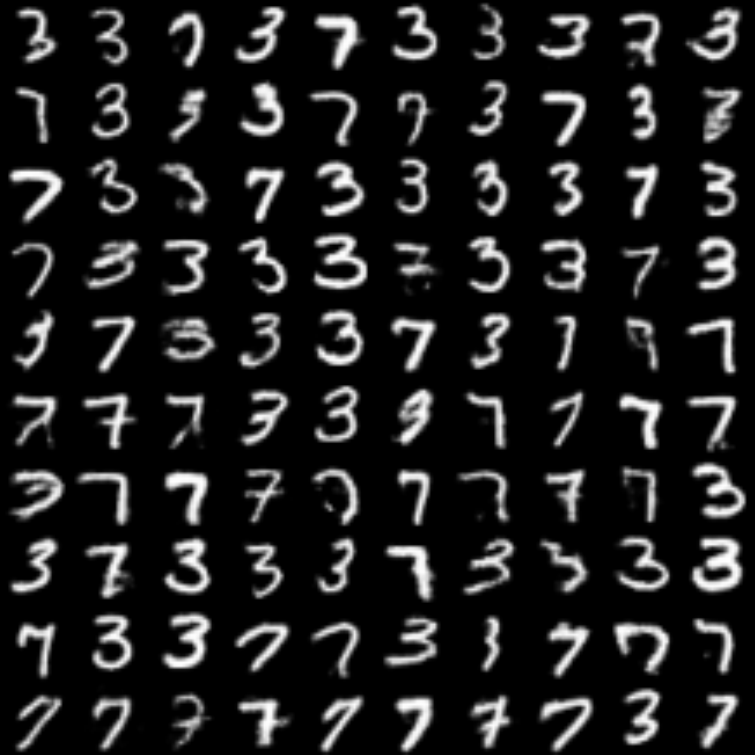} &
    \includegraphics[width=0.31\textwidth]{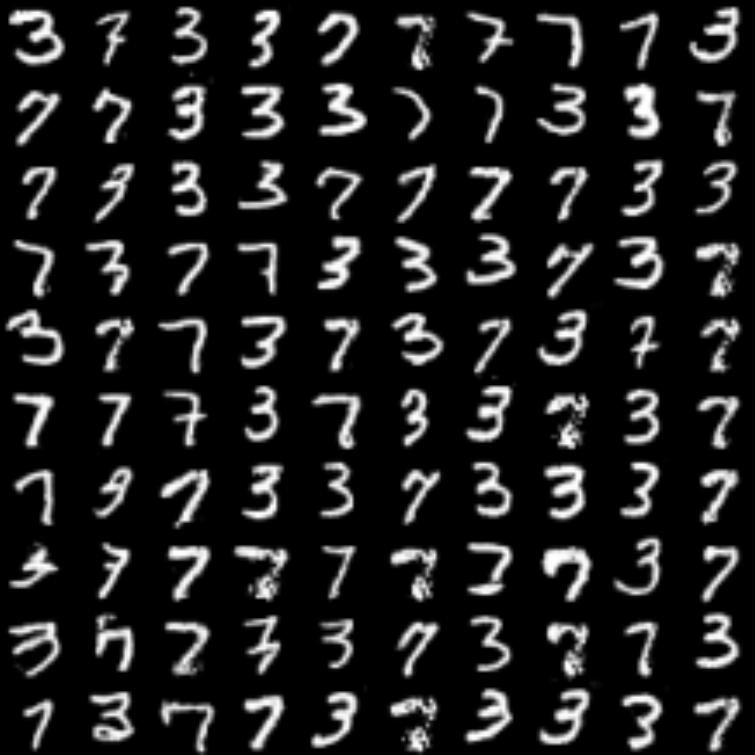} &   
    \includegraphics[width=0.31\textwidth]{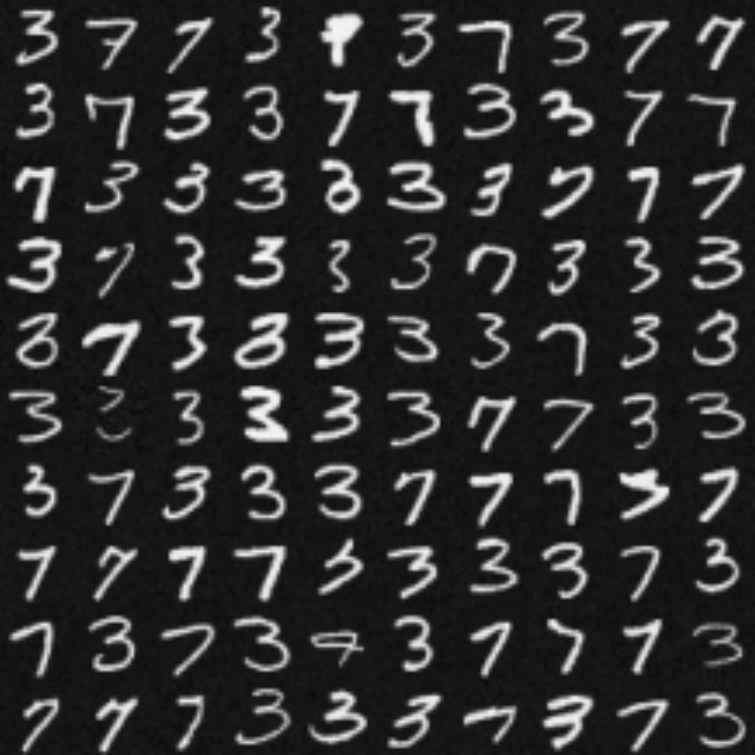}
    \end{tabular}
    \caption{Generated samples with VAE, GAN and SGM on the synthetic mixture of Gaussian on MNIST (top)
    and the subset of all $ 3$ and $ 7 $ of MNIST (bottom). The samples have been randomly chosen.}
    \label{fig:my_label}
\end{figure}

\end{document}